\documentclass[12pt]{article}
\usepackage[utf8]{inputenc}

\usepackage{comment}

\usepackage{subcaption}
\usepackage{setspace}
\usepackage[ruled,linesnumbered]{algorithm2e}
\usepackage{xcolor}
\usepackage{graphicx,natbib,amsfonts,amsthm,amssymb}
\usepackage[colorlinks,citecolor=black,urlcolor=black,linkcolor = black]{hyperref}
\usepackage{amsmath,mathrsfs}
\usepackage{wrapfig,lipsum,booktabs}
\usepackage{appendix}
\usepackage{enumitem}
\usepackage{multirow}

\newcommand{\blind}{1}

\def\spacingset#1{\renewcommand{\baselinestretch}%
{#1}\small\normalsize} \spacingset{1}

\newcommand{\IP}{{\rm I}\kern-0.18em{\rm P}}
\newcommand{\1}{{\rm 1}\kern-0.24em{\rm I}}
\newcommand{\E}{{\rm I}\kern-0.18em{\rm E}}
\newcommand{\R}{{\rm I}\kern-0.18em{\rm R}}

\newcommand{\cS}{\mathcal{S}}

\newcommand{\cT}{\mathcal{T}}
\newcommand{\cE}{\mathcal{E}}
\newcommand{\cI}{\mathcal{I}}
\newcommand{\qmq}[1]{\quad\mbox{#1}\quad}

\newcommand{\ot}{\overline{t}}
\newcommand{\hatphi}{\widehat{\phi}}
\newcommand{\hatpi}{\widehat{\pi}}
\newcommand{\hatY}{\widehat{Y}}
\newcommand{\hatp}{\hat{p}}

\newcommand{\bone}{\mathbf{1}}
\newcommand{\tR}{\Tilde{R}}

\usepackage[bottom]{footmisc}
\newtheorem{theorem}{Theorem}

\newtheorem{remark}{Remark}
\newtheorem{lemma}{Lemma}
\newtheorem{proposition}{Proposition}

\renewcommand{\appendix}{
 \setcounter{section}{0}%
  \setcounter{subsection}{0}%
  \renewcommand\thesection{\Alph{section}}
  \setcounter{equation}{0}
  \renewcommand{\theequation}{S.\arabic{equation}}
  \setcounter{figure}{0}
  \renewcommand\thefigure{S\arabic{figure}}
  \setcounter{table}{0}
  \renewcommand\thetable{S\arabic{table}}  
  }

\allowdisplaybreaks

\begin{document}

\if1\blind
{
  \title{Hierarchical Neyman-Pearson Classification for Prioritizing Severe Disease Categories in COVID-19 Patient Data}
  \author{Lijia Wang \thanks{Equal contribution}  \hspace{.2cm}\\
    School of Data Science, City University of Hong Kong\\
    Y. X. Rachel Wang \footnotemark[1] \hspace{.1cm} \thanks{Correspondence should be addressed to Y.X. Rachel Wang (rachel.wang@sydney.edu.au)} \hspace{.2cm}\\
    School of Mathematics and Statistics, University of Sydney
    \\
    Jingyi Jessica Li\\
    Department of Statistics, University of California, Los Angeles\\
   Xin Tong   \\
    Department of Data Sciences and Operations, University of Southern California}
  \maketitle
} \fi

\if0\blind
{
	\bigskip
	\bigskip
	\bigskip
	\begin{center}
		{\LARGE\bf Hierarchical Neyman-Pearson Classification for Prioritizing Severe Disease Categories in COVID-19 Patient Data}
	\end{center}
	\medskip
} \fi


\def\spacingset#1{\renewcommand{\baselinestretch}%
{#1}\small\normalsize} \spacingset{1}


\begin{abstract}
     COVID-19 has a spectrum of disease severity, ranging from asymptomatic to requiring hospitalization. Understanding the mechanisms driving disease severity is crucial for developing effective treatments and reducing mortality rates. One way to gain such understanding is using a multi-class classification framework, in which patients' biological features are used to predict patients' severity classes. In this severity classification problem, it is beneficial to prioritize the identification of more severe classes and control the “{under-classification}” errors, in which patients are misclassified into less severe categories. The Neyman-Pearson (NP) classification paradigm has been developed to prioritize the designated type of error. However, current NP procedures are either for binary classification or do not provide high probability controls on the prioritized errors in multi-class classification. Here, we propose a hierarchical NP (H-NP) framework and an umbrella algorithm that generally adapts to popular classification methods and  controls the {under-classification} errors with high probability. On an integrated collection of single-cell RNA-seq (scRNA-seq) datasets for $864$ patients, we explore ways of featurization and demonstrate the efficacy of the H-NP algorithm in controlling the {under-classification} errors regardless of featurization. Beyond COVID-19 severity classification, the H-NP algorithm generally applies to multi-class classification problems, where classes have a priority order.
\end{abstract}

\newpage

\spacingset{1.9}
\section{Introduction}


\vspace{-0.15in}


{The COVID-19 pandemic has infected over 767 million people and caused 6.94 million deaths (27 June 2023) \citep{whocovid19},} prompting collective efforts from statistics and other communities to address data-driven challenges.  Many statistical works have modeled epidemic dynamics \citep{betensky2020accounting,quick2021regression}, forecasted the case growth rates and outbreak locations \citep{brooks2020comparing, tang2021interplay, mcdonald2021can}, and analyzed and predicted the mortality rates \citep{james2021COVID,doi:10.1080/01621459.2022.2044826}. Classification problems, such as diagnosis (positive/negative) \citep{wu2020rapid, li2020using, zhang2021classification} and severity prediction \citep{yan2020prediction,sun2020combination,zhao2020prediction,ortiz2022effective}, have been tackled by machine learning approaches (e.g., logistic regression,  support vector machine (SVM),  random forest,  boosting, and neural networks; see \cite{alballa2021machine} for a review).

In the existing COVID-19 classification works, the commonly used data types are CT images, routine blood tests, and other clinical data including age, blood pressure and medical history \citep{meraihi2022machine}. In comparison, multiomics data are harder to acquire but can provide better insights into the molecular features driving patient responses \citep{overmyer2021large}. Recently, the increasing availability of single-cell RNA-seq (scRNA-seq) data offers the opportunity to understand transcriptional responses to COVID-19 severity at the cellular level \citep{wilk2020single,stephenson2021single,ren2021COVID}. 


More generally, genome-wide gene expression measurements have been routinely used in classification settings to characterize and distinguish disease subtypes, both in bulk-sample \citep{aibar2015analyse} and, more recently, single-cell level \citep{arvaniti2017sensitive,hu2019robust}.
While such genome-wide data can be costly, they provide a comprehensive view of the transcriptome and can unveil significant gene expression patterns for diseases with complex pathophysiology, where multiple genes and pathways are involved. Furthermore, as the patient-level measurements continue to grow in dimension and complexity (e.g., from a single bulk sample to thousands-to-millions of cells per patient), a supervised learning setting enables us to better establish the connection between patient-level features and their associated disease states, paving the way towards personalized treatment.

In this study, we focus on patient severity classification using an integrated collection of multi-patient scRNA-seq datasets. Based on the WHO guidelines \citep{world2020r}, COVID-19 patients have at least three severity categories: healthy, mild/moderate, and severe. 
The classical classification paradigm aims at minimizing the overall classification error. However, prioritizing the identification of more severe patients may provide important insights into the biological mechanisms underlying disease progression and severity, and facilitate the discovery of potential biomarkers for clinical diagnosis and therapeutic intervention. Consequently, it is important to prioritize the control of ``{under-classification}" errors, in which patients are misclassified into less severe categories.

Motivated by the gap in existing classification algorithms for severity classification (Section~\ref{sec:NP_review}), we propose a hierarchical Neyman-Pearson (H-NP) classification framework that prioritizes the {under-classification} error control in the following sense. Suppose there are $\cI$ classes with class labels $[\cI]=\{1, 2,\ldots, \cI\}$ ordered in decreasing severity. For $i\in[\cI-1]$, the $i$-th {under-classification} error is the probability of misclassifying an individual in class $i$ into any class $j$ with $j > i$. We develop an H-NP umbrella algorithm 
that controls the $i$-th {under-classification} error below a user-specified level $\alpha_i \in (0, 1)$ 
with high probability while minimizing a weighted sum of the remaining classification errors. Similar in spirit to the NP umbrella algorithm for binary classification in \citet{tong2018neyman}, the H-NP umbrella algorithm adapts to popular scoring-type multi-class classification methods (e.g., logistic regression, random forest, and SVM). To our knowledge, the algorithm is  the first to achieve asymmetric error control with high probability in multi-class classification.

Another contribution of this study is the exploration of appropriate ways to featurize multi-patient scRNA-seq data. Following the workflow in \citet{lin2022scalable}, we integrate 20 publicly available scRNA-seq datasets to form a sample of 864 patients with three levels of severity. For each patient, scRNA-seq data were collected from peripheral blood mononuclear cells (PBMCs) and processed into a sparse expression matrix, which consists of tens of thousands of genes in rows and thousands of cells in columns. We propose four ways of extracting a feature vector from each of these $864$ matrices. Then we evaluate the performance of each featurization way in combination with multiple classification methods under both the classical and H-NP classification paradigms. We note that our H-NP umbrella algorithm is applicable to other featurizations of scRNA-seq data, other forms of patient data, and more general disease classification problems with a severity ordering.

Below we review the NP paradigm and featurization of multi-patient scRNA-seq data as the background of our work. 


\vspace{-0.2in}

\subsection{Neyman-Pearson paradigm and multi-class classification}\label{sec:NP_review}

\vspace{-0.15in}

Classical binary classification focuses on minimizing the overall classification error, i.e., a weighted sum of type I and II errors, where the weights are the marginal probabilities of the two classes. However, the class priorities are not reflected by the class weights in many applications, especially disease severity classification, where the severe class is the minor class and has a smaller weight (e.g., HIV \citep{meyer1987screening} and cancer \citep{dettling2003boosting}). 
One class of methods that addresses this error asymmetry is cost-sensitive learning \citep{elkan2001foundations,margineantu2002class}, which assigns different costs to type I and type II errors. However, such weights may not be easy to choose in practice, especially in a multi-class setting; nor do these methods provide high probability controls on the prioritized errors. The NP classification paradigm \citep{cannon2002learning,scott2005neyman,rigollet2011neyman} was developed as an alternative framework to enforce class priorities: it finds a classifier that controls the population type I error (the prioritized error, e.g.,  misclassifying diseased patients as healthy) under a user-specified level $\alpha$ while minimizing the type II error (the error with less priority, e.g., misdiagnosing healthy people as sick). 
Practically, using an order statistics approach,  \citet{tong2018neyman} proposed an NP umbrella algorithm that adapts all scoring-type classification methods (e.g., logistic regression) to the NP paradigm for classifier construction. 
The resulting classifier has the population type I error under $\alpha$ with high probability.   
Besides disease severity classification, the NP classification paradigm has found diverse applications, including social media text classification \citep{xia2021intentional} and crisis risk control \citep{feng2021targeted}.  Nevertheless, the original NP paradigm is for binary classification only.

Although several works aimed to control prioritized errors in multi-class classification \citep{landgrebe2005neyman,xiong2006measuring,tian2021neyman}, they  did not provide high probability control. That is, if they are applied to severe disease classification, there is a non-trivial chance that their {under-classification} errors exceed the desired levels. 

\vspace{-0.15in}

\subsection{ScRNA-seq data featurization}

\vspace{-0.1in}
 In multi-patient scRNA-seq data, every patient has a gene-by-cell expression matrix; genes are matched across patients, but cells are not. For learning tasks with patients as instances, featurization is a necessary step to ensure that all patients have feautures in the same space. A common featurization approach is to assign every patient's cells into cell types, which are comparable across patients, by clustering \citep{stanley2020vopo,ganio2020preferential} and/or manual annotation \citep{han2019differential}. Then, each patient's gene-by-cell expression matrix can be converted into a gene-by-cell-type expression matrix using a summary statistic (e.g., every gene's mean expression in a cell type), so all patients have gene-by-cell-type expression matrices with the same dimensions. We note here that most of the previous multi-patient single-cell studies with a reasonably large cohort used CyTOF data \citep{davis2017systems}, which typically measures $50$--$100$ protein markers, whereas scRNA-seq data have a much higher feature dimension, containing expression values of $\sim10^4$ genes. Thus further featurization is necessary to convert each patient's gene-by-cell-type expression matrix into a feature vector for classification. 



Following the data processing workflow in \citet{lin2022scalable}, we obtain $864$ patients' cell-type-by-gene expression matrices, which include 18 cell types and 3{,}000 genes (after filtering). We propose and compare four ways of featurizing these matrices into vectors, which differ in their treatments of 0 values and approaches to dimension reduction. Note that we perform featurization as a separate step before classification so that all classification methods are applicable. Separating the featurization step also allows us to investigate whether a featurization way maintains robust performance across classification methods. 


The rest of the paper is organized as follows. In Section~\ref{sec:method}, we introduce the H-NP classification framework and propose an umbrella algorithm to control the {under-classification} errors with high probability. Next, we conduct extensive simulation studies to evaluate the performance of the
 umbrella algorithm. In Section~\ref{sec:COVID}, we describe four ways of featurizing the COVID-19 multi-patient scRNA-seq data and show that the H-NP umbrella algorithm consistently controls the {under-classification} errors in COVID-19 severity classification across all featurization ways and classification methods. Furthermore, we demonstrate that utilizing the scRNA-seq data allows us to gain biological insights into the mechanism and immune response of severe patients at both the cell-type and gene levels. Supplementary Materials contain technical derivations, proofs and additional numerical results.

\vspace{-0.25in}

\section{Hierarchical Neyman-Pearson (H-NP) classification}\label{sec:method}
\vspace{-0.15in}

\subsection{{Under-classification} errors in H-NP classification}

We first introduce the formulation of H-NP classification and define the {under-classification} errors, which are the probabilities of individuals being misclassified to less severe (more generally, less important) classes. In an H-NP problem with $\cI \ge 2$ classes, the class labels $i \in [\cI]:=\{1,2, \ldots, \cI\}$ are ranked in a decreasing order of importance, i.e., class $i$ is more important than class $j$ if $i < j$. Let $(X, Y)$ be a random pair, where $X \in  \mathcal{X} \subset \R^d$ represents a vector of features, and $Y \in [\cI]$ denotes the class label. A classifier $\phi: \mathcal{X}  \rightarrow [\cI]$ maps a feature vector $X$ to a predicted class label. In the following discussion, we abbreviate $\IP(\cdot\mid Y = i)$ as $P_i(\cdot)$. Our H-NP framework aims to control the {under-classification} errors at the population level in the sense that
\vspace{-0.7cm}
\begin{equation}\label{eq:np_problem}
     R_{i\star}(\phi)  = P_i(\phi(X)\in\{i+1, \ldots, \mathcal{I}\}) \leq \alpha_i \qmq{for} i \in [\cI -1]\,, \vspace{-0.6cm}
\end{equation}
where $\alpha_i \in (0,1)$ is the desired control level for the $i$-th {under-classification} error $R_{i\star}(\phi)$. Simultaneously, 
our H-NP framework minimizes the weighted sum of the remaining errors, which can be expressed as
\vspace{-0.7cm}
\begin{equation}\label{eq:the_error}
R^c(\phi)  = \IP(\phi(X) \neq Y) - \sum^{\cI - 1}_{i = 1}\pi_i R_{i\star}( \phi) \,,   \qmq{where} \pi_i = \IP(Y = i)\,. \vspace{-0.5cm}
\end{equation}
We note that when $\cI=2$, this H-NP formulation is equivalent to the binary NP classification (prioritizing class 1 over class 2),  with $R_{1\star}(\phi)$ being the population type I error.

For COVID-19 severity classification with three levels, severe patients labeled as  $Y = 1$ have the top priority, and we want to control the probability of severe patients not being identified, which is $R_{1\star}(\phi) $. The secondary priority is for moderate patients labeled as $Y = 2$; $R_{2\star}(\phi) $ is the probability of moderate patients being classified as healthy. Healthy patients that do not need medical care are labeled as  $Y=3$. Note that $R_{i\star}(\cdot)$ and $R^c(\cdot)$ are population-level quantities as they depend on the intrinsic  distribution of $(X, Y)$, and it is hard to control the $R_{i\star}(\cdot)$'s almost surely due to the randomness of the classifier.

\vspace{-0.1in}

\subsection{H-NP algorithm with high probability control}\label{subsec:K_class_classifier}

In this section, we construct an H-NP  umbrella algorithm that controls the population {under-classification} errors in the sense that $\IP(R_{i\star}( \hatphi)  > \alpha_i ) \leq \delta_{i}$ for $i \in [\cI - 1]$, where $(\delta_1, \ldots, \delta_{\cI - 1})$ is a vector of tolerance parameters, and $\hatphi$ is a scoring-type classifier to be defined below.

Roughly speaking, we employ a sample-splitting strategy, which uses  some data subsets to train the scoring functions from a base classification method and other data subsets to select appropriate thresholds on the scores to achieve population-level error controls. Here, the scoring functions refer to the scores assigned to each possible class label for a given input observation and include examples such as the output from the softmax transformation in multinomial logistic regression. For $i\in[\cI]$, let $\cS_i = \{X^i_j\}_{j=1}^{N_i}$ denote $N_i$ independent observations from class $i$, where $N_i$ is the size of the class.  In the following discussion, the superscript on $X$ is dropped for brevity when it is clear which class the observation comes from. 
Our procedure randomly splits the class-$i$ observations into (up to) three parts: $\cS_{is}$ ($i\in[\cI]$) for obtaining scoring functions, $\cS_{it}$ ($i\in[\cI-1]$) for selecting thresholds, and $\cS_{ie}$ ($i=2, \dots, \cI$) for computing empirical errors. As will be made clear later,  our procedure does not require $\cS_{1e}$ or $\cS_{\cI t}$ and splits class 1 and class $\cI$ into two parts only. After splitting, we use the combination $\cS_s = \bigcup_{i \in [\cI]} \cS_{is}$ to train the scoring functions. 

We consider a classifier that relies on $\cI - 1$ scoring functions $T_1, T_2, \ldots, T_{\cI - 1}: \mathcal{X} \rightarrow \R$, where the class decision is made sequentially with each $T_i(X)$ determining whether the observation belongs to class $i$ or one of the less prioritized classes $(i + 1), \ldots, \cI$. Thus at each step $i$, the decision is binary, allowing us to use the NP Lemma to motivate the construction of our scoring functions. Note that $\IP(Y=i \mid X=x)/\IP(Y\in\{i+1, \dots, \cI\}\mid X=x) \propto f_i(x)/f_{>i}(x)$, where $f_{>i}(x)$ and $f_{i}(x)$ represent the density function of $X$ when $Y>i$ and $Y=i$, respectively, and the density ratio is the statistic that leads to the most powerful test with a given level of control on one of the errors by the NP Lemma. Given a typical scoring-type classification method (e.g., logistic regression, random forest, SVM, and neural network) that provides the probability estimates $\widehat{\IP}(Y = i \mid X)$ for $i\in[\cI]$, we can construct our scores using these estimates by defining 
\vspace{-0.5cm}
$$T_1 (X) = \widehat{\IP}(Y = 1 \mid X)\,, \qmq{and} T_i (X) = \frac{\widehat{\IP}(Y = i\mid X)}{\sum^{\cI}_{j= i + 1}\widehat{\IP}(Y = j\mid X)} \qmq{for}  1 < i <\cI -1\,. \vspace{-0.5cm}$$

Given thresholds $(t_1,t_2,\ldots, t_{\cI-1})$, we consider an H-NP classifier of the form
\vspace{-0.5cm}
\begin{equation}\label{eq:decision}
    \hatphi(X)  = \begin{cases} 1\,, &T_1(X)  \geq t_1\,;  \vspace{-0.3cm}\\
    2\,, &T_2(X)  \geq t_2 \qmq{and} T_1(X) < t_1\,; \vspace{-0.5cm}\\
    \cdots & \vspace{-0.5cm}\\
    \cI - 1\,, &T_{\cI - 1}(X)  \geq t_{\cI-1}  \qmq{and} T_1(X) < t_1,\ldots,T_{\cI-2}(X) < t_{\cI-2} \,; \vspace{-0.3cm}\\
     \cI\,, &\mbox{otherwise}\,. 
    \end{cases}
    \vspace{-0.4cm}
\end{equation}
Then the $i$-th {under-classification} error for this classifier can be written as 
\vspace{-0.6cm}
\begin{equation}\label{eq:R_detail}
    R_{i\star}( \hatphi) = P_i\left(\hatphi(X) \in\{i+1, \ldots, \mathcal{I}\} \right) = P_i\left(T_{1} (X)< t_1, \ldots, T_i (X)< t_i\right)\,, \vspace{-0.5cm}
\end{equation}
where $X$ is a new observation from the $i$-th class  independent of the data used for score training and threshold selection. The thresholds $(t_1,t_2,\ldots, t_{\cI-1})$ are selected using the observations in $\cS_{1t}, \ldots, \cS_{(\cI - 1)t}$, and they are chosen to satisfy $\IP(R_{i\star}( \hatphi)  > \alpha_i) \leq \delta_i$ for all $i\in[\cI-1]$. In what follows, we will develop our arguments conditional on the data $\cS_s$ for training the scoring functions so that $T_i$'s can be viewed as fixed functions. 

According to Eq~\eqref{eq:decision}, the first {under-classification} error $R_{1\star}(\hatphi) = P_1\left(T_{1} (X)< t_1 \right)$ only depends on $t_1$, while the other {under-classification} errors $R_{i\star}( \hatphi) $ depend on $t_1, \ldots, t_i$. To achieve the high probability controls with $\IP(R_{i\star}( \hatphi)  > \alpha_i) \leq \delta_i$ for all $i \in [\cI - 1]$, we select $t_1,\ldots t_{\cI - 1}$ sequentially using an order statistics approach. We start with the selection of $t_1$, which is covered by the following general proposition. The proof is a modification of Proposition 1 in \citet{tong2018neyman} and can be found in Supplementary Section~B.1.

\begin{proposition}\label{prop:t1}
For any $i\in[\cI]$, 
denote $\cT_i = \{ T_i(X) \mid X\in \cS_{it}\}$, and let $t_{i(k)} $ be the corresponding $k$-th order statistic. Further denote the cardinality of $\cT_i$ as $n_i$. Assuming that the data used to train the scoring functions and the left-out data are independent, then given a control level $\alpha$, for another independent observation $X$ from class $i$,

\vspace{-0.6cm}
\begin{equation}\label{eq:prop1}
 \IP\left(P_i\left[T_i(X) < t_{i(k)}\mid t_{i(k)} \right] > \alpha\right) \leq  v(k,n_i,\alpha):= \sum^{k - 1}_{j = 0} {n_i \choose j}( \alpha)^{ j} (1  - \alpha)^{n_i - j}\,.
\end{equation}
\end{proposition}
We remark that similar to Proposition 1 in \cite{tong2018neyman}, if $T_i$ is a continuous random variable, the bound in Eq~\eqref{eq:prop1} is tight.

\vspace{-0.1in}

\begin{algorithm}[h!!!]
\caption{DeltaSearch$(n,\alpha,\delta)$} \label{alg:delta}
\SetKw{KwBy}{by}
\SetKwInOut{Input}{Input}\SetKwInOut{Output}{Output}

\SetAlgoLined

\Input{size: $n$; level: $\alpha$; tolerance: $\delta$.}

$k = 0$, $v_k = 0$

\While{$v_k \leq \delta$}{

$v_k = v_k  + {n \choose k} (\alpha)^{k} (1 - \alpha)^{n - k}$

$k = k+1$

}

\Output{$k$}
\end{algorithm}

Let $k_i = \max \{k \mid v(k,n_i,\alpha_i) \leq \delta_i\}$, which can be computed using Algorithm~\ref{alg:delta}. Then Proposition~\ref{prop:t1} and  Eq~\eqref{eq:R_detail} imply

\vspace{-1.5cm}

\begin{equation}\label{eq:control_1}
  \IP\left(R_{i\star}(\hatphi) > \alpha_i \right) \leq \IP\left(P_i\left[T_i(X)< t_i\mid t_{i(k_i)}\right] > \alpha_i \right) \leq \delta_i\,\qmq{for all} t_i \leq t_{i(k_i)}\,.
  \vspace{-0.5cm}
\end{equation} 
We note that to have a solution for $v(k, n_i, \alpha_i)\leq \delta_i$ among $k\in[n_i]$, we need $n_i \geq \log \delta_i/ \log (1 - \alpha_i)$, the minimum sample size required for the class $\cS_{it}$.  When $i = 1$, the first inequality in Eq~\eqref{eq:control_1} becomes equality, so $t_{1(k_1)}$ is an effective upper bound on $t_1$ when we later minimize the empirical counterpart of $R^c(\cdot)$ in Eq~\eqref{eq:the_error} with respect to different feasible threshold choices. On the other hand, for $i> 1$, the inequality is mostly strict, which means that the bound $t_{i(k_i)}$ on $t_i$ is expected to be loose and  can be improved. To this end, we note that Eq~\eqref{eq:R_detail} can be decomposed as

\vspace{-1.5cm}

\begin{equation}\label{eq:R_sep}
   R_{i\star}( \hatphi)    = P_i\left(\left.T_i (X)< t_i \right| T_1 (X)< t_1, \ldots,  T_{i - 1} (X)< t_{i - 1}\right) \cdot P_i\left(T_1 (X)< t_1, \ldots,  T_{i - 1} (X)< t_{i - 1} \right) 
\end{equation}
leading to the following theorem that upper bounds $t_i$ given the previous thresholds.  

\begin{theorem}\label{prop:ti}
Given the previous thresholds $t_1, \ldots, t_{i - 1}$, consider all the scores $T_i$ on the left-out class $\cS_{it}$, $\cT_i = \{ T_i(X) \mid X \in \cS_{it}\}$, and a subset of these scores depending on the previous thresholds, defined as $\cT'_i = \{ T_i(X) \mid X \in \cS_{it},  T_1(X)< t_1, \ldots, T_{i - 1}(X)< t_{i - 1} \}$. We use $t_{i(k)}$ and  $t'_{i(k)}$ to denote the $k$-th order statistic of $\cT_i$ and $\cT'_i$, respectively. Let  $n_i$ and $n_i'$ be the cardinality of $\cT_i$ and $\cT'_i$, respectively, and $\alpha_i$ and $\delta_i$ be the prespecified control level and violation tolerance for the $i$-th {under-classification} error $R_{i\star}( \cdot) $.
We set 
\vspace{-0.5cm}

\begin{equation}\label{eq:adjustment}
\hatp_i = \frac{n_i'}{n_i}\,,\,  p_i =  \hatp_i + c(n_i)\,,\, \alpha_i' = \frac{\alpha_i}{p_i}\,,\, \delta_i' = \delta_i - \exp\{- 2n_i c^2(n_i)\}\,,    
\end{equation}
where $c(n) = \mathcal{O}(1/\sqrt{n})$. Let

\vspace{-0.5cm}

\begin{equation}\label{eq:threshold_2}
\ot_i = \begin{cases}  t'_{i(k'_i)}\,, & \mbox{if } n_i' \geq \log \delta_i'/ \log (1 - \alpha_i') \qmq{and} \alpha_i' < 1\,; \\
t_{i(k_i)}\,, &\mbox{otherwise}\,,
\end{cases}
\end{equation}
where $k_i = \max \{k\in[n_i] \mid v(k,n_i,\alpha_i) \leq \delta_i\} \qmq{and} 
    k'_i = \max \{k\in[n_i'] \mid v(k,n'_i,\alpha'_i)\leq \delta_i'\}\,.$ Then,

\vspace{-2cm}

\begin{equation}\label{eq:control_2}
\IP(R_{i\star}( \hatphi)  > \alpha_i) = \IP\left( P_i\left[T_{1} (X)< t_1, \ldots T_i (X)< t_i \mid \ot_i \right] > \alpha_i \right) \leq \delta_i \qmq{for all} t_i \leq \ot_i\,.    
\end{equation}
\end{theorem}
\vspace{-0.6cm}

In other words, if the cardinality of $\cT'_i$ exceeds a threshold, we can refine the choice of the upper bound according to Eq~\eqref{eq:threshold_2}; otherwise, the bound in Proposition~\ref{prop:t1} always applies. The proof of the theorem is provided in Supplementary Section~B.2; the computation of the upper bound $\ot_i$ is summarized in Algorithm~\ref{alg:upper}. $\ot_{i}$ guarantees the required high probability control on the $i$-th {under-classification} error, while providing a tighter bound compared with Eq~\eqref{eq:R_detail}. We make two additional remarks as follows.
\begin{remark}
\begin{enumerate}[label=\alph*)]
\item The minimum sample size requirement for $\cS_{it}$ is still $n_i \geq \log \delta_i/ \log(1 - \alpha_i)$ because $t_{i(k_i)}$ in Eq~\eqref{eq:threshold_2} always exists when this inequality holds.  For instance, if $\alpha_i = 0.05$ and $\delta_i = 0.05$, then $n_i\geq 59$.

\item  The choice of $c(n)$ involves a trade-off between $\alpha_i'$ and $\delta_i'$, although under the constraint $c(n)=\mathcal{O}(1/\sqrt{n})$, any changes in both quantities are small in magnitude for large $n$. For example, a larger $c(n)$ leads to a smaller $\alpha_i'$ and  a larger $\delta_i'$, thus a looser tolerance level comes at the cost of a stricter error control level. In practice, larger $\alpha_i'$  and larger $\delta_i'$ values are desired since they lead to a wider region for $t_i$. We set $c(n) = 2/\sqrt{n}$ throughout the rest of the paper. Then by Eq~\eqref{eq:adjustment}, $\alpha_i'$ increases as $n$ increases, and $\delta_i' = \delta_i - e^{-4}$, so the difference between $\delta_i'$ and the prespecified $\delta_i$ is sufficiently small. 

\item Eq~\eqref{eq:control_2} has two cases, as Eq~\eqref{eq:threshold_2} indicates. When $\ot_i = t_{i(k_i)}$, the bound remains the same as Eq~\eqref{eq:control_1}, which is not tight for $i>1$. When $\ot_i = t'_{i(k'_i)}$, Eq~\eqref{eq:control_2} provides a tighter bound through the decomposition in Eq~\eqref{eq:R_sep}, where the first part is bounded by a concentration argument, and the second part achieves a tight bound the same way as Proposition~\ref{prop:t1}.

\end{enumerate}
\end{remark}

With the set of upper bounds on the thresholds chosen according to Theorem~\ref{prop:ti}, the next step is to find an optimal set of thresholds $(t_1, t_2, \ldots, t_{\cI-1})$ satisfying these upper bounds while minimizing the empirical version of $R^c(\hatphi)$ , which is calculated using observations in $\cS_e = \bigcup^\cI_{i = 2} \cS_{ie}$ (since class-1 observations are not needed in $R^c(\hatphi)$). For brevity, we denote all the empirical errors as $\Tilde{R}$, e.g.,  $\Tilde{R}^c$.  In Section~\ref{sec:sim}, we will show numerically that Theorem~\ref{prop:ti} provides a wider search region  for the  threshold $t_i$ compared to Proposition~\ref{prop:t1}, which benefits the minimization of $R^c$. 

As our COVID-19 data has three severity levels, in the next section, we will focus on the three-class H-NP umbrella algorithm and describe in more details how the above procedures can be combined to select the optimal thresholds in the final classifier.

\begin{algorithm}[h!!!]
\caption{UpperBound$(\cS_{it}, \alpha_i, \delta_i, (T_1,\ldots, T_i), (t_1,\ldots, t_{i - 1}) )$} \label{alg:upper}
\SetKw{KwBy}{by}
\SetKwInOut{Input}{Input}\SetKwInOut{Output}{Output}

\SetAlgoLined

\Input{The left-out class-$i$ samples: $\cS_{it}$;
 level: $\alpha_i$; tolerance: $\delta_i$; score functions: $(T_1,\ldots, T_i)$; thresholds: $(t_1,\ldots, t_{i - 1})$.}
 
 $n_i \leftarrow |\cS_{it}|$

$ \{t_{i(1)}, \ldots, t_{i(n_i)}\} \leftarrow$ sort $\cT_i = \{ T_i(X ) \mid X \in \cS_{it}\}$

$k_i \leftarrow \mbox{DeltaSearch}(n_i, \alpha_i, \delta_i)$ \tcp*{i.e., Algorithm~\ref{alg:delta}}

$\ot_i \leftarrow t_{i(k_i)}$

\If{ $i > 1$}{

$\cT'_i \leftarrow \{t'_{i(1)}, \ldots, t'_{i(n_i')}\} = \mbox{sort} \{ T_i(X) \mid X \in \cS_{it}, T_1(X)< t_1, \ldots, T_{i - 1}(X)< t_{i - 1} \}$  \tcp*{Note that $n_i'$ is random }

$\hatp_i \leftarrow \frac{ n_i'}{n_i}$, $p_i \leftarrow \hatp_i + c(n_i)$, $\alpha_i' \leftarrow \alpha_i/p_i$ , $\delta_i'\leftarrow \delta_i - e^{-2n_i c^2(n_i)}$\tcp*{  e.g., $c(n)$ = $\frac{2}{\sqrt{n}}$}

\If{$n_i' \geq \log \delta_i'/ \log (1 - \alpha_i') \;\mathrm{ and }\; \alpha_i' < 1$}{
$k'_i \leftarrow \mbox{DeltaSearch}(n'_i, \alpha'_i, \delta'_i)$

$\ot_i \leftarrow t'_{i(k'_i)}$}}

\Output{$\ot_i$}
\end{algorithm}

\vspace{-1cm}

\subsection{H-NP umbrella algorithm for three classes}\label{sec:3_class}

Since our COVID-19 data groups patients into three severity categories, we introduce our H-NP umbrella algorithm for $\cI = 3$. In this case, there are
two {under-classification} errors $R_{1\star}(\phi) = P_1(\phi(X) \in \{2,3\})$ and $R_{2\star}(\phi) = P_2(\phi(X) = 3)$, which need to be controlled at prespecified levels $\alpha_1, \alpha_2$ with tolerance levels $\delta_1, \delta_2$, respectively. In addition, we wish to minimize  the weighted sum of errors
\vspace{-0.7cm}
\begin{multline}\label{eq:R_c_3}
   R^c(\phi) = \IP(\phi(X) \neq Y) - \pi_1 R_{1\star}(\phi) - \pi_2 R_{2\star}(\phi)\\ = \pi_2 P_2(\phi(X) = 1) + \pi_3 [P_3(\phi(X) = 1) + P_3(\phi(X) = 2)]\,. \vspace{-0.7cm}
\end{multline}
When $\cI = 3$, our H-NP umbrella algorithm relies on two scoring functions $T_1, T_2: \mathcal{X} \rightarrow \R$, which can be constructed by Eq~\eqref{eq:decision} using the estimates $\widehat{\IP}(Y = i\mid X)$ from any scoring-type classification method:
\vspace{-0.7cm}
\begin{equation}\label{eq:scores}
T_1 (X) = \widehat{\IP}(Y = 1 \mid X) \qmq{and} T_2 (X)  = \frac{\widehat{\IP}(Y = 2\mid X)}{ \widehat{\IP}(Y = 3\mid X)}\,.  
\vspace{-0.7cm}
\end{equation}
The H-NP classifier then takes the form
\vspace{-0.6cm}
\begin{equation}\label{eq:decision_simple}
    \hatphi(X)  = \begin{cases}  1\,, &T_1(X)  \geq t_1\,; \vspace{-0.3cm}\\
    2\,, &T_2(X)  \geq t_2 \qmq{and} T_1(X) < t_1\,;\vspace{-0.3cm}\\
     3\,, &\mbox{otherwise}\,.
    \end{cases}
    \vspace{-0.4cm}
\end{equation}
Here $T_2$ determines whether an observation belongs to class 2 or class 3, with a larger value indicating a higher probability for class 2. 
Applying Algorithm~\ref{alg:upper}, we can find $\ot_1$ such that any threshold $t_1 \leq \ot_1$ will satisfy the high probability control on the first {under-classification} error, that is $\IP(R_{1\star}(\hatphi)> \alpha_1) =  \IP\left(P_1\left[T_1 (X)< t_1 \mid \ot_1 \right] >  \alpha_1\right) \leq \delta_1$. Recall that the computation of $\ot_2$ (and consequently $t_2$) depends on the choice of $t_1$.
Given a fixed $t_1$, the high probability control on the second {under-classification} errors is $\IP(R_{2\star} (\hatphi)> \alpha_2) = \IP\left(P_2\left[T_1 (X)< t_1, T_2 (X)< t_2 \mid \ot_2 \right] >  \alpha_2\right) \leq \delta_2$, where $\ot_2$ is computed by Algorithm~\ref{alg:upper} so that any $t_2 \leq \ot_2$ satisfies the constraint. 

 
 \begin{figure}[!ht]
 \centering
     \begin{subfigure}[b]{0.45\textwidth}
    \centering
    \includegraphics[width=\textwidth]{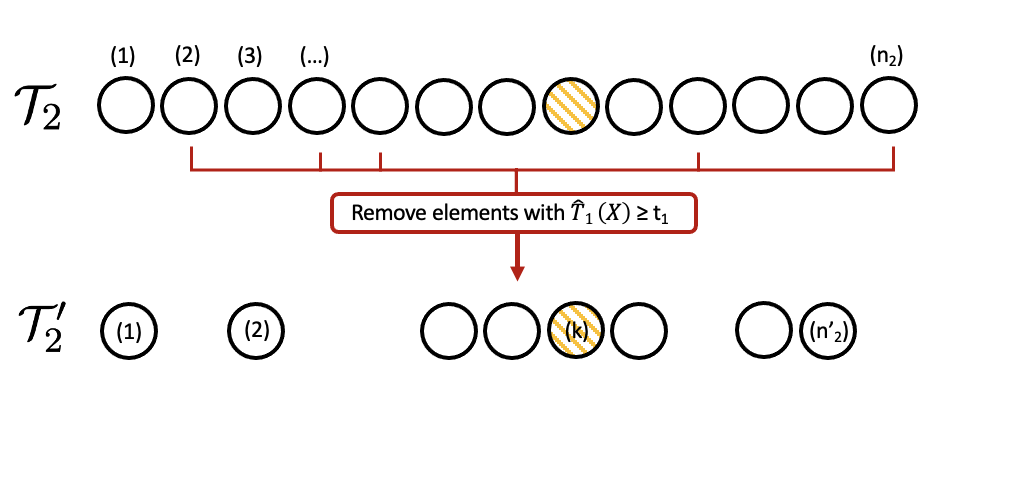}
    \caption{\footnotesize The construction of $\cT'_{2}$ with a fixed $t_1$. }\label{fig:t1_effect}
     \end{subfigure}
     \begin{subfigure}[b]{0.45\textwidth}
         \centering
    \includegraphics[width=\textwidth]{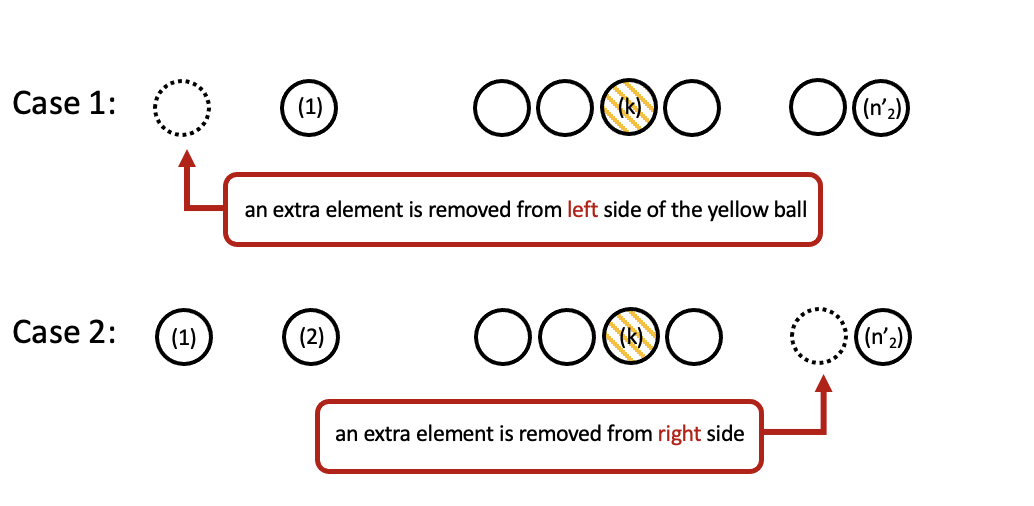}
   \caption{\footnotesize The effect of decreasing $t_1$.}\label{fig:t1_effect_2}
     \end{subfigure}
    \caption{\footnotesize The influence of $t_1$ on the error $P_3\left(\hatY = 2\right)$.}
\end{figure}

\vspace{-0.5cm}
 
The interaction between $t_1$ and $t_2$ comes into play when minimizing the remaining errors in $R^c(\hatphi)$. First note that using Eq~\eqref{eq:R_c_3} and~\eqref{eq:decision_simple}, the other types of errors in $R^c(\hatphi)$ are
\vspace{-0.5cm}
\begin{align}
    & P_2\left(\hatphi(X) = 1 \right)   = P_2\left(T_1(X) \geq t_1\right)\,,\, P_3\left(\hatphi(X) = 1\right)   = P_3\left(T_1(X) \geq t_1\right)\,,\label{eq:risk2}\\
 & P_3\left(\hatphi(X) = 2\right)   = P_3\left(T_1(X) < t_1, T_2(X) \geq t_2\right)\,. \notag 
\end{align}
To simplify the notation, let $\hatY$ denote $\hatphi(X)$ in the following discussion. For a fixed $t_1$, decreasing $t_2$ leads to an increase in $P_3(\hatY = 2)$ and has no effect on the other errors in \eqref{eq:risk2}, which means that $t_2 = \ot_2$ minimizes $R^c(\hatphi)$. However, the selection of $t_1$ is not as straightforward as $t_2$. Figure~\ref{fig:t1_effect} illustrates how the set $\cT'_{2} = \{ T_2(X) \mid X \in \cS_{2t}, T_1(X)< t_1\}$ (as appeared in Theorem~\ref{prop:ti}) is constructed for a given $t_1$, where the elements are ordered by their $T_2$ values. Clearly, more elements are removed from  $\cT'_{2}$ as $t_1$ decreases, leading to a smaller $n_2'$. Consider an element in the set $\cT'_{2}$ which has rank $k$ in the ordered list (colored yellow in Figure~\ref{fig:t1_effect}). Then $k, n_2', \alpha_2'$, and consequently $v(k,n_2',\alpha_2')$, will all be affected by decreasing $t_1$, but the change is not monotonic as shown in Figure~\ref{fig:t1_effect_2}. Decreasing $t_1$ could remove elements (dashed circles in Figure~\ref{fig:t1_effect_2}) either to the left side (case 1) or right side (case 2) of the yellow element, depending on the values of the scores $T_1$. In case 1, $v(k,n_2',\alpha_2')$ decreases, resulting in a larger $\ot_2$ and a  smaller $P_3(\hatY= 2)$ error, whereas the reverse can happen in case 2. The details of 
how $v(k,n_2',\alpha_2')$ changes can be found in Supplementary Section~B.3, with additional simulations in Supplementary Figure~S13. In view of the above, minimizing the empirical error $\Tilde{R^c}$ requires a grid search over $t_1$, for which we use the set $\cT_1 = \{ T_1(X) \mid X \in \cS_{1t}\}$, and the overall algorithm for finding the optimal thresholds and the resulting classifier is described in Algorithm~\ref{alg:opt_class}, which we name as the H-NP umbrella algorithm. The algorithm for the general case with $\cI>3$ can be found in Supplementary Section~E.


\begin{algorithm}[htb!]
\caption{H-NP umbrella algorithm for $\cI = 3$ }\label{alg:opt_class}
\SetKw{KwBy}{by}
\SetKwInOut{Input}{Input}\SetKwInOut{Output}{Output}

\SetAlgoLined

\Input{ Sample: $\cS = \cS_1 \cup \cS_2 \cup \cS_3$; levels: $(\alpha_1, \alpha_2)$;  tolerances: $(\delta_1,\delta_2)$; grid set: $A_1$ (e.g., $\cT_1$).}

$\hatpi_2 = |\cS_2|/|\cS|$; $\hatpi_3 = |\cS_3|/|\cS|$

$\cS_{1s},\cS_{1t},  \leftarrow $ Random split $\cS_1$; $\cS_{2s},\cS_{2t}, \cS_{2e}  \leftarrow $ Random split $\cS_2$;  $\cS_{3s}, \cS_{3e}  \leftarrow $ Random split $\cS_3$

$\cS_s = \cS_{1s} \cup \cS_{2s} \cup \cS_{3s}$

$T_1, T_2 \leftarrow \mbox{A  base classification method}(\cS_s)$ \tcp*{c.f. Eq~\eqref{eq:scores}}

$\ot_1 \leftarrow \mbox{UpperBound}(\cS_{1t}, \alpha_1, \delta_1, (T_1), \mbox{NULL}) $ \tcp*{i.e., Algorithm~\ref{alg:upper}}

$\Tilde{R^c} = 1 $

\For{$t_1 \in A_1 \cap (-\infty, \ot_1]$}{

$t_2 \leftarrow \mbox{UpperBound}(\cS_{2t}, \alpha_2, \delta_2, (T_1, T_2), (t_1) )$ 

$\hatphi \leftarrow$ a classifier with respect to $t_1, t_2$

$e_{21} = \sum_{X \in \cS_{2e}} \1\{ \hatphi(X) = 1 \}/|\cS_{2e}|$, $e_{3} = \sum_{X \in \cS_{3e}} \1\{ \hatphi(X) \in \{1,2\}\} \}/|\cS_{3e}|$ 

$\Tilde{R^c}_{\mathrm{new}} = \hatpi_2 e_{21} + \hatpi_3e_{3}$

\If{$\Tilde{R^c}_{\mathrm{new}}< \Tilde{R^c}$}{
$\Tilde{R^c} \leftarrow \Tilde{R^c}_{\mathrm{new}}$, $\hatphi^* \leftarrow \hatphi$ }
}

\Output{$\hatphi^*$}
\end{algorithm}

\subsection{Simulation studies}\label{sec:sim}

We first examine the validity of our H-NP umbrella algorithm using simulated data from a setting denoted \textbf{T1.1}, where $\cI=3$, and the feature vectors in class $i$ are generated as ${(X^i)}^\top \sim N(\mu_i, I)$, where $\mu_1 = (0, -1)^\top$, $\mu_2 = (-1, 1)^\top$, $\mu_3 = (1, 0)^\top$ and $I$ is the $2\times2$ identity matrix.
For each simulated dataset, we generate the feature vectors and labels with 500 observations in each of the three classes. The observations are randomly separated into parts for score training, threshold selection and computing empirical errors: $\cS_1$ is split into $50\%$, $50\%$ for $\cS_{1s}$, $\cS_{1t}$;  $\cS_2$ is split into $45\%$, $50\%$ and $5\%$ for $\cS_{2s}$, $\cS_{2t}$ and $\cS_{2e}$; $\cS_3$ is split into $95\%$, $5\%$ for $\cS_{3s}$, $\cS_{3e}$, respectively.  All the results in this section are based on $1{,}000$ repetitions from a given setting. We set $\alpha_1 = \alpha_2 = 0.05$ and $\delta_1 = \delta_2 = 0.05$.
To approximate and evaluate the true population errors $R_{1\star}$, $R_{2\star}$  and  $R^c$, we additionally generate $20{,}000$  observations for each class and refer to them as the test set. 


First, we demonstrate that Algorithm~\ref{alg:opt_class} outputs an H-NP classifier with the desired high probability controls. More specifically, we show that  any $t_1 \leq \ot_1$  and $t_2 = \ot_2$ ($\ot_1$, $\ot_2$ are computed by Algorithm~\ref{alg:upper})  will lead to a valid threshold pair $(t_1, t_2)$ satisfying $\IP(R_{1\star}(\hatphi) > \alpha_1 ) \leq \delta_1$ and $\IP(R_{2\star}(\hatphi) > \alpha_2 ) \leq \delta_2$, where $R_{1\star}$ and $R_{2\star}$ are approximated using the test set in each round of simulation. Here,  we use multinomial logistic regression to construct the scoring functions $T_1$ and $T_2$, the inputs of Algorithm~\ref{alg:opt_class}. Figure~\ref{fig:sim_error} displays the boxplots of various approximate errors with $t_1$ chosen as the  $k$-th  largest element in $\cT_1\cap (-\infty, \ot_1]$ as $k$ changes. In Figure \ref{fig:sim_error_1} and \ref{fig:sim_error_2}, where the blue diamonds mark the $95\%$ quantiles, we can see that the violation rate of the required error bounds (red dashed lines, representing $\alpha_1$ and $\alpha_2$) is about $5\%$ or less, suggesting our procedure provides effective controls on the errors of concerns. In this case, in most simulation rounds, $\ot_1$ minimizes the empirical error $\tR^c$ computed on $\cS_{2e}$ and $\cS_{3e}$, and $t_1 = \ot_1$ is chosen as the optimal threshold by Algorithm~\ref{alg:opt_class} in the final classifier. We can see this coincide with Figure~\ref{fig:sim_error_c}, which shows that the largest element in $\cT_1\cap (-\infty, \ot_1]$ (i.e., $\ot_1$) minimizes the approximate error $R^c$ on the test set. 
 We note here that the results from other splitting ratios can be found in Supplementary Section~C.2, where we observe that once the sample size for threshold selection reaches about twice the minimum sample size requirement, there are little observable differences in the results. In Supplementary Section~C.3, we also compare with variations in computing the scoring functions to examine the effect of score normalization and calibration, showing that our current scoring functions are ideal for our purpose.

\begin{figure}[!ht]
\centering
     \begin{subfigure}[b]{0.3\textwidth}
    \centering
    \includegraphics[width=\textwidth]{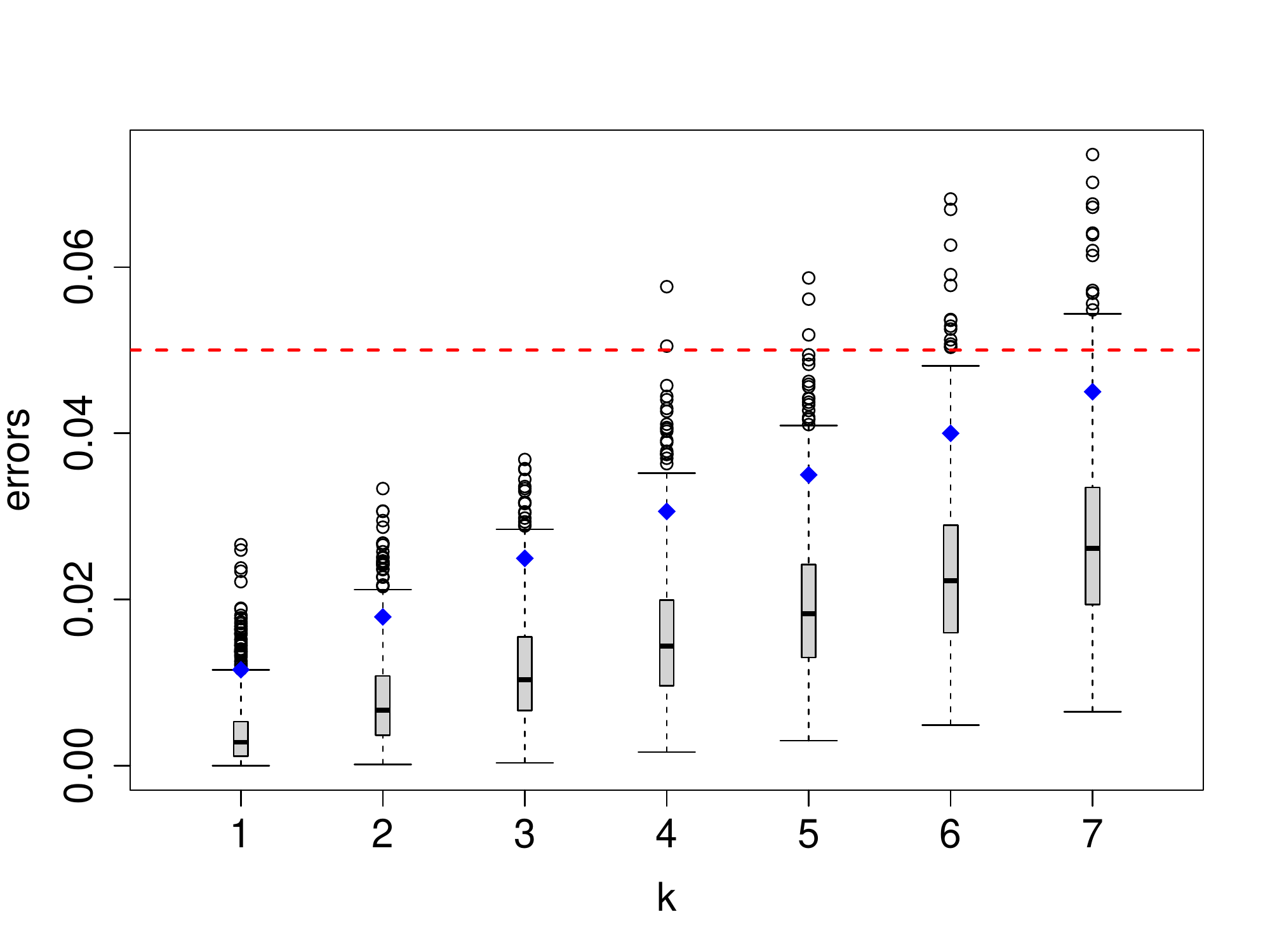}
    \caption{$R_{1\star}$}\label{fig:sim_error_1}
     \end{subfigure}
     \hspace{-0.3cm}
     \begin{subfigure}[b]{0.3\textwidth}
         \centering
    \includegraphics[width=\textwidth]{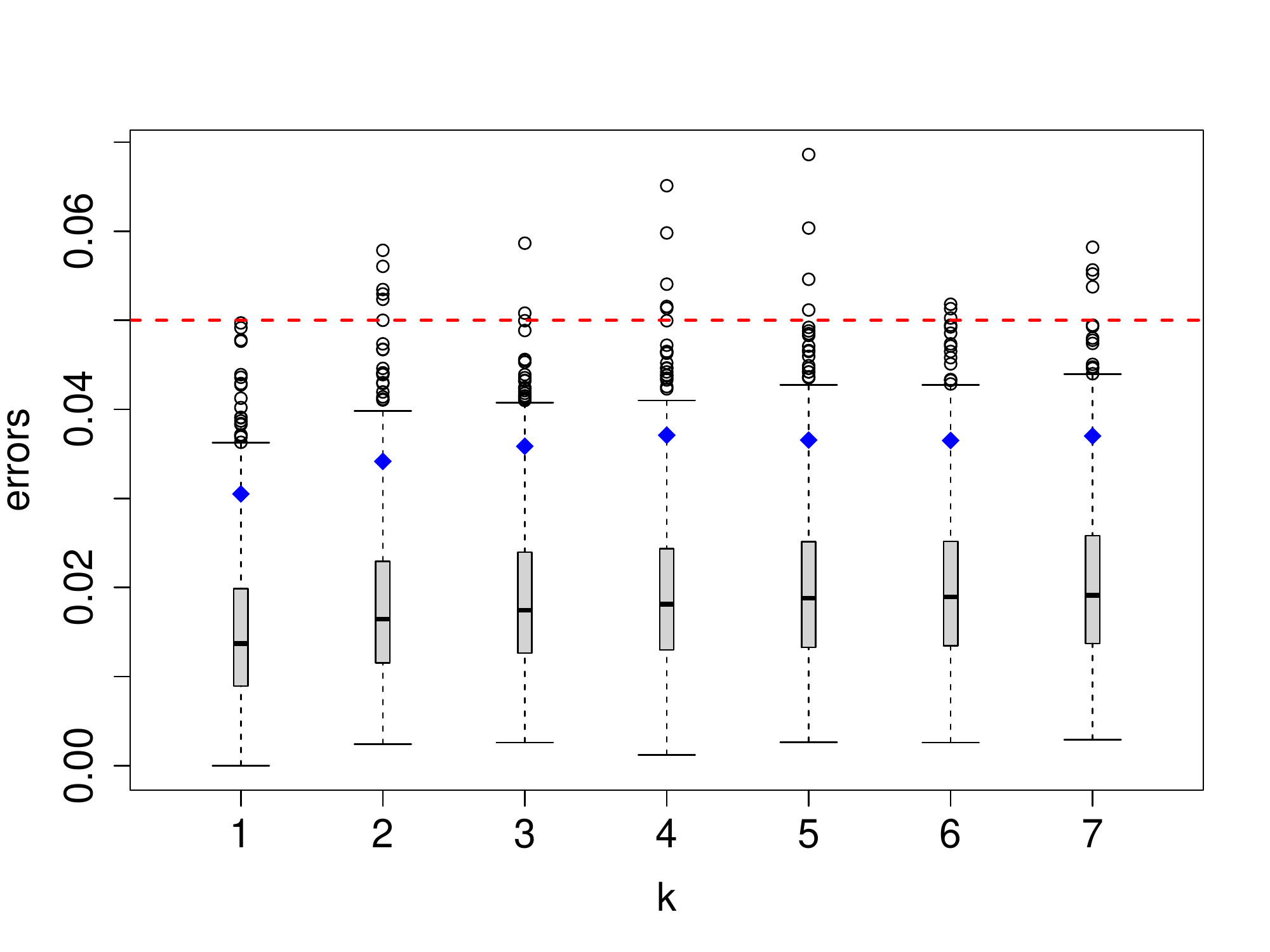} 
   \caption{$R_{2\star}$}\label{fig:sim_error_2}
     \end{subfigure}
     \hspace{-0.3cm}
      \begin{subfigure}[b]{0.3\textwidth}
         \centering
    \includegraphics[width=\textwidth]{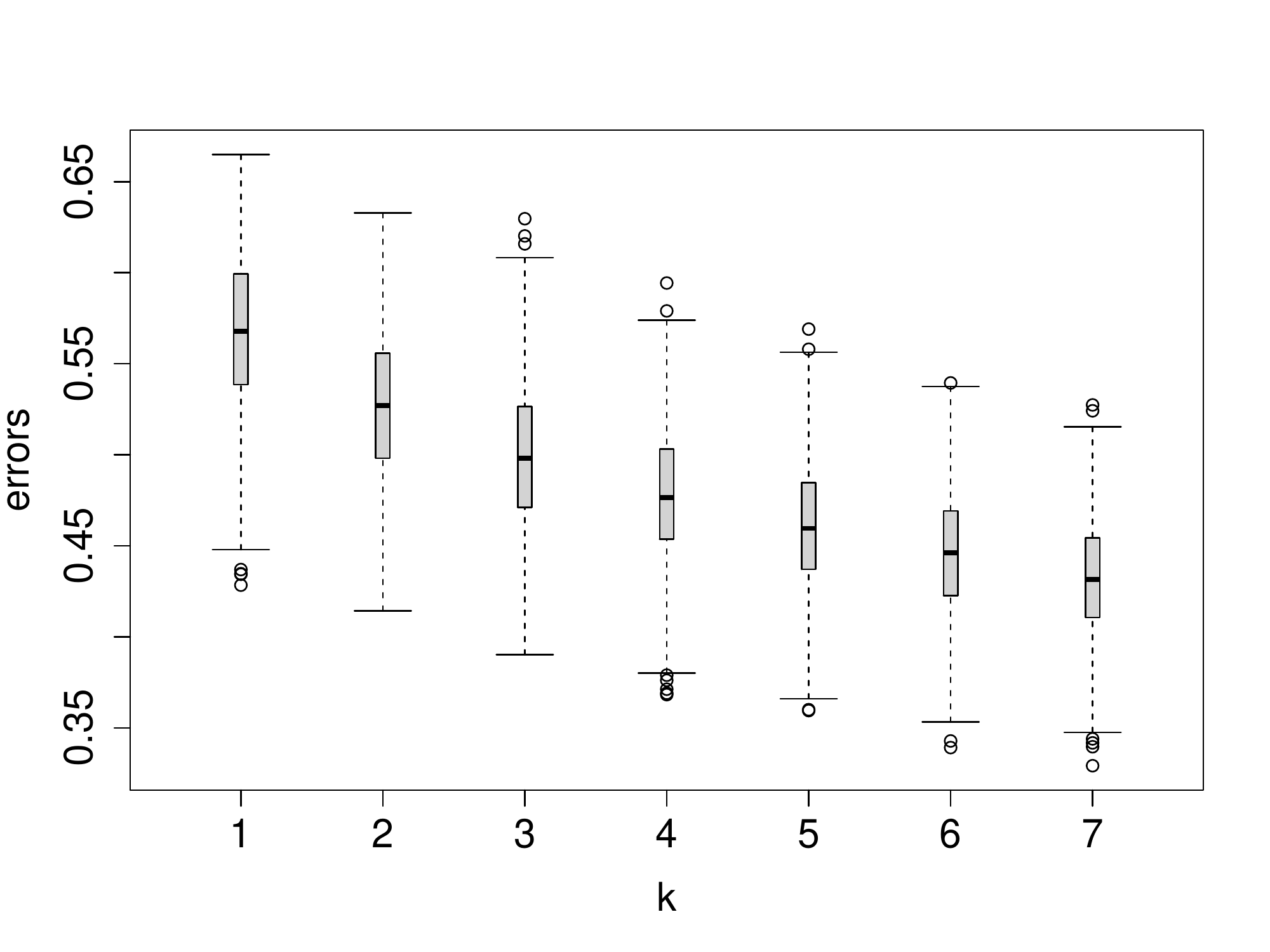}
   \caption{$R^c$}\label{fig:sim_error_c}
     \end{subfigure}
     \caption{\footnotesize The distribution of approximate errors on the test set when $t_1$ is the  $k$-th largest element in $\cT_1\cap (-\infty, \ot_1)$. The $95\%$ quantiles of $R_{1\star}$ and $R_{2\star}$ are marked by blue diamonds.  The target control levels for  $R_{1\star}(\hatphi)$  and  $R_{2\star}(\hatphi)$ ($\alpha_1 = \alpha_2 = 0.05$) are plotted as red dashed lines.  }\label{fig:sim_error}
\end{figure}

Next, we check whether indeed Theorem~\ref{prop:ti} gives a better upper bound on $t_2$ than Proposition~\ref{prop:t1} for overall error minimization. Recall the two upper bounds in Eq~\eqref{eq:control_1} ($t_{2(k_2)}$) and Eq~\eqref{eq:threshold_2} ($\ot_2$). For each base classification algorithm (e.g., logistic regression), we set $t_1=\ot_1$ and $t_2$ equal to these two upper bounds respectively, resulting in two classifiers with different $t_2$ thresholds. 
We compare their performance by evaluating the approximate errors of $R_{2\star}(\hatphi)$ and $P_3(\hat{Y} = 2)$ since, as discussed in Section~\ref{sec:3_class}, the threshold $t_2$ only influences these two errors for a fixed $t_1$. Figure~\ref{fig:sim_t2} shows the distributions of the errors and also their averages for three different base classification algorithms. Under each algorithm, both choices of $t_2$ effectively control $R_{2\star}(\hatphi)$, but the upper bound from Proposition~\ref{prop:t1} is overly conservative compared with that of Theorem~\ref{prop:ti}, which results in a notable increase in $P_3(\hat{Y} = 2)$. This is undesirable since $P_3(\hat{Y} = 2)$ is one component in $R^c(\hatphi)$, and the goal is to minimize $R^c(\hatphi)$ under appropriate error controls.

 \begin{figure}
 \centering
 \begin{minipage}{0.5\linewidth}
		\includegraphics[width=6cm]{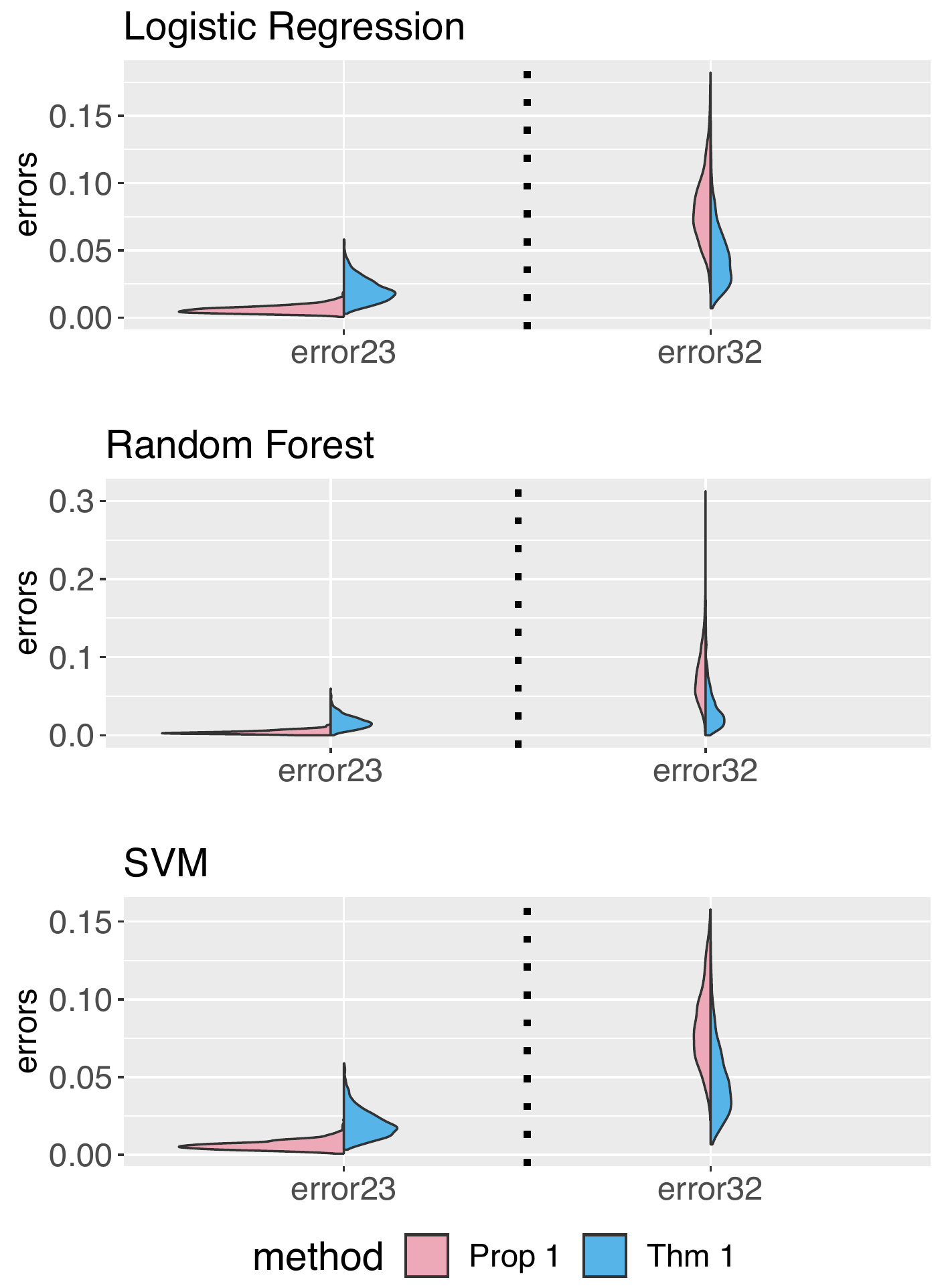}
	\end{minipage}
	\hspace{-1cm}
	\begin{minipage}{0.35\linewidth}
		\centering
			\footnotesize
		\begin{tabular}{c|cc}
			\hline\hline 
			 \multicolumn{3}{c}{Logistic Regression} \\\hline
			    Method & Error23 & Error32 \\\hline
			  Prop 1 & 0.006  & 0.082 \\
			    Thm 1 & 0.020   & 0.046\\\hline\hline
			 \multicolumn{3}{c}{Random Forest} \\\hline
			   Method & Error23 & Error32 \\\hline
			    Prop 1 & 0.004  &  0.077\\
			   Thm 1 & 0.017   & 0.033 \\\hline\hline
			 \multicolumn{3}{c}{SVM} \\\hline
			   Method & Error23 & Error32 \\\hline
			   Prop 1 &  0.006 &  0.083\\
			   Thm 1  &  0.020  & 0.047 \\\hline
		\end{tabular}
	\end{minipage}\hfill
	\caption{ \footnotesize The distribution and averages of approximate errors on the test set under the setting \textbf{T1.1}. ``error23" and ``error32" correspond to $R_{2\star}(\hatphi)$ and $P_3(\hat{Y}= 2)$, respectively.}\label{fig:sim_t2}
 \end{figure}

Now we consider comparing our H-NP classifier against alternative approaches. We construct an example of ``approximate" error control using the empirical ROC curve approach. In this case, each class of observations is split into two parts: one for training the base classification method, the other for threshold selection using the ROC curve. Under the setting \textbf{T1.1}, using similar splitting ratios as before,  we separate $\cS_{i}$ into $50\%$ and $50\%$ for $\cS_{is}$  and $\cS_{it}$ for $i = 1,2,3$. The same test set is used. We re-compute the scoring functions ($T_1$ and $T_2$) corresponding to the new split. $t_1$ is selected using the ROC curve generated by $T_1$ aiming to distinguish between class 1 (samples in $\cS_{1t}$) and class $2'$ (samples in $\cS_{2t}\cup\cS_{3t}$) merging classes 2 and 3, with specificity calculated as the rate of misclassifying a class-1 observation into class $2'$. Similarly, $t_2$ is selected using $T_2$ dividing samples in $\cS_{2t}\cup\cS_{3t}$ into class 2 and class 3, with specificity defined as the rate of misclassifying a class-2 observation into class 3. More specifically, in Eq~\eqref{eq:decision_simple} we use $t_1 = \sup\left\{t:\frac{\sum_{X \in \cS_{1t}}\bone\{T_1(X) < t\}}{|\cS_{1t}|}\leq \alpha_1\right\}$
and  $t_2 = \sup\left\{t:\frac{\sum_{X \in \cS_{2t}}\bone\{T_2(X) < t\}}{|\cS_{2t}|}\leq \alpha_2\right\}$ to obtain the classifier for the ROC curve approach. 

The comparison between our H-NP classifier  and the ROC curve approach  is summarized in Figure~\ref{fig:sim_roc}. Recalling $\alpha_i$ and $\delta_i$ are both $0.05$, we mark the $95\%$ quantiles of the {under-classification} errors by solid black lines and the target error control levels  by dotted red lines. First we observe that the $95\%$ quantiles of $R_{1\star}$ using the ROC curve approach well exceed the target level control, with their averages centering around the target. We also see the influence of $t_1$ on the $R_{2\star}$ -- without suitably adjusting $t_2$ based on $t_1$, the control on $R_{2\star}(\hatphi)$ in the ROC curve approach is overly conservative despite it being an approximate error control method, which in turn leads to inflation in error $P_3(\hat{Y}= 2)$. In view of this, we further consider a simulation setting where the influence of $t_1$ on $t_2$ is smaller. The setting \textbf{T2.1} moves samples in class 1 further away from classes 2 and 3 by having $\mu_1 = (0, -3)^\top$, while the other parts remain the same as in the setting \textbf{T1.1}. $\alpha_i, \delta_i$ are still $0.05$. As shown in Figure~\ref{fig:sim_roc_1}, the ROC curve approach does not provide the required level of control for $R_{1\star}$ or $R_{2\star}$.

In Supplementary Sections~C.4-C.6, we include more comparisons with alternative methods with different overall approaches to the problem, including weight-adjusted classification, cost-sensitive learning, and ordinal regression, and show that our H-NP framework is more ideal for our problem of interest.  

\vspace{-0.5cm}

 \begin{figure}
 \centering
 \begin{minipage}{0.5\linewidth}
		\includegraphics[width=6.5cm]{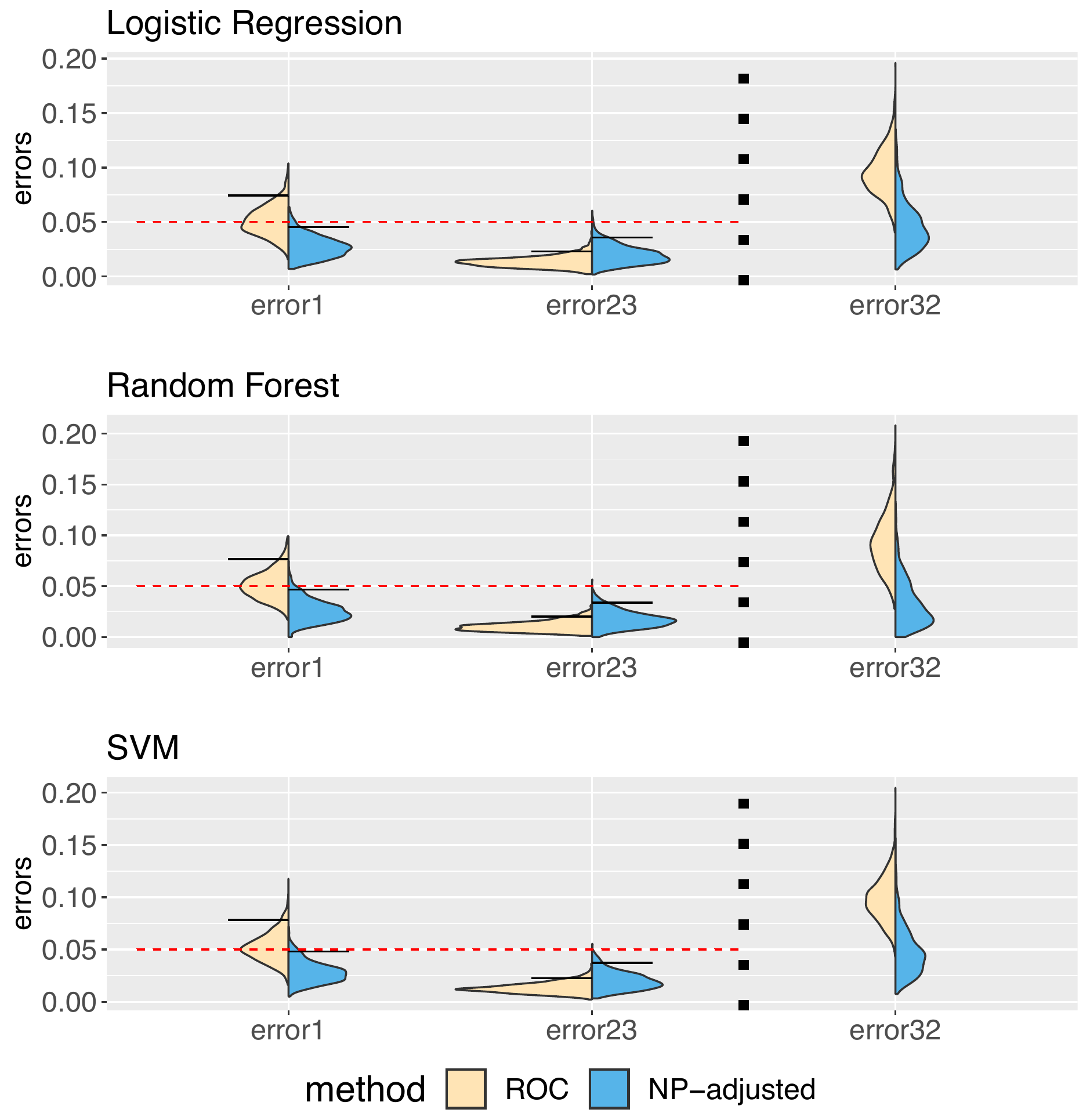}
	\end{minipage}
	\begin{minipage}{0.45\linewidth}
		\centering
			\footnotesize
		\begin{tabular}{c|ccc}
			\hline\hline 
			 \multicolumn{4}{c}{Logistic Regression} \\\hline
			    & Error1 & Error23 & Error32  \\
			    Method &\multicolumn{2}{c}{($95\%$ quantile)} &(mean) \\\hline
			  ROC & 0.074  & 0.023 & 0.096\\
			   H-NP & 0.045 &  0.036 & 0.047\\\hline\hline
			 \multicolumn{4}{c}{Random Forest} \\\hline
			    & Error1 & Error23 & Error32  \\
			    Method  &\multicolumn{2}{c}{($95\%$ quantile)} &(mean) \\\hline
			  ROC & 0.077 &  0.020 & 0.093\\
			  H-NP& 0.047  & 0.034 & 0.032\\\hline\hline
			 \multicolumn{4}{c}{SVM} \\\hline
			    & Error1 & Error23 & Error32  \\
			    Method  &\multicolumn{2}{c}{($95\%$ quantile)} &(mean) \\\hline
			  ROC &  0.078  &  0.023 & 0.098 \\
			  H-NP&  0.048 & 0.037 & 0.047 \\\hline
		\end{tabular}
	\end{minipage}\hfill
	\caption{ \footnotesize The distributions of approximate errors on the test set under setting \textbf{T1.1}. ``error1", ``error23" and ``error32" correspond to  $R_{1\star}(\hatphi)$, $R_{2\star}(\hatphi)$ and $P_3(\hat{Y}= 2)$, respectively.}\label{fig:sim_roc}
 \end{figure}

\begin{figure}[h!]
\centering
 \includegraphics[width=15cm]{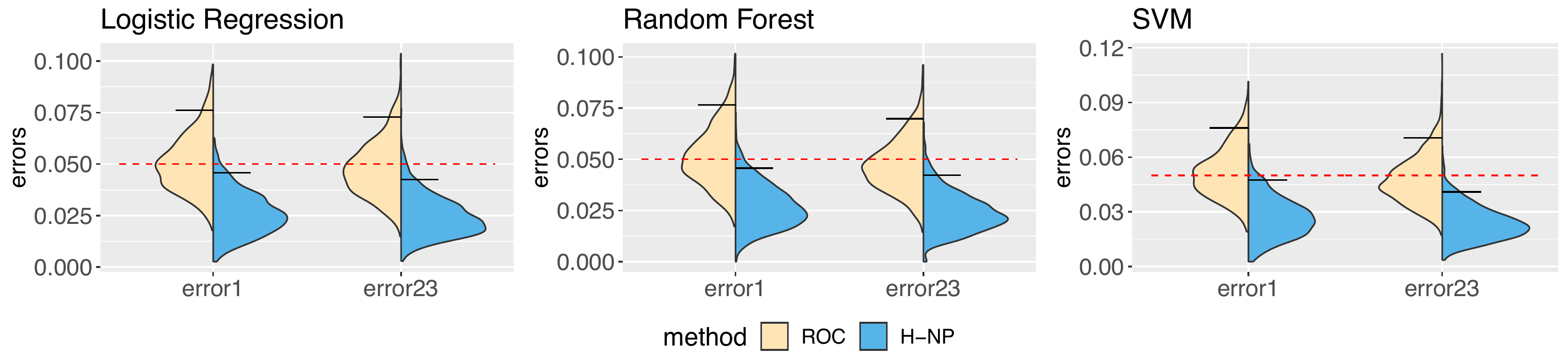}
 \caption{\footnotesize The distributions of approximate errors on the test set under setting \textbf{T2.1}. ``error1" and ``error23"  correspond to the errors $R_{1\star}(\hatphi)$ and $R_{2\star}(\hatphi)$, respectively.}\label{fig:sim_roc_1}
\end{figure}

\section{Application to COVID-19 severity classification}\label{sec:COVID}
\vspace{-0.1in}

\subsection{ScRNA-seq data and featurization}\label{sec:feature}

We integrate 20 publicly available scRNA-seq datasets 
 to form a total of $864$ COVID-19 patients with three severity levels marked as ``Severe/Critical" (318 patients), ``Mild/Moderate" (353 patients), and ``Healthy" (193 patients). The detail of each dataset and patient composition can be found in Supplementary Table~S1. The severe, moderate and healthy patients are labeled as class 1, 2 and 3, respectively. 

For each patient, PBMC scRNA-seq data is available in the form of a matrix recording the expression levels of genes in hundreds to thousands of cells. Following the workflow in \cite{lin2022scalable}, we first perform data integration including cell type annotation and batch effect removal, before selecting $3{,}000$ highly variable genes and constructing their pseudo-bulk expression profiles under each cell type, where each gene's expression is averaged across the cells of this type in every patient. The resulting processed data for each patient $j$ is a matrix $A^{(j)} \in \R^{n_g \times n_c}$, where $n_c=18$ is the number of cell types, and $n_g=3{,}000$ is the number of genes for analysis. More details of the integration process can be found in Supplementary Section~A. Supplementary Figure~S1 shows the distribution of the sparsity levels, i.e., the proportion of genes with zero values, under each cell type across all the patients. 
Several cell types, despite having a significant proportion of zeros, have varying sparsity across the three severity classes (Supplementary Figure~S3), suggesting their activity level might be informative for classification. Since age information is available (although in different forms, see Supplementary Table~S4) in most of the datasets we integrate, we include it as an additional clinical variable for classification.
The details of processing the age variable are deferred to Supplementary Section~A.

Since classical classification methods typically use feature vectors as input,  appropriate featurization that transforms the expression matrices into vectors is needed. We propose four ways of featurization that differ in their considerations of the following aspects.

\begin{itemize}
    \item As we observe the sparsity level in some cell types changes across the severity classes, we expect different treatments of zeros will influence the classification performance. Three approaches are proposed: 1) no special treatment (\ref{method:NS}); 2) remove individual zeros but keeping all cell types (\ref{method:FIL}); 3) remove cell types with significant amount of zeros across all three classes (\ref{method:NPCA} and \ref{method:FLL}).
    
    \item Dimension reduction is commonly used to project the information in a matrix onto a vector. We consider performing dimension reduction along different directions, namely row projections, which take combinations of genes (\ref{method:NPCA}), and column projections, which combine cell types with appropriate weights (\ref{method:FLL} and \ref{method:FIL}). We aim to compare choices of projection direction, so we focus on principal component analysis (PCA) as our dimension reduction method.

    \item We consider two approaches to generate the PCA loadings: 1) overall PCA loadings (\ref{method:NPCA} and \ref{method:FIL}), where we perform PCA on the whole data to output a loading vector for all patients; 2) patient-specific PCA loadings (\ref{method:FLL}), where PCA is performed  for each matrix $A^{(j)}$ to get an individual-specific loading vector.
\end{itemize}

The details of each featurization method are as follows.
\begin{enumerate}[label= M.\arabic*]
    \item\label{method:NS}  \textbf{Simple feature screening:} we consider each element $A^{(j)}_{uv}$ (gene $u$ under cell type $v$) as a possible feature for patient $j$ and use its standard deviation across all patients, denoted as $SD_{uv}$, to
    screen the features. Elements that hardly vary across the patients are likely to have a low discriminative power for classification.
    Let $SD_{(i)}$ be the $i$-th largest element in $\{SD_{uv}\mid u \in [n_g], v \in [n_c]\}$. The feature vector for each patient consists of the entries in $\{ A^{(j)}_{uv} \mid SD_{uv} \geq SD_{(n_f)} \}$, where $n_f$ is the number of features desired and set to $3{,}000$. 
    
    \item\label{method:NPCA}  \textbf{Overall gene combination:} removing cell types with mostly zero expression values across all patients (details in Supplementary Section~A), we select 17 cell types   to construct $\Tilde{A}^{(j)} \in \R^{n_g \times 17}$ that only preserves columns in $A^{(j)}$ corresponding to the selected cell types. Then, $\Tilde{A}^{(1)},\ldots, \Tilde{A}^{(N)}$  are concatenated column-wise to get $\Tilde{A}^{\mathrm{all}} \in \R^{n_g \times (N \times 17)}$, where $N=864$. Let $\Tilde{w} \in \R^{n_g\times 1}$ denote the first principle component loadings of $(\Tilde{A}^{\mathrm{all}})^\top$, and the feature vector for patient $j$ is given by $X_j= \Tilde{w}^\top \Tilde{A}^{(j)}$. 

     \item\label{method:FLL}\textbf{Individual-specific cell type combination:} for patient $j$, the loading vector $\Tilde{w}_j \in \R^{1 \times 17}$ is taken as the absolute values of first principle component loadings for $\Tilde{A}^{(j)}$, the matrix with selected 17 cell types in \ref{method:NPCA} (details in Supplementary Section~A).   The principle component loading vector $\Tilde{w}_j$ that produces $X_j= (\Tilde{A}^{(j)} \Tilde{w}_j)^\top$ is patient-specific, intending to reflect different cell type compositions in different individuals.  
     
    
    \item\label{method:FIL}    \textbf{Common cell type combination:} we compute an expression matrix $\overline{A}$ averaged over all patients defined as   
    \vspace{-1cm}
    $$\overline{A}_{uv} = \frac{\sum_{j \in [N]} A^{(j)}_{uv}}{|\{ j \in [N] \mid A^{(j)}_{uv} \neq 0\}|}\,, \vspace{-0.5cm}$$
    where $|\cdot|$ is the cardinality function. Let $w \in \R^{n_c \times 1}$ denote the first principle component loadings of $\overline{A}$, then the feature vector for the $j$-th patient is $X_j= (A^{(j)} w)^\top$. 

\end{enumerate}

We next evaluate the performance of these featurizations when applied as input to different base classification methods for H-NP classification.

\vspace{-0.4cm}

\subsection{Results of H-NP classification}\label{sec:result}

After obtaining the feature vectors and applying a suitable base classification method, we apply Algorithm~\ref{alg:opt_class} to  control
the {under-classification} errors. Recall that $Y=1, 2, 3$ represent the severe, moderate and healthy categories, respectively, and the goal is to control $R_{1\star}(\hatphi)$ and $R_{2\star}(\hatphi)$. In this section, we evaluate the performance of the H-NP classifier applied to each combination of featurization method in Section~\ref{sec:feature} and base classification method (logistic regression, random forest, SVM (linear)), which is used to train the scores ($T_1$ and $T_2$). In each class, we leave out $30\%$ of the data as the test set and split the rest $70\%$ as follows for training the H-NP classifier: $35\%$ and $35\%$ of $\cS_1$ form $\cS_{1s}$ and $\cS_{1t}$;  $35\%$, $25\%$ and $10\%$ of $\cS_2$ form $\cS_{2s}$, $\cS_{2t}$ and $\cS_{1e}$; $35\%$ and $35\%$ of $\cS_3$ form $\cS_{3s}$ and $\cS_{3e}$.
For each combination of featurization and base classification method, we perform random splitting of the observations for 50 times to produce the results in this section.

In Figure~\ref{fig:distribution}, the yellow halves of the violin plots show the distributions of different approximate errors from the classical classification methods; Supplementary Table~S7 records the averages of these errors. In all the cases, the average of the approximate $R_{1\star}$ error is greater than $20\%$, in many cases greater than $40\%$. On the other hand, the approximate $R_{2\star}$ error under the classical paradigm is already relatively low, with the averages around $10\%$. Under the H-NP paradigm, we set  $\alpha_1,\alpha_2 = 0.2$ and $\delta_1,\delta_2 = 0.2$, i.e., we want to control each {under-classification} error under $20\%$ at a $20\%$ tolerance level.

\begin{figure}[!ht]
\centering
    \begin{subfigure}[b]{0.44\textwidth}
    \centering
    \includegraphics[width=\textwidth]{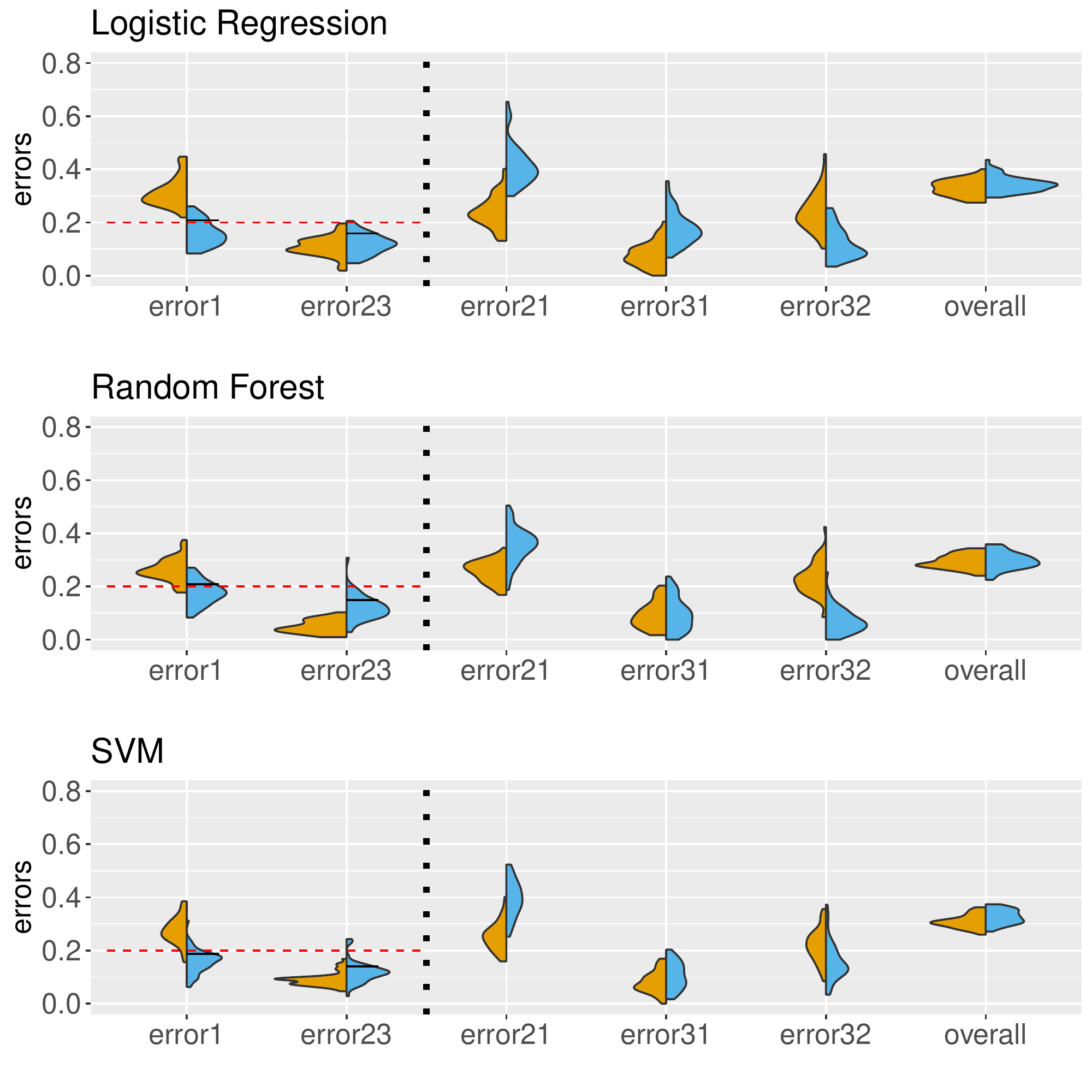}
    \caption{\ref{method:NS}}\label{fig:distribution_NS_1}
     \end{subfigure}
    \begin{subfigure}[b]{0.44\textwidth}
    \centering
    \includegraphics[width=\textwidth]{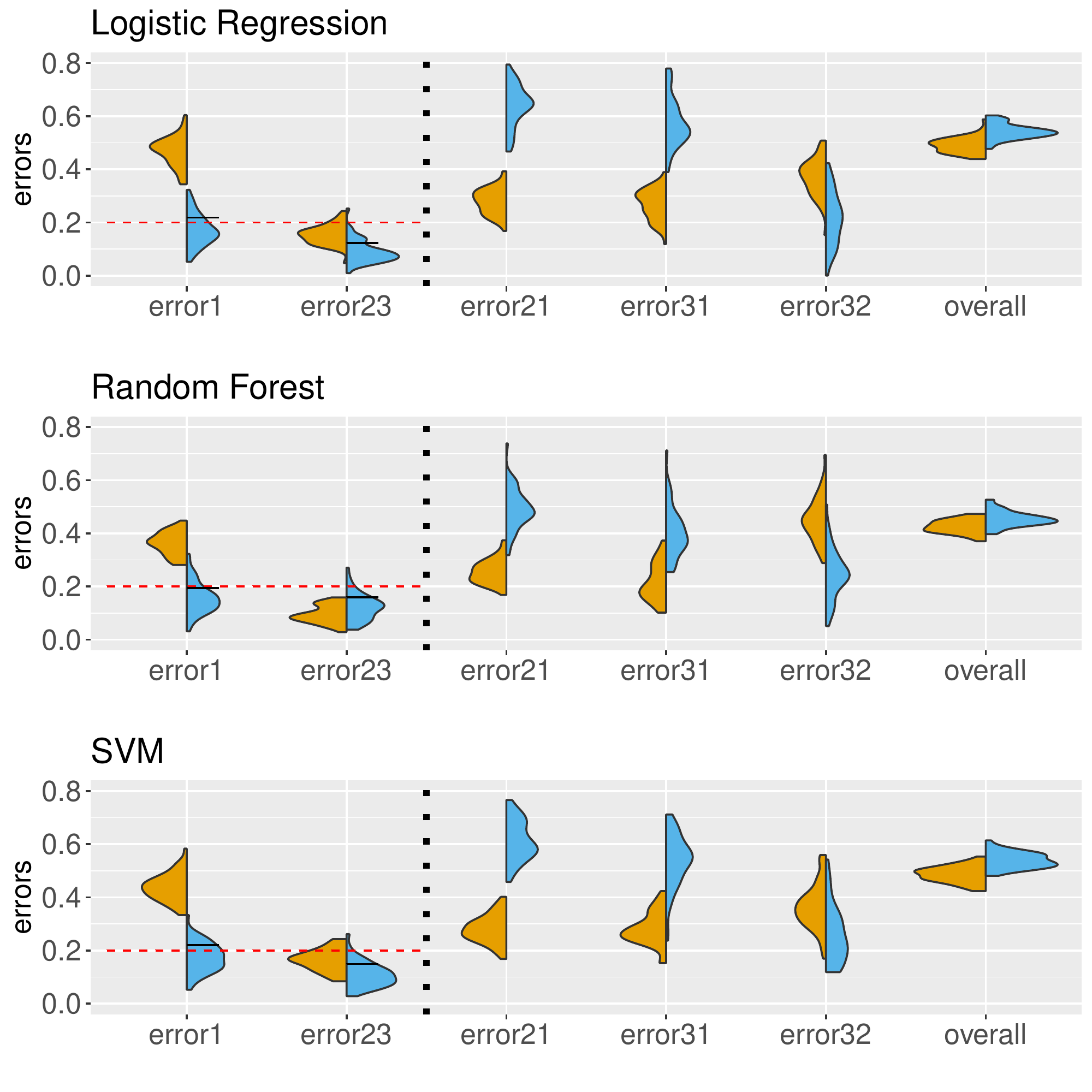}
    \caption{\ref{method:NPCA}}\label{fig:distribution_NPCA_1}
     \end{subfigure}
     \begin{subfigure}[b]{0.44\textwidth}
    \centering
    \includegraphics[width=\textwidth]{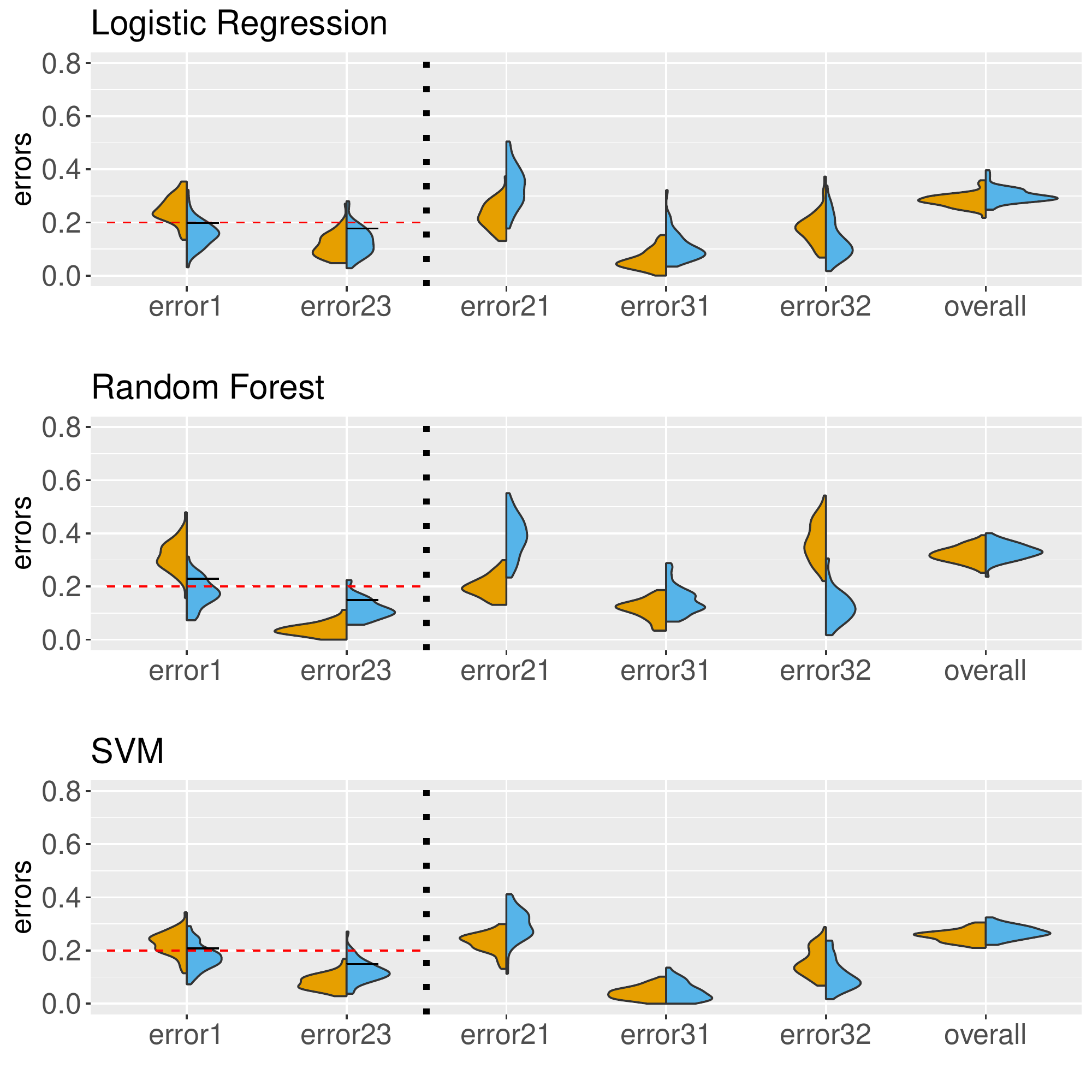}
    \caption{\ref{method:FLL}}\label{fig:distribution_FLL_1}
     \end{subfigure}
     \begin{subfigure}[b]{0.44\textwidth}
    \centering
    \includegraphics[width=\textwidth]{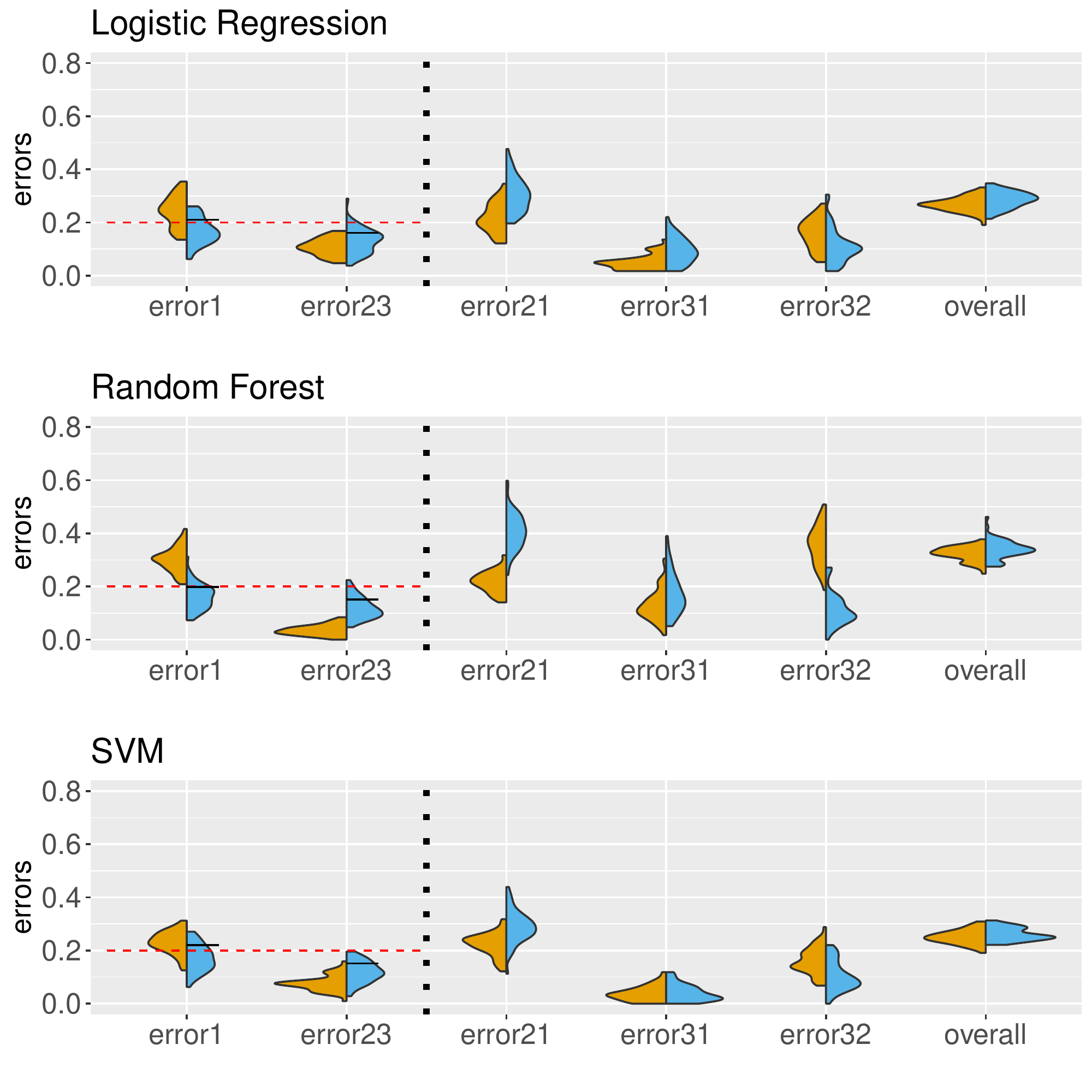}
    \caption{\ref{method:FIL}}\label{fig:distribution_FIL_1}
     \end{subfigure}
     \centering
     \begin{subfigure}[b]{0.2\textwidth}
    \centering
    \includegraphics[width=\textwidth]{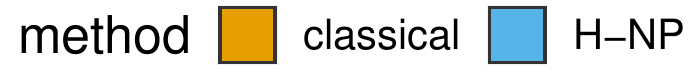}
     \end{subfigure}
     \caption{ \footnotesize The distribution of approximate errors for each combination of featurization method and base classification method. ``error1", ``error23", ``error21", ``error31", ``error32", ``overall" correspond to $R_{1\star}(\hatphi)$, $R_{2\star}(\hatphi)$, $P_2(\hat{Y}= 1)$, $P_3(\hat{Y}= 1)$, $P_3(\hat{Y}= 2)$ and $P(\hat{Y} \neq Y)$, respectively.}\label{fig:distribution}
\end{figure}

With the prespecified $\alpha_1, \alpha_2, \delta_1, \delta_2$, for a given base classification method Algorithm~\ref{alg:opt_class} outputs an H-NP classifier that controls the {under-classification} errors while minimizing the weighted sum of the other empirical errors.  The blue half violin plots in Figure~\ref{fig:distribution} show the resulting approximate errors after H-NP adjustment. We observe that the common cell type combination feature~\ref{method:FIL} consistently leads to smaller errors under both the classical and H-NP classifiers, especially for linear classification models (logistic regression and SVM). We have also implemented a neural network classifier. However, as the training sample size is relatively small, its performance is not as good as the linear classification models, and the results are deferred to Supplementary Figure~S14.

In each plot of Figure~\ref{fig:distribution}, the two leftmost plots are the distributions of the two approximate {under-classification} errors $R_{1\star}$ and $R_{2\star}$. We mark the $80\%$ quantiles of $R_{1\star}$ and $R_{2\star}$ by short black lines (since $\delta_1, \delta_2 = 0.2$), and the desired control levels ($\alpha_1, \alpha_2 = 0.2$) by red dashed lines. The four rightmost plots show the approximate errors for the overall risk and the three components in $R^c(\hatphi)$ as discussed in Eq~\eqref{eq:risk2}. For all the featurization and base classification methods, the {under-classification} errors are controlled at the desired levels with a slight increase in the overall error, which is much smaller than the reduction in {under-classification} errors. This demonstrates consistency of our method and indicates its general applicability to various base classification algorithms chosen by users.


Another interesting phenomenon is that when a classical classification method is conservative for specified $\alpha_i$ and $\delta_i$, our algorithm will increase the corresponding threshold $t_i$, which relaxes the decision boundary for classes less prioritized than $i$. As a result, the relaxation will benefit some components in $R^c(\hatphi)$. In Figure~\ref{fig:distribution_FIL_1}, in many cases the classifier produces an approximate error $R_{2\star}$ less than $0.2$ under the classical paradigm, which means it is conservative for the control level $\alpha_2 = 0.2$ at the tolerance level $\delta_2 = 0.2$. In this case, the NP classifier adjusts the threshold $t_2$ to lower the requirement for class 3, thus notably decreasing the approximate error of $P_3(\hat{Y} = 2)$.

\subsection{{Identifying genomic features associated with severity}}

Finally, we show that using this integrated scRNA-seq data in a classification setting
enables us to identify genomic features associated with disease severity in patients at both
the cell-type and gene levels. First, by combining logistic regression with an appropriate featurization, we generate a ranked list of features (i.e., cell types or genes) that are important in predicting severity. At the cell type level, we utilize logistic regression with the featurization~\ref{method:NPCA}, which compresses the expression matrix for each patient into a cell-type-length vector, and rank the cell types based on their coefficients from the log odds ratios of the severe category relative to
the healthy category. Supplementary Table~S8 shows the top-ranked cell types are CD$14^+$ monocytes, NK cells, CD$8^+$ effector T cells, and neutrophils, all with significant p-values. This is consistent with known involvement of these cell types in the immune response of severe patients \citep{lucas2020longitudinal,liu2020longitudinal,rajamanickam2021dynamic}. 

At the gene level, we utilize logistic regression with the featurization~\ref{method:FIL}, which has the best overall classification performance, and compresses each patient's expression matrix into a gene-length vector. Similar to the above analysis at the cell-type level, we generate a ranked gene list which leads to the identification of pathways associated with the severe condition. By performing the pathway enrichment analysis on the ranked gene list, we find that the top-ranked genes are significantly enriched in pathways involved in viral defense and leukocyte-mediated immune response (Supplementary Table~S9).

{Next, we perform further analysis to directly demonstrate the benefits of the H-NP classification results without relying on feature ranking}. Based on the featurization M.4, we construct a gene co-expression network and identify modules with groups of genes that are potentially co-regulated and functionally related. By comparing the predicted severity labels from the H-NP classifier and the classical approach, we show that the H-NP labels are better correlated with the eigengenes from these functional modules, suggesting that the H-NP labels better capture the underlying signals in the data related to disease mechanism and immune response (Supplementary Figures~S15-S17). Then, we compare the gene ontology enrichment of the functional modules constructed for the severe and healthy patients separately, using the predicted H-NP labels. We find strong evidence of immune response to the virus among
severe patients, while no such evidence is observed in the healthy group (Supplementary Tables~S10 and~S11). {Finally, we note that compared with the results from the severe patients
as labeled by the classical paradigm, the H-NP paradigm shows more significantly
enriched modules with specific references to important cell types, including T cells, and
subtypes of T cells (Supplementary Tables~S10 and~S12). Together, these results demonstrate that by prioritizing
the severe category in our H-NP framework, we can uncover stronger biological signals in
the data related to immune response.}

More detailed descriptions of the methods used and analysis of results can be found in Supplementary Sections~D.4 and~D.5.


\vspace{-0.2in}

\section{Discussion}
\vspace{-0.1in}

In general disease severity classification, {under-classification} errors are more consequential as they can increase the risk of patients receiving insufficient medical care. By assuming the classes have a prioritized ordering, we propose an H-NP classification framework and its associated algorithm (Algorithm~\ref{alg:opt_class})  capable of controlling {under-classification} errors at desired levels with high probability. The algorithm performs post hoc adjustment on scoring-type classification methods and thus can be applied in conjunction with most methods preferred by users.  The idea of choosing thresholds on the scoring functions based on a held-out set bears resemblance to conformal splitting methods \citep{lei2014classification, wang2022set}. However, our approach differs in that we assign only one label to each observation, while maintaining high probability error controls. Additionally, our approach prioritizes certain misclassification errors, unlike conformal prediction which treats all classes equally.

Through simulations and the case study of COVID-19 severity classification, we demonstrate the efficacy of our algorithm in achieving the desired error controls. We have also compared different ways of constructing interpretable feature vectors from the multi-patient scRNA-seq data and shown that the common cell type PCA featurization overall achieves better performance under various classification settings. {By performing extensive gene ontology enrichment analysis, we illustrate that the use of scRNA-seq data has allowed us to gain biological insights into the disease mechanism and immune response of severe patients. We note here that although parts of our analysis rely on a ranked feature list obtained from logistic regression, there exist tools to perform such a feature selection step for all the other base classification methods used in this paper, including neural networks, which can utilize saliency maps and other feature selection procedures \citep{adebayo2018sanity, novakovsky2023obtaining}. We have chosen logistic regression in our illustrative analysis based on its stable classification performance and ease of interpretation.
} 
In addition, if the main objective is to build a classifier for triage diagnostics using other clinical variables,
one can easily apply our method to other forms of patient-level COVID-19 data with other
base classification methods.


Even though our case study has three classes, the framework and algorithm developed are general. Increasing the number of classes has no effect on the minimum size requirement of the left-out part of each class for threshold selection since it suffices for each class $i$ to satisfy $n_{i} \geq \log \delta_i/ (1 - \alpha_i)$. We also note that the notion of prioritized classes can be defined in a context-specific way. For example, in some diseases like Alzheimer's disease, the transitional stage is considered to be the most important \citep{xiong2006measuring}. 

There are several interesting directions for future work. For small data problems where the minimum sample size requirement is not full-filled, we might consider adopting a parametric model, under which we can not only develop a new algorithm without minimum sample size requirement, but also study the oracle type properties of the classifiers.  
In terms of featurizing multi-patient scRNA-seq data, we have chosen PCA as the dimension reduction method to focus on other aspects of comparison; 
more dimension reduction methods can be explored in future work. It is also conceivable that the class labels in the case study are noisy with possibly biased diagnosis. Accounting for label noise with a realistic noise model and extending the work of \citet{yao2022asymmetric} to a multi-class NP classification setting will be another interesting direction to pursue.

\section*{Acknowledgements}
The authors would like to thank the Editor, Associate Editor, and two anonymous reviewers for their valuable comments, which have led to a much improved version of this paper. The authors would also like to thank Dr Yingxin Lin and the Sydney Precision Data Science Centre for their generous help with curating and processing the COVID-19 scRNA-seq data. The authors gratefully acknowledge: the UT Austin Harrington Faculty Fellowship to Y.X.R.W. and NSF DMS-2113754 to J.J.L. and X.T. The authors report there are no competing interests to declare.

\vspace{-0.25in}

\bibliographystyle{unsrtnat}
\bibliography{HNPref}

\def\spacingset#1{\renewcommand{\baselinestretch}%
{#1}\small\normalsize} \spacingset{1.9}

\appendix

\appendixpage

\section{Preprocessing of the integrated COVID-19 data}\label{supp.sec:data}


We integrate 20 collections of scRNA-seq datasets from peripheral blood mononuclear cells (PBMCs). A total of 864 patients are available and their severity levels can be found in Table~\ref{supp.table:count}. Table~\ref{supp.table:pop} summarizes populations and geographic locations covered by the datasets. We note that some of these datasets contain patients with longitudinal records; we take only one sample from these multiple measurements to ensure independence.   

Before integration, we performed size factor standardization and log transformation on the raw count expression matrices using the \texttt{logNormCount} function in the R package 
\texttt{scater} (version 1.16.2) \citep{mccarthy2017scater} and generated log transformed gene expression matrices. All the PBMC datasets are integrated by scMerge2 \citep{lin2022atlas}, which is specifically designed for merging multi-sample and multi-condition studies. Following the standard pipeline for assessing the quality of integration, in Figure~\ref{supp.fig:umap_integration}, we show the UMAP projections of all cells from all the studies, obtained from the top 20 principle components of the merged gene-by-cell expression matrix, for (a) before integration and (b) after integration. The cells are colored by their cell types (left column) or which study (or batch) they come from (right column). Before integration, cells from the same cell type are split into separate clusters based on batch labels, indicating the presence of batch effects. After integration, cells from the same cell type are significantly better mixed while the distinctions among cell types are preserved.


To construct pseudo-bulk expression profiles, we input the cell types annotated by 
\texttt{scClassify} \citep{lin2020scclassify} (using cell types in \citet{stephenson2021single} as reference) into scMerge2. The resulting profiles are used to identify mutual nearest subgroups as pseudo-replicates and to estimate parameters of the scMerge2 model. We select the top $3{,}000$ highly variable genes through the function \texttt{modelGeneVar} in R package 
\texttt{scran} \citep{lun2016step}, and for each patient calculate the average expression of each cell type for selected genes, i.e., for each patient, the integrated dataset provides a $n_g\times n_c$ matrix recording the average gene expressions, where $n_g$ is the number of genes ($n_g = 3{,}000$) and $n_c$ is the number of cell types ($n_c = 18$).

\begin{table}[ht]
\centering
\footnotesize
\begin{tabular}{c|ccc |c}
\hline
Publication & Severe/Critical & Mild/Moderate & Healthy & Total\\\hline
\citet{arunachalam2020systems} & 4&3&5&12\\
\cite{bost2021deciphering} &21&6&5&32\\
\cite{covid2021blood} & 62&31&10&103\\
\cite{combes2021global} & 9&11&14&34\\
\cite{lee2020immunophenotyping} & 3&4&5&12\\
\cite{liu2021time} & 30&3&14&47\\
\cite{ramaswamy2021immune}* &-&-&19&19\\
\cite{ren2021COVID} &70&61&20&151\\
\cite{schulte2020severe} &17&19&38&74\\
\cite{schuurman2021integrated} & 2&6&4&12\\
\cite{silvin2020elevated}& 5&2&3&10\\
\cite{sinha2022dexamethasone}& 21&-&-&21\\
\citet{stephenson2021single} &28&53&32&113\\
\citet{su2020multi} &12&117&-&129\\
\cite{thompson2021metabolic}& 5&-&3&8\\
\cite{unterman2022single}* &10&-&-&10\\
\citet{wilk2020single}&11&20&8&39 \\
\cite{yao2021cell}& 6&5&-&11\\
\cite{zhao2021single}& 1&8&10&19\\
\citet{zhu2020single} &1&4&3&8\\\hline
Total &  318 & 353 & 193 & 864\\\hline
\end{tabular}
\caption{Number of patients under each severity level in each dataset. The datasets marked with * were utilized by both studies \citep{ramaswamy2021immune,unterman2022single} in their respective analyses.} \label{supp.table:count}
\end{table}

\begin{table}[h!!!]
\centering
\footnotesize
\begin{tabular}{p{5cm}|p{5cm}|p{2cm}}
\hline
Publication & Population &Country\\\hline
\cite{arunachalam2020systems} & Black, Caucasian &  US\\\hline
\cite{bost2021deciphering} & -&  Italy\\\hline
\cite{covid2021blood} & - & UK\\\hline
\cite{combes2021global} & - &  US\\\hline
\cite{lee2020immunophenotyping} & - & South Korea\\\hline
\cite{liu2021time} &  Asian, Caucasian & Italy\\\hline
\cite{ramaswamy2021immune} &- &   US\\\hline
\cite{ren2021COVID} & Asian & China\\\hline
\cite{schulte2020severe} & -  & Germany\\\hline
\cite{schuurman2021integrated} &  Black, Caucasian & Netherlands\\\hline
\cite{silvin2020elevated}& -  &  France\\\hline
\cite{sinha2022dexamethasone}&  Asian, Black, Caucasian, Others &Canada\\\hline
\citet{stephenson2021single} &- &UK \\\hline
\citet{su2020multi} & Asian, Black, Caucasian, Others  & US\\\hline
\cite{thompson2021metabolic}& - & US\\\hline
\cite{unterman2022single} &- &   US\\\hline
\cite{wilk2020single}&  Asian, Black, Caucasian, Hispanic/Latino, Others &US \\\hline
\cite{yao2021cell}& Asian, Black, Caucasian, Hispanic/Latino, Others &US\\\hline
\cite{zhao2021single}& - & China\\\hline
\citet{zhu2020single} & Asian & China\\\hline

\end{tabular}
\caption{Populations and geographic locations covered by the datasets.}\label{supp.table:pop}
\end{table}

In the featurization methods~M.2 and~M.3, we remove the cell type ILC with its zero proportion hardly changing across all three classes (Figure~\ref{fig:zeros}) and an average zero proportion greater than $95\%$ (Table~\ref{supp.table:avg_missing}). 17 cell types are left:  B, CD14 Mono, CD16 Mono, CD4 T, CD8 T, DC,gdT, HSPC, MAST, Neutrophil, NK, NKT, Plasma, Platelet, RBC, DN, MAIT.  Also, in~M.3, we find that using the absolute values of PCA loadings notably increase the prediction performance under the classical paradigm (even though it is still not as good as~M.4). 

\begin{figure}[h!]
\centering
 \includegraphics[width=15cm]{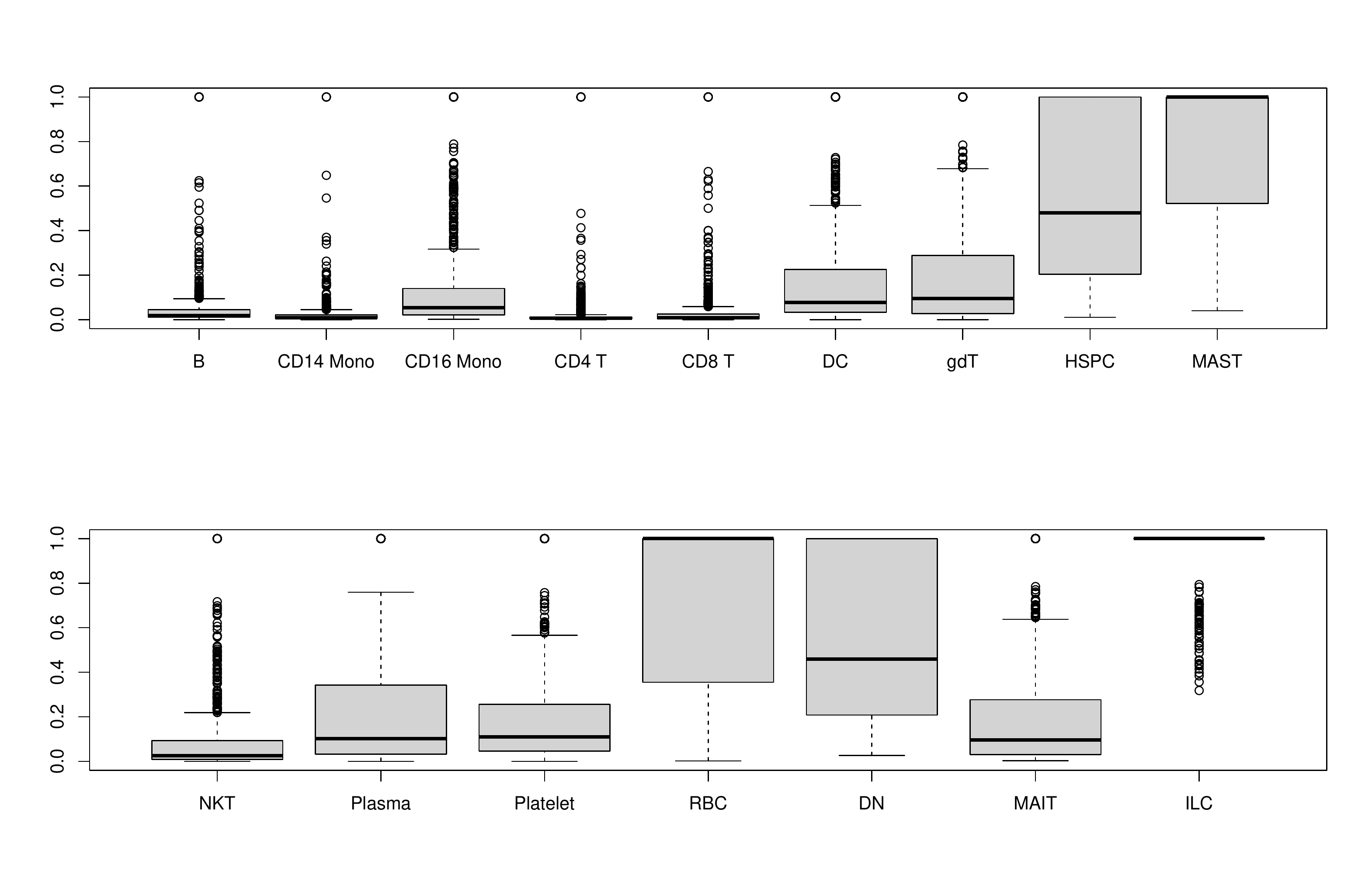}
 \caption{The distribution of the proportion of zero values across patients for each cell type.}\label{supp.fig:zeros_2}
\end{figure}

\begin{table}[h!!!]
\centering
\footnotesize
\begin{tabular}{c|c|c|c}
\hline
cell type & zero proportion & cell type & zero proportion \\
\hline
       B &  0.054    &  Neutrophil& 0.514  \\
       CD14 Mono &0.028& NK  &  0.025  \\
       CD16 Mono &  0.142 & NKT &  0.099\\ 
       CD4 T  & 0.024    &  Plasma & 0.243\\
       CD8 T &0.042  &  Platelet & 0.232\\                    
        DC & 0.186  & RBC    & 0.730 \\
        gdT &  0.209    &  DN  & 0.524\\
        HSPC  &   0.548  &  MAIT   & 0.220\\
         MAST&     0.786     &    ILC     &     0.972 \\ \hline                                            
\end{tabular}
\caption{The average proportion of zero values across patients for each cell type.}\label{supp.table:avg_missing}
\end{table}

\begin{figure}[!ht]
\centering
  \begin{subfigure}[b]{1\textwidth}
    \centering
    \includegraphics[width=\textwidth]{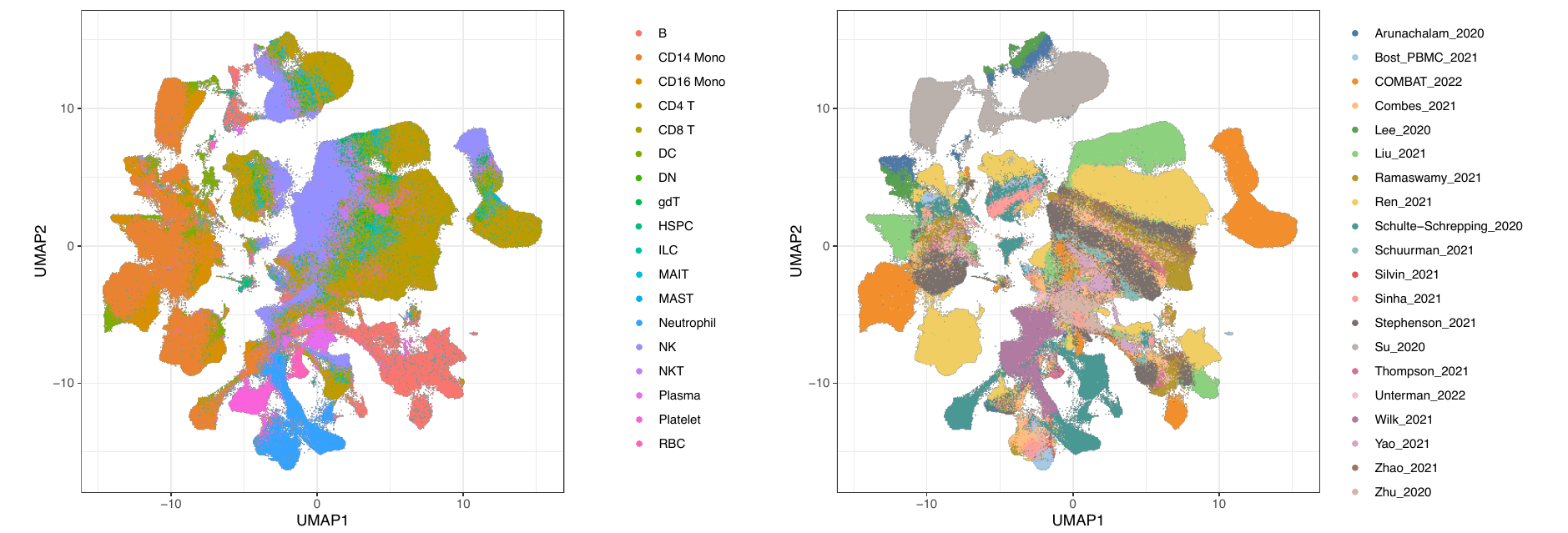}
    \caption{Before integration.}
     \end{subfigure}
  \begin{subfigure}[b]{1\textwidth}
    \centering
    \includegraphics[width=\textwidth]{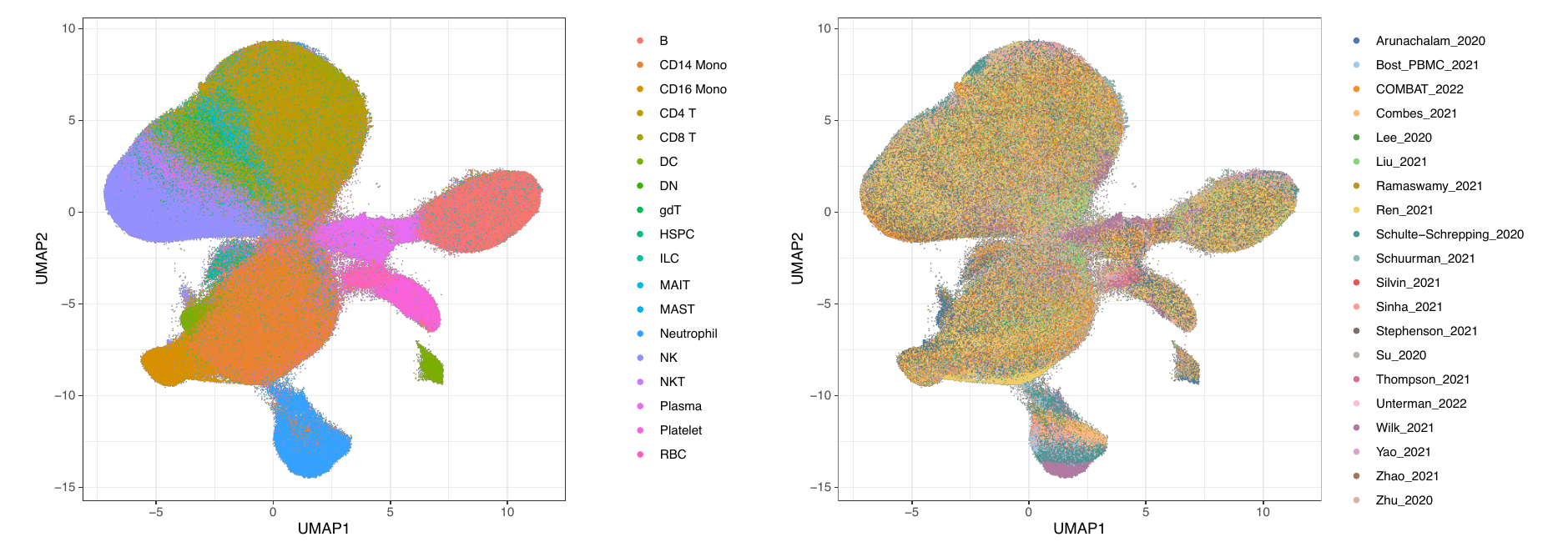}
    \caption{After integration through scMerge.}
     \end{subfigure}
     \caption{ Two left UMAPs plots are colored based on cell types predicted by scClassify (using  \cite{stephenson2021single} as reference). Two right UMAPs plots are colored by the batches.}
     \label{supp.fig:umap_integration}
 \end{figure}

\begin{figure}[h!]
\centering
 \includegraphics[width=16cm]{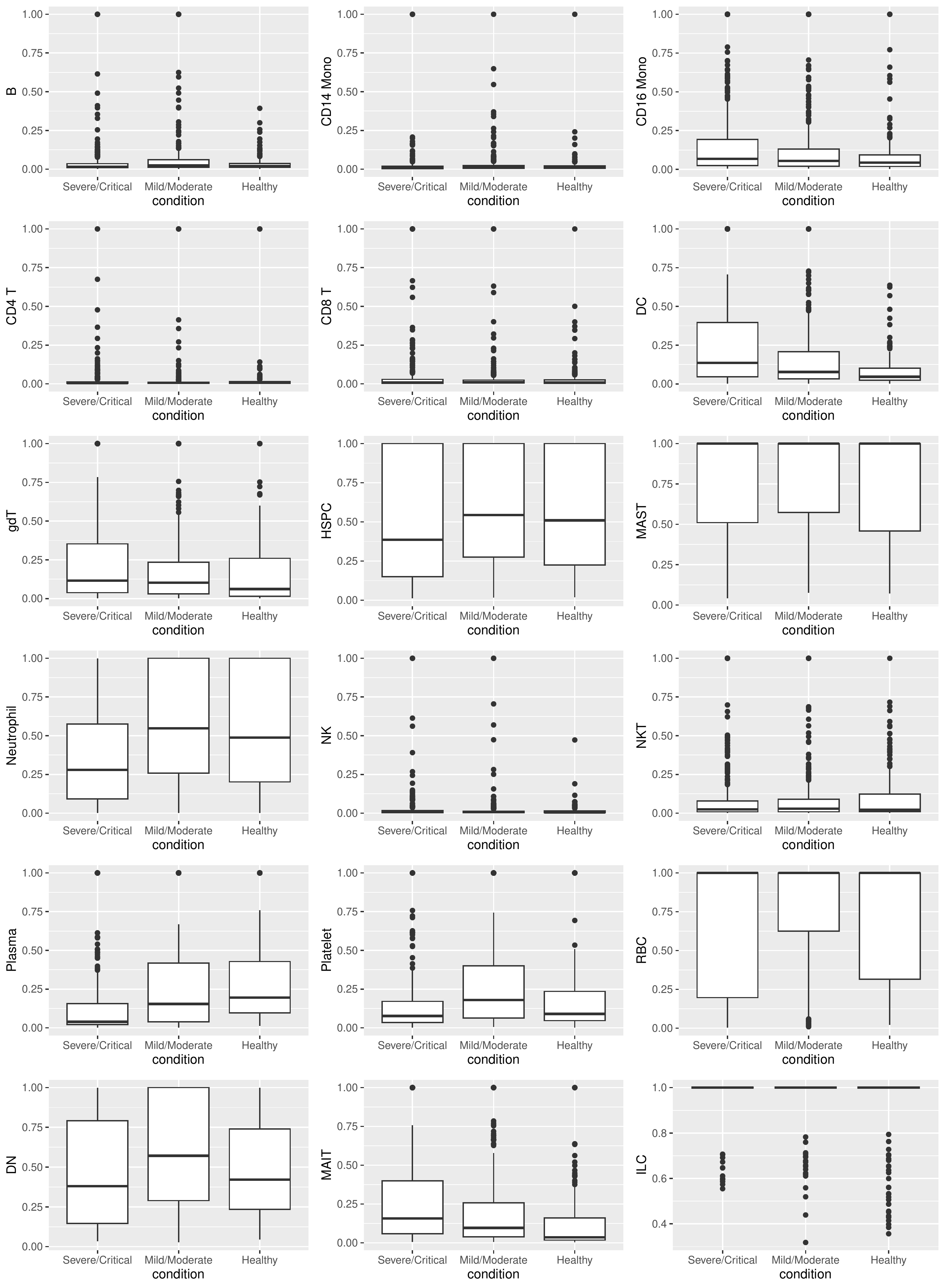}
 \caption{The proportions of zeros for different severity classes.}\label{fig:zeros}
\end{figure}

Other than scRNA-seq data, we also include age as a predictor in the integrated dataset. Most of the datasets used in our study recorded age information either as an exact number or an age group, while the rest did not provide this information (see Table~\ref{supp.table:age}). In the integrated dataset, we use the lower end of the age group recorded for patients with no exact age, and replace the missing values with the average age (52.23).

\begin{table}[h!!!]
\centering
\footnotesize
\begin{tabular}{c|cc}
\hline
Publication & Age recording format & Example\\\hline
\cite{arunachalam2020systems} & exact age & 64\\
\cite{bost2021deciphering} & not available & NA\\
\cite{covid2021blood} & age group&61-70 \\
\cite{combes2021global} & exact age & 64\\
\cite{lee2020immunophenotyping} & exact age & 64\\
\cite{liu2021time} & exact age & 64\\
\cite{ramaswamy2021immune} &exact age & 64\\
\cite{ren2021COVID} &exact age & 64\\
\cite{schulte2020severe} &age group & 61-65 \\
\cite{schuurman2021integrated} & exact age & 64\\
\cite{silvin2020elevated}& exact age & 64\\
\cite{sinha2022dexamethasone}&exact age & 64\\
\citet{stephenson2021single} & age group & 60-69 \\
\citet{su2020multi} &exact age & 64\\
\cite{thompson2021metabolic}& not available & NA\\
\cite{unterman2022single} &exact age & 64\\
\cite{wilk2020single}& age group & 60-69  \\
\cite{yao2021cell}& not available & NA\\
\cite{zhao2021single}& exact age & 64\\
\citet{zhu2020single} &exact age & 64\\\hline

\end{tabular}
\caption{Format of age information in each dataset. An example record for a 64-year-old patient is provided for each dataset.  }\label{supp.table:age}
\end{table}

\clearpage

\section{Proofs of the main results}
\subsection{Proof of Proposition~1}\label{supp.sec:prop_1}

Recall that $\cT_i =  \{ T_i(X ) \mid X \in \cS_{it}\}$, and $t_{i(1)}, \ldots, t_{i(n_i)}$ are the order statistics, with $n_i$ being the cardinality of $\cT_i$. Let $t_{i(k)} $ be the $k$-th order statistic. Suppose $T_i(X)$ is the classification score
of an independent observation from class $i$, and $F_i$ is the cumulative distribution function for $-T_i(X)$. Then, $$P_i\left[\left.T_i(X)< t_{i(k)} \right|  t_{i(k)} \right] =P_i\left[\left.-T_i(X)> -t_{i(k)} \right| t_{i(k)} \right] = 1 - F_i\left(- t_{i(k)} \right),$$ 
and 
\begin{align}\label{eq:threshold}
 & \IP\left(P_i\left[\left.T_i(X)< t_{i(k)} \right| t_{i(k)} \right] > \alpha \right) \notag\\
  & = \IP\left[1 - F_i\left(- t_{i(k)} \right)> \alpha \right]
   =\IP\left[-t_{i(k)} < F_i^{-1}(1 -\alpha)\right] \notag\\
   &= \IP\left[ - t_{i(k)} <  F_i^{-1}( 1- \alpha) , -t_{i(k + 1)}< F_i^{-1}(1 - \alpha), \ldots, - t_{i(n_i)} < F_i^{-1}(1 - \alpha)  \right]\notag\\
   &=   \IP\left[ \mbox{at least $n_i - k + 1$ elements in $\cT_i$ are less than  $F_i^{-1}(1 - \alpha)$} \right] \notag \\
  & = \sum^{n_i}_{j = n_i - k + 1} {n_i \choose j} \IP\left[-T_i(X)  < F_i^{-1}(1- \alpha)\right]^j \left(1 - \IP\left[-T_i(X)  < F_i^{-1}(1 - \alpha)\right]\right)^{n_i - j}\notag\\
  &= \sum^{k - 1}_{j = 0} {n_i \choose j} \left(1 - \IP\left[-T_i(X)  < F_i^{-1}(1 - \alpha)\right]\right)^j \IP\left[-T_i(X)  < F_i^{-1}(1- \alpha)\right]^{n_i - j}  \notag\\ 
  & = (n_i + 1 - k ) {n_i \choose {k- 1}} \int^{\IP\left[-T_i(X)  < F_i^{-1}(1 - \alpha)\right]}_0 u^{ n_i - k} ( 1- u)^{k - 1} du \notag\\
 &   \leq (n_i + 1 - k ) {n_i \choose {k- 1}} \int^{1- \alpha}_0 u^{ n_i - k} ( 1- u)^{k - 1} du \notag\\
 &= \sum^{k - 1}_{j = 0} {n_i \choose j}( \alpha)^{ j} (1  - \alpha)^{n_i - j} =  v(k,n_i,\alpha)
\end{align}
The inequality holds because $ \IP\left[-T_i(X)  < F_i^{-1}(1 - \alpha)\right] \leq 1 - \alpha$, and it becomes an equality when $F_i$ is continuous.

\subsection{Proof of Theorem~1}\label{supp.sec:thm_1}

Given $(t_1,t_2,\ldots, t_{i-1})$, recall that $t_{i(k)}$ and  $t'_{i(k)}$  are the $k$-th order  statistic of the sets $$\cT_i = \{ T_i(X) \mid X \in \cS_{it}\} \qmq{and} \cT'_i = \{ T_i(X) \mid X \in \cS_{it},  T_1(X)< t_1, \ldots, T_{i - 1}(X)< t_{i - 1} \}\,,$$
respectively, where $\cS_{it}$ is the left-out class-$i$ samples, and $(t_1,t_2,\ldots, t_{i-1})$ are the thresholds for the previous decisions in the classifier~(3).  $n_i$ and $n_i'$ are the cardinalities of $\cT_i$ and $\cT'_i$. Obviously, $n_i' \leq n_i$. Also, we set $$\hatp_i = \frac{n_i'}{n_i}\,,\,  p_i =  \hatp_i + c(n_i)\,,\, \alpha_i' = \frac{\alpha_i}{p_i}\,,\, \delta_i' = \delta_i - \exp\{-2n_i c^2(n_i)\}\,,$$
where $\alpha_i$ and $\delta_i$ are the prespecified control level and violation tolerance level, $\alpha_i'$ and $\delta_i'$ are the adjusted counterparts, and $c(n) = \mathcal{O}(1/\sqrt{n})$. With the adjusted $\alpha_i'$ and $\delta_i'$, we consider the following two cases when selecting the upper bound of threshold:
\begin{equation}\label{supp.eq:ot}
\ot_i = \begin{cases}  t'_{i(k'_i)}\,, & \text{if } n_i' \geq \log \delta_i'/ \log (1 - \alpha_i') \qmq{and} \alpha_i' < 1\,; \\
t_{i(k_i)}\,, &\mbox{otherwise}\,,
\end{cases}
\end{equation}
where $k_i = \max \{k \mid v(k,n_i,\alpha_i) \leq \delta_i\} \qmq{and}
    k'_i = \max \{k \mid v(k,n'_i,\alpha'_i)\leq \delta_i'\}\,.$  We are going to prove that $ \IP\left( P_i\left[\left. T_1(X)< t_1, \ldots, T_{i - 1}(X)< t_{i - 1},\qmq{and}T_i (X)< \ot_i\right| \ot_i \right] > \alpha_i \right) \leq \delta_i$ by two cases.
    
\textbf{Case 1:}  We consider the set $ \cE = \{ n_i' \geq \log \delta_i'/ \log (1 - \alpha_i') \qmq{and} 1 - \alpha_i' > 0\}$ (case 1 in Eq~\eqref{supp.eq:ot}). Under this event, and we want to show that $$\IP\left( \left. P_i\left[\left. T_1(X)< t_1, \ldots, T_{i - 1}(X)< t_{i - 1},\qmq{and}T_i (X)< t'_{i(k'_i)} \right|  t'_{i(k'_i)}\right] > \alpha_i \right|\cE \right) \leq \delta_i\,.$$

 Suppose that $T_i(X)$ is the classification score
of an independent observation from class $i$. $F'_i$  is  the cumulative distribution function for the classification score $-T_i(X)$ when $T_1(X)< t_1, \ldots, T_{i - 1}(X)< t_{i - 1} $.  Then, similar to the proof of Proposition~1,
\begin{equation*}
    P_i\left[\left.T_i (X)< t'_{i(k)} \right| t'_{i(k)}\,, T_1(X)< t_1, \ldots, T_{i - 1}(X)< t_{i - 1}\right] = 1 - F'_i\left( -t'_{i(k)} \right)\,.
\end{equation*}
Note that $\alpha_i'$ is determined by $n_i$, $n_i'$, $\alpha_i$. Meanwhile, $\delta'_i$ is fixed, as it only depends on the given values $n_i$ and $\delta_i$. We have
\begin{align}\label{supp.eq:condi_ni}
  & \IP\left( P_i\left[ \left. T_i (X)<  t'_{i(k)} \right| t'_{i(k)}\,, T_1(X)< t_1, \ldots, T_{i - 1}(X)< t_{i - 1} \right] > \alpha'_i \mid n_i'\right)\notag \\
   & = \IP\left[   -t'_{i(k)}   > (F'_i)^{-1} (1 -\alpha'_i)\mid n_i' \right]\notag\\
  & = \frac{\IP\left[   -t'_{i(k)}   > (F'_i)^{-1} (1 -\alpha'_i) \qmq{and} |\cT_i| = n_i' \right]}{\IP(|\cT_i| = n_i')}
\end{align}
Note that the event $\{-t'_{i(k)}   > (F'_i)^{-1} (1 -\alpha'_i) \qmq{and} |\cT_i| = n_i'\}$ indicates that $n_i'$ elements in $\cS_{it}$ satisfy $\{T_1(X)< t_1, \ldots, T_{i - 1}(X)< t_{i - 1}\}$; among these elements, at least $n_i' - k + 1$ elements have $T_i(X)$ less than  $(F'_i)^{-1} (1 -\alpha'_i)$. We can consider $\cS_{it}$ as independent draws from a multinomial distribution with three kinds of outcomes ($A_1,  A_2, A_3$) defined in Supplementary Figure~\ref{supp.fig:multi}. 
\begin{figure}[h!]
\centering
 \includegraphics[width=15cm]{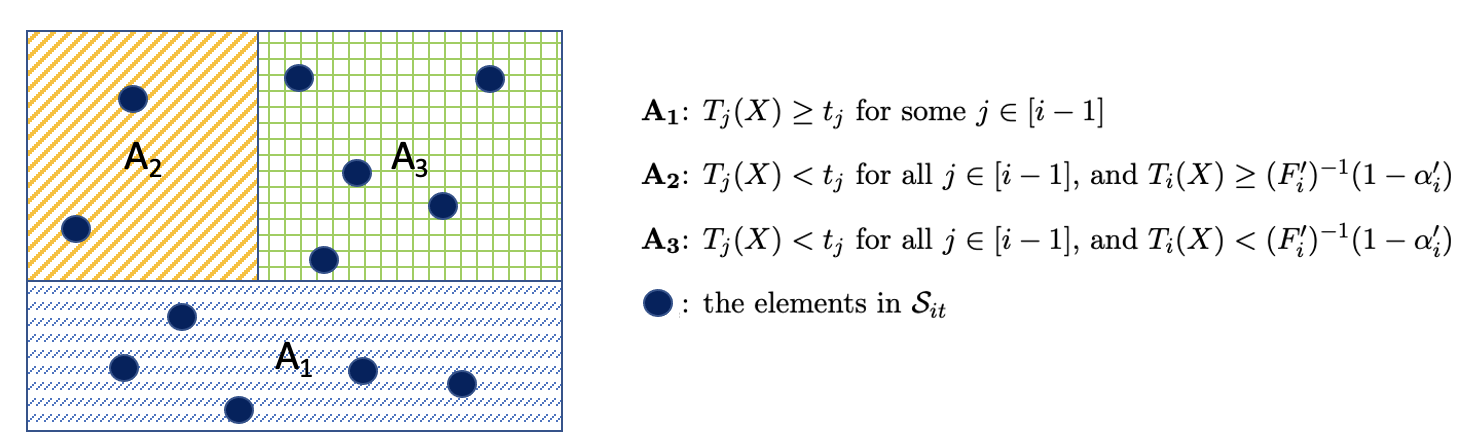}
 \caption{Partition of $\cS_{it}$ into three kinds of outcomes.}\label{supp.fig:multi}
\end{figure}
Therefore, 
\begin{align*}
  &\IP\left[   -t'_{i(k)}   > (F'_i)^{-1} (1 -\alpha'_i) \qmq{and} |\cT_i| = n_i' \right] \\
  &=  \sum^{n'_i}_{j = n'_i - k + 1} {n_i \choose n_i - n_i' \,,n_i' - j\,, j} P_i(A_1)^{n_i - n_i'}P_i(A_2)^{n_i' - j} P_i(A_3)^{j}\,;\\
&\IP(|\cT_i| = n_i') \\
& = {n_i \choose n_i - n_i'} P_i(A_1)^{n_i - n_i'} P_i(A_2 \cup A_3)^{n_i'}\,.
\end{align*}
Then, by Eq~\eqref{supp.eq:condi_ni},
\begin{align*}
& \IP\left( P_i\left[ \left. T_i (X)<  t'_{i(k)} \right| t'_{i(k)}\,, T_1(X)< t_1, \ldots, T_{i - 1}(X)< t_{i - 1} \right] > \alpha'_i \mid n_i'\right)\notag \\
&=  \sum^{n'_i}_{j = n'_i - k + 1} {n_i' \choose j} \left(\frac{P_i(A_2)}{P_i(A_2 \cup A_3)}\right)^{n_i' - j}\left( \frac{P_i(A_3)} {P_i(A_2 \cup A_3)}\right)^{j}\\
&= \sum^{n'_i}_{j = n'_i - k + 1} {n_i' \choose j} (1 -  P_i\left[-T_i(X) < (F_i')^{-1}(1- \alpha) \mid T_1(X)< t_1, \ldots, T_{i - 1}(X)< t_{i - 1}\right])^{n'_i - j} \\
& \qmq{} \times P_i\left[-T_i(X)  < (F_i')^{-1}(1 - \alpha_i')\mid T_1(X)< t_1, \ldots, T_{i - 1}(X)< t_{i - 1}\right]^j\\
&= \sum^{k - 1}_{j = 0} {n_i' \choose j} (1 -  P_i\left[-T_i(X) < (F_i')^{-1}(1- \alpha) \mid T_1(X)< t_1, \ldots, T_{i - 1}(X)< t_{i - 1}\right])^{ j} \\
& \qmq{} \times P_i\left[-T_i(X)  < (F_i')^{-1}(1 - \alpha_i')\mid T_1(X)< t_1, \ldots, T_{i - 1}(X)< t_{i - 1}\right]^{n'_i - j}\\
& \leq \sum^{k - 1}_{j = 0} {n_i' \choose j} (\alpha_i')^j (1- \alpha_i')^{n_i' - j} = v(k,n_i',\alpha_i')\,.
\end{align*}
The last inequality holds because $$P_i\left[-T_i(X)  < (F_i')^{-1}(1 - \alpha_i')\mid T_1(X)< t_1, \ldots, T_{i - 1}(X)< t_{i - 1}\right] \leq 1 - \alpha_i'\,,$$ and it becomes an equality when $(F_i')^{-1}$ is continuous.

Also,
$$\IP\left( P_i\left[ \left. T_i (X)<  t'_{i(k'_i)} \right| t'_{i(k'_i)}, T_1(X)< t_1, \ldots, T_{i - 1}(X)< t_{i - 1} \right] > \alpha'_i \mid n_i'\right) \leq \delta_i'$$
since $k'_i = \max \{k \mid v(k,n'_i,\alpha'_i)\leq \delta_i'\}$.
Then, 
\begin{align}
    & \IP\left(\left.  P_i\left[ \left. T_i (X)<  t'_{i(k'_i)} \right| t'_{i(k'_i)},T_1(X)< t_1, \ldots, T_{i - 1}(X)< t_{i - 1} \right] > \alpha'_i  \right| \cE \right) \notag\\
    &= \E\left[\left. \sum^{k'_i - 1}_{j = 0} {n_i' \choose j} (\alpha'_i)^{j} (1 - \alpha'_i)^{n'_i - j} \right| \cE \right] \label{eq:E_i_bound}\\
    & \leq \E[\delta_i'\mid \cE]  = \delta_i' \notag\,.
\end{align}

On the other hand, note that  the event $\cE = \{ \alpha_i' < 1 \qmq{and} n_i' \geq \log \delta_i'/ \log (1 - \alpha_i') \}$ is equivalent to 
\begin{equation}\label{eq:E_1}
 \left(1 - \frac{\alpha_i n_i}{n_i' + n_ic(n_i)}\right)^{n_i'} \leq \delta_i'  
\end{equation}
and 
\begin{equation}\label{eq:E_i}
 \frac{\alpha_i n_i}{n_i' + n_ic(n_i)} < 1\,.
\end{equation}
The left part of the inequality~\eqref{eq:E_i} is decreasing with respect to $n_i'$. Also, the left part of the inequality~\eqref{eq:E_1} is nonincreasing in $n_i'$ when the inequality~\eqref{eq:E_i} holds, which implies that $$\E[n_i'\mid n_i' \geq \log \delta_i' \log (1 - \alpha_i') \mbox{ and }  \alpha_i' < 1] \geq \E[n_i'\mid n_i' < \log \delta_i'/ \log (1 - \alpha_i') \mbox{ or } \alpha_i' \geq 1]\,,$$
i.e.,  $E[n_i'\mid \cE] \geq E[n_i'\mid\cE^c]$. Immediately,
$$P_i\left( T_1(X)< t_1, \ldots, T_{i - 1}(X)< t_{i - 1}\right) = \E\left[\frac{n_i'}{n_i}\right] \leq \E\left[\left.\frac{n_i'}{n_i}\right|\cE \right] = \E[\hatp_i \mid \cE]\,,$$
so 
\begin{align*}
 &\IP( P_i(T_1(X)< t_1, \ldots, T_{i - 1}(X)< t_{i - 1})   >  p_i  |\cE)\\
 &= \IP(P_i(T_1(X)< t_1, \ldots, T_{i - 1}(X)< t_{i - 1})  - \hatp_i >  c(n_i) |\cE)\\
 & \leq \IP( \E[\hatp_i\mid \cE] - \hatp_i >  c(n_i) |\cE) \\
 & = \IP\left( \left. \E\left[\left.\frac{n'_i}{n_i}\right| \cE\right] - \frac{n'_i}{n_i}>  c(n_i) \right|\cE\right)\\ 
 & = \IP\left( \left. \E\left[\left.n'_i\right| \cE\right] - n'_i>  n_i c(n_i) \right|\cE\right)\\
 & \leq e^{-2n_i c^2(n_i)}\,.
\end{align*}
The last inequality is by Hoeffding's inequality. Then, \begin{align}\label{eq:E_event}
&\IP( P_i ( T_1(X)< t_1, \ldots, T_{i - 1}(X)< t_{i - 1},\qmq{and}T_i(X)< t'_{i(k'_i)}\mid t'_{i(k'_i)}) > \alpha_i \mid \cE) \notag\\
 & = \IP\left[ P_i(T_i (X)< t'_{i(k'_i)} \mid t'_{i(k'_i)},   T_1(X)< t_1, \ldots, T_{i - 1}(X)< t_{i - 1}) \right.\notag\\
 &\qmq{}\left.\times P_i(T_1(X)< t_1, \ldots, T_{i - 1}(X)< t_{i - 1}) > \alpha_i \mid \cE\right] \notag \\
 & \leq \IP[ P_i(T_i (X)< t'_{i(k'_i)} \mid t'_{i(k'_i)}, T_1(X)< t_1, \ldots, T_{i - 1}(X)< t_{i - 1}) > \alpha_i/p_i \mid \cE] \notag\\
  & \qmq{}+ \IP[P_i(T_1(X)< t_1, \ldots, T_{i - 1}(X)< t_{i - 1}) > p_i \mid \cE] \notag \\
  &\leq \delta_i'   + e^{-2n_i c^2(n_i)} =  \delta_i  \,,
\end{align}


\textbf{Case 2:}  We consider the event $\cE^c = \{n_i' < \log \delta_i'/ \log (1 - \alpha_i') \qmq{or} 1 - \alpha'_i \leq 0\}$. Under this event, $\ot_i = t_{i(k_i)}$. Since $n_i$ is deterministic, we have
$$\IP\left(\left. P_i\left[\left.T_i(X)< t_{i(k)} \right| t_{i(k)}\right] > \alpha_i  \right| \cE^c \right) \leq   \sum^{k - 1}_{j = 0} {n_i \choose j}( \alpha_i)^{ j} (1  - \alpha_i)^{n_i - j} = v(k,n_i,\alpha_i)\,.$$
The proof is similar to that of Eq~\eqref{eq:threshold} and~\eqref{eq:E_i_bound}, and even easier since $n_i$ is deterministic. We do not repeat the argument here. Recall that $k_i = \max \{k \mid v(k,n_i,\alpha_i) \leq \delta_i\}$, so 
\begin{multline}\label{eq:t2_all}
 \IP\left(\left. P_i\left[T_1(X)< t_1, \ldots, T_{i - 1}(X)< t_{i - 1},\qmq{and}T_i(X) < t_{i(k_i)}\mid t_{i(k_i)}\right] > \alpha_i \right| \cE^c \right) \\ \leq \IP\left(\left. P_i \left[T_i(X)<t_{i(k_i)}\mid t_{i(k_i)}\right] > \alpha_i \right| \cE^c \right) \leq \delta_i \,.    
\end{multline}
Eq~\eqref{eq:t2_all} implies

 By Eq~\eqref{supp.eq:ot}, \eqref{eq:E_event} and~\eqref{eq:t2_all},
 $$ \IP\left(P_i\left[T_1(X)< t_1, \ldots, T_{i - 1}(X)< t_{i - 1},\qmq{and}T_i(X)<\ot_i \mid \ot_i \right] > \alpha_i\right) \leq \delta_i \,,$$
 which implies the inequality~(10) in Theorem~1.
 
 \subsection{The change of $v(k,n_2',\alpha_2')$ in Figure~1b}\label{supp.sec:v_change}

In the end of Section~2.3, we discuss the selection of $t_1$  and $t_2$ to minimize $R^c$. Recall that $k$ denotes the position of yellow ball (in $\cT'_{2}$) in Figure~1b; $\alpha_2'$ and $\delta_2'$ are the adjusted control level and violation tolerance for the  second {under-classification} error $R_{2\star}(\hatphi)$;  $n_2'$ is the cardinality of $\cT'_{2}$. $k$, $\alpha_2'$ and $n_2'$ are functions of $t_1$. In the following discussion, we will show that the change in $v(k, n_2', \alpha_2')$ with respect to changing $t_1$ is not monotonic. 

We consider a random variable $Z \sim \mbox{Binomial} (n, \alpha)$. We abbreviate $P_{n, \alpha}(z) = P_{n, \alpha}(Z = z)$ and $F_{n, \alpha}(z) = P_{n, \alpha}(Z \leq z)$. Note that $F_{n, \alpha}(z) = v(z, n, \alpha)$. We will explore how the change of $n$ (i.e., the change of $n_2'$) affect the value of $F_{n, \alpha}(z)$. The following lemma discusses the changes in case 1 of Figure~1b.

\begin{lemma}
Let $Z \sim \mbox{Binomial}(n, \alpha_{n} )$, where $\alpha_{n} = \frac{c_1}{n+ c_2}$ for some positive constant $c_1$ and $c_2$, then 
$F_{ n, \alpha_n}(z -1) < F_{ n + 1, \alpha_{n+1}}(z )$.
\end{lemma}
\begin{proof}
Let $Y = Z + X$, where $Z \sim \mbox{Binomial}(n,\alpha)$ and $X \sim \mbox{Bernoulli}(\alpha)$ are independent. Then, $Y \sim \mbox{Binomial}(n + 1,\alpha)$ and 
\begin{align}
    F_{n + 1,\alpha}(z)  &=    P(Y \leq z) =  P(X + Z \leq z) \notag\\
    & = P(X \leq 1 \mid Z \leq z - 1) P(Z \leq z - 1) +  P(X = 0 \mid Z = z) P(Z = z) \notag\\
    & = P(X \leq 1) F_{n, \alpha} (z -1)  +  P(X = 0) P_{n, \alpha}(z) \notag\\
    & = F_{n, \alpha} (z -1)  +  (1 -\alpha) P_{n, \alpha}(z)\,. \label{eq:FF_n_differ}
\end{align}
Also, for fixed $z$ and $n$, we have 
\begin{align*}
    \frac{d  F_{n,\alpha}(z) }{d\alpha} & =  \frac{d}{d\alpha}  (n - z ) {n \choose z} \int^{1 - \alpha}_0 u^{ n - z - 1} ( 1- u)^{z} du\\
   & = - (n - z) {n \choose z}  \alpha^{z} (1 - \alpha)^{ n - z - 1}  = - n P_{n -1, \alpha} (z)\,.
\end{align*}
Therefore, there exists a constant $c \in [\alpha_{n + 1}, \alpha_{n}]$ such that 
\begin{equation}\label{eq:FF_differ}
    \frac{F_{n, \alpha_{n+1}} (z)  - F_{n, \alpha_{n}} (z)}{\alpha_{n} - \alpha_{n + 1} } = n P_{n -1, c} (z)\,.
\end{equation}
Note that $ \alpha_n  - \alpha_{n + 1} > 0$. Then,
\begin{align*}
     &F_{n + 1, \alpha_{n+ 1}}(z) - F_{n, \alpha_{n}}(z - 1)\\
     &= F_{n + 1, \alpha_{n+ 1}}(z) - F_{n, \alpha_{n + 1}} (z -1) + F_{n, \alpha_{n + 1}} (z - 1) - F_{n, \alpha_{n}}(z - 1)\\
     & \geq (1 - \alpha_{n+1}) P_{n ,\alpha_{n+1}}(z) >0.
\end{align*}
\end{proof}

Note that
$$\alpha_2' = \frac{\alpha_2}{p_2} = \frac{\alpha_2}{n_2'/n_2  + c(n_2)} = \frac{\alpha_2 n_2}{n_2'  + c(n_2) n_2}.$$
Since $n_2$ is fixed, we can write $\alpha_2$ in the form of $\frac{c_1}{n_2' + c_2}$ for some positive constants $c_1$ and $c_2$.
In other words, if we remove an element to the left of the yellow element in Figure~1b case 1, the value of $v(k,n_2', \alpha_2'$) will decrease.

By contrast, there are situations in case 2 of Figure~1b that will increase $v(k,n_2', \alpha_2')$, as discussed in the following lemma.

\begin{lemma}
Let $Z \sim \mbox{Binomial}(n, \alpha_{n} )$, where $\alpha_{n} = \frac{c_1}{n+ c_2}$ and $c_2 \geq c_1 > 0$. Fix $z > 0$, $F_{ n, \alpha_n}(z)$ is decreasing with respect to $n \geq n_l$ where  $n_l = \min\{n \mid (n - 1) \alpha_{n+ 1} \geq z\}$.
\end{lemma}

\begin{proof}
Eq~\eqref{eq:FF_n_differ} implies
\begin{equation}\label{eq:F_differ_1}
    F_{n + 1,\alpha}(z) - F_{n,\alpha}(z) = - \alpha P_{n ,\alpha}(z).
\end{equation}
On the other hand, for fixed $n$ and $z$,
\begin{equation}\label{eq:dp_alpha}
   \frac{d  \log(P_{n, \alpha}(z)) }{d\alpha}  =  \frac{z}{\alpha} - \frac{n - z}{1 - \alpha}
   \begin{cases}
    >0 \,, & \alpha < z/n \,;\\
    =0\,, & \alpha = z/n \,;\\
     <0\,, &\alpha > z/n\,\,,
   \end{cases}
\end{equation}
so  $P_{n - 1, c}(z)$ is decreasing on $[\alpha_{n + 1}, \alpha_{n}]$ when $ (n - 1) \alpha_{n + 1} \geq z$. Then, according to Eq~\eqref{eq:FF_differ}
\begin{equation}\label{eq:F_differ_2}
   F_{n, \alpha_{n+1}} (z)  - F_{n, \alpha_{n}} (z) \leq n(\alpha_{n} - \alpha_{n + 1} ) P_{n -1, \alpha_{n+ 1}} (z) \,. 
\end{equation}  

By Eq~\eqref{eq:F_differ_1} and~\eqref{eq:F_differ_2}, for  $ (n - 1) \alpha_{n + 1} \geq z$,
\begin{align*}
    F_{n + 1, \alpha_{n+ 1}}(z) - F_{n, \alpha_{n}}(z) &= F_{n + 1, \alpha_{n+ 1}}(z) - F_{n, \alpha_{n + 1}} (z) + F_{n, \alpha_{n + 1}} (z) - F_{n, \alpha_{n}}
    (z)\\
   & \leq - \alpha_{n+1} P_{n ,\alpha_{n+1}}(z) + n(\alpha_{n}  - \alpha_{n+1}) P_{n -1, \alpha_{n+1}} (z)\\
   & = - {n \choose z} (\alpha_{n+1})^{z + 1} (1 - \alpha_{n+1})^{ n - z} \\
   & \qmq{} +  \frac{n}{ n + c_2}  {n  - 1\choose z} (\alpha_{n+1})^{z + 1} (1 - \alpha_{n+1})^{ n - z - 1}\\
   & = \left( \alpha_{n+1}  - 1 + \frac{n - z}{ n + c_2}\right) {n \choose z} ( \alpha_{n +1})^{z+1} ( 1 -\alpha_{n +1})^{n - z - 1} \\
    & \leq  \left( -\frac{ n+1 + c_2 - c_1 }{n+1 + c_2}   +  \frac{n}{ n + c_2}\right) {n \choose z} ( \alpha_{n +1})^{z+1} ( 1 -\alpha_{n +1})^{n - z - 1} \\
    & =   - \frac{(c_2 - c_1) ( n + c_2) + c_2}{(n + c_2)(n + c_2+1)}  {n \choose z} ( \alpha_{n +1})^{z+1} ( 1 -\alpha_{n +1})^{n - z - 1}\\
    &< 0\,,
\end{align*}
since $ c_2 \geq c_1 > 0$. 

Furthermore,
$$(n - 1)\alpha_{n+1} = \frac{n -1}{ n + 1 + c_2} c_1$$ 
is increasing in $n$, i.e, any $n \geq n_l$ satisfies $(n - 1)\alpha_{n+1} \geq z$. It establishes that $F_{n, \alpha_{n}}(z)$ is decreasing with respect to $n$ when $n \geq n_l$. 
\end{proof}

This lemma presents a situation for case 2 in Figure~1b, where $v(k,n_2', \alpha_2')$ will increase.

\section{Additional results for simulation studies}

\subsection{Summary of simulation settings}

In the simulation studies, we consider $\cI = 3$ and the feature vectors in class $i$ are generated as ${(X^i)}^\top \sim N(\mu_i, I)$. The following simulation settings are used throughout the paper:

\begin{itemize}
    \item Setting \textbf{T1.1}:  $N_i = 500$, $\mu_1 = (0, -1)^\top$, $\mu_2 = (-1, 1)^\top$, $\mu_3 = (1, 0)^\top$. When applying the H-NP method, the observations are randomly separated into parts for score training, threshold selection, and computing empirical errors: $\cS_1$ is split into $50\%$, $50\%$ for $\cS_{1s}$, $\cS_{1t}$;  $\cS_2$ is split into $45\%$, $50\%$ and $5\%$ for $\cS_{2s}$, $\cS_{2t}$ and $\cS_{2e}$; $\cS_3$ is split into $95\%$, $5\%$ for $\cS_{3s}$, $\cS_{3e}$, respectively. 
    \item Setting \textbf{T2.1}: $N_i = 500$, $\mu_1 = (0, -1)^\top$, $\mu_2 = (-1, 1)^\top$, $\mu_1 = (0, -3)^\top$. The data splitting is the same as in the setting \textbf{T1.1} when applying the H-NP method.
    \item Setting \textbf{T3.1}: $N_1 = 1{,}000, N_2 = 200, N_3 = 800$. $\mu_1 = (0, 0)^\top$, $\mu_2 = (-0.5, 0.5)^\top$, and $\mu_3 = (2, 2)^\top$.  The data splitting is the same as in the setting \textbf{T1.1} when applying the H-NP method.
\end{itemize}

All the results in the simulation studies are based on 1{,}000 repetitions from a given setting. To approximate and evaluate the true population errors, we additionally generate $60{,}000$  observations as the test set. The ratio of the three classes in the test set is the same as $N_1:N_2:N_3$.

We further consider simulation settings \textbf{T1.2}-\textbf{T1.7}, which are variations of \textbf{T1.1} with the same values of $\mu_i$. The details of these simulation settings can be found in Table~\ref{supp.fig:setting summary}. In the following subsections, these settings allow us to investigate the impact of different splitting settings, score functions, and imbalanced class sizes on the performance of our H-NP classifier. We also compare the performance of our method with cost-sensitive learning, ordinal classification and weight-adjusted classification.

\begin{table}[h!!!]
\centering
\footnotesize
\begin{tabular}
{c|c|ccc|cccc|ccc}
\hline\hline
        & Class& & 1& & &2 & & & &3 \\\hline
     & & $\cS_{1s}$   & $\cS_{1t}$  & $N_1$ &  $\cS_{2s}$   & $\cS_{2t}$ & $\cS_{2e}$   &$N_2$ & $\cS_{3s}$  & $\cS_{3e}$   & $N_3$ \\ 
      Setting    &  Method   & (\%) & (\%)  & (size) &  (\%)   & (\%) & (\%)  & (size) & (\%) & (\%) & (size)\\ \hline
          \multicolumn{12}{c}{Basic setting}\\\hline
        \multirow{2}{*}{  T1.1 } & classical &  100 & - & $500$ &  100 & - & - & $500$ & 100  &- & $500$ \\
        & H-NP &  50 & 50 & $500$ &  45  & 50 & 5 & $500$ & 95  &5 & $500$  \\ \hline
          \multirow{2}{*}{T1.2} &classical&  100(90,10) & - & $500$ &  100(90,10) & - & - & $500$ & 100(90,10) & - & $500$  \\
        &H-NP&  50(40,10) & 50 & $500$ &  45(35,10) & 50 & 5 & $500$ & 95(85,10) &5 & $500$  \\\hline
         \multicolumn{12}{c}{Different splitting ratios (smaller $\cS_{it}$)}\\\hline
 T1.3& H-NP &  80 & 20 & $500$ &  75 & 20 & 5 & $500$ & 95 &5 & $500$  \\
        T1.4&H-NP &  70& 30 & $500$ &  65 & 30 & 5 & $500$ & 95 &5 & $500$  \\
            T1.5& H-NP & 60 & 40& $500$ &  55 & 40 & 5 & $500$ & 95 &5 & $500$  \\\hline
           \multicolumn{12}{c}{Different splitting ratios (larger $\cS_{it}$)}\\\hline
   T1.6& H-NP & 30 & 70 & $500$ &  25 & 70 & 5 & $500$ & 95 &5 & $500$  \\\hline
    
               \multicolumn{12}{c}{Imbalanced class sizes }\\\hline
      \multirow{2}{*}{T1.7} & classical &  100 & - & 500 &  100 & - & - & 500 & 100 & - & $1{,}000$ \\
      & H-NP&  50 & 50 & 500 &  45 & 50 & 5 & 500 & 95 &5 & $1{,}000$\\\hline
\end{tabular}
\caption{Data splitting and class sizes in the setting \textbf{T1.1} and its variations. The feature vectors in class $i$ are generated as ${(X^i)}^\top \sim N(\mu_i, I)$, where $\mu_1 = (0, -1)^\top$, $\mu_2 = (-1, 1)^\top$, $\mu_3 = (1, 0)^\top$ and $I$ is the $2\times2$ identity matrix. Under \textbf{T1.2}, we use 10\% of the data (split from $\cS_{is}$) to calibrate the  probability estimates $\widehat{\IP}(Y = i \mid X)$ for $i\in[\cI]$ before computing the score $T_i$. To approximate
the true population errors, we additionally generate $60{,}000$ observations and divide them into each class following the ratio $N_1:N_2:N_3$.}\label{supp.fig:setting summary}
\end{table}

\clearpage

\subsection{The influence of different splitting settings}\label{supp.sec:split}

We compare different splitting ratios in simulation settings \textbf{T1.3-T1.6}, assigning larger or smaller proportions of the data to select thresholds (see Table~\ref{supp.fig:setting summary} for more details).

Under the settings \textbf{T1.6} (larger threshold sets) and \textbf{T1.4} (smaller threshold sets), Figures~\ref{supp.fig:sim_error_0} and~\ref{supp.fig:sim_error_3} and show similar trends as Figure~2 in the main paper. As the minimum sample size requirement on $\cS_{it}$ is $59$ in this case, the setting in Figure~2 already exceeds the requirement. Further increasing its size makes minimal difference in the performance of the H-NP classifier.

Figure~\ref{sub.fig:partition_options} shows a comprehensive comparison of settings \textbf{T1.3-1.5} and the basic setting \textbf{T1.1}, which demonstrate that again these splitting ratios lead to no notable changes. We note that the setting \textbf{T1.3} only assigns 100 samples to each threshold selection set $\cS_{it}$, which is about twice the minimum sample size requirement.
\begin{figure}[!h]
     \begin{subfigure}[b]{0.33\textwidth}
    \centering
    \includegraphics[width=\textwidth]{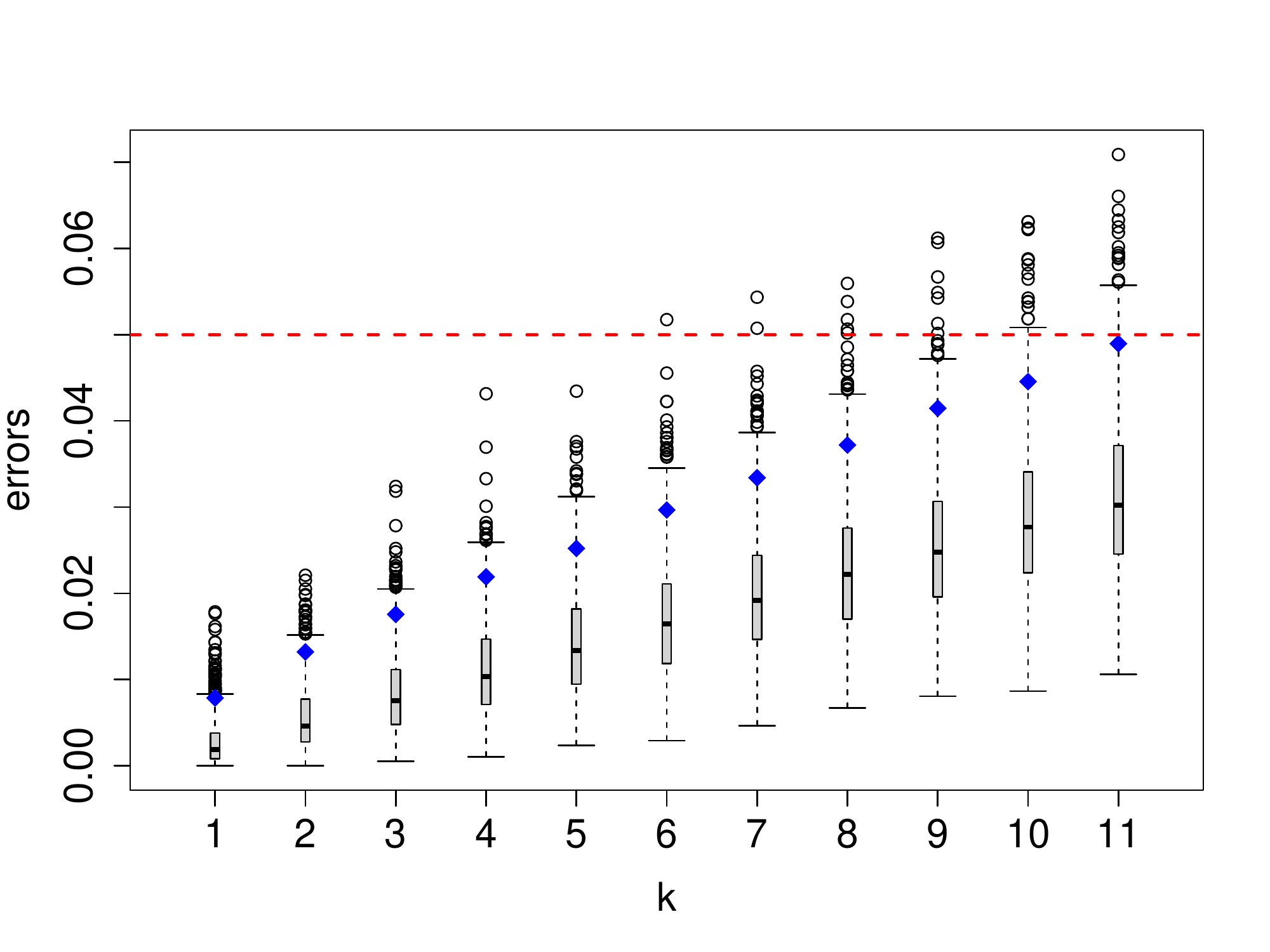}
    \caption{$R_{1\star}$}
     \end{subfigure}
     \hspace{-0.3cm}
     \begin{subfigure}[b]{0.33\textwidth}
         \centering
    \includegraphics[width=\textwidth]{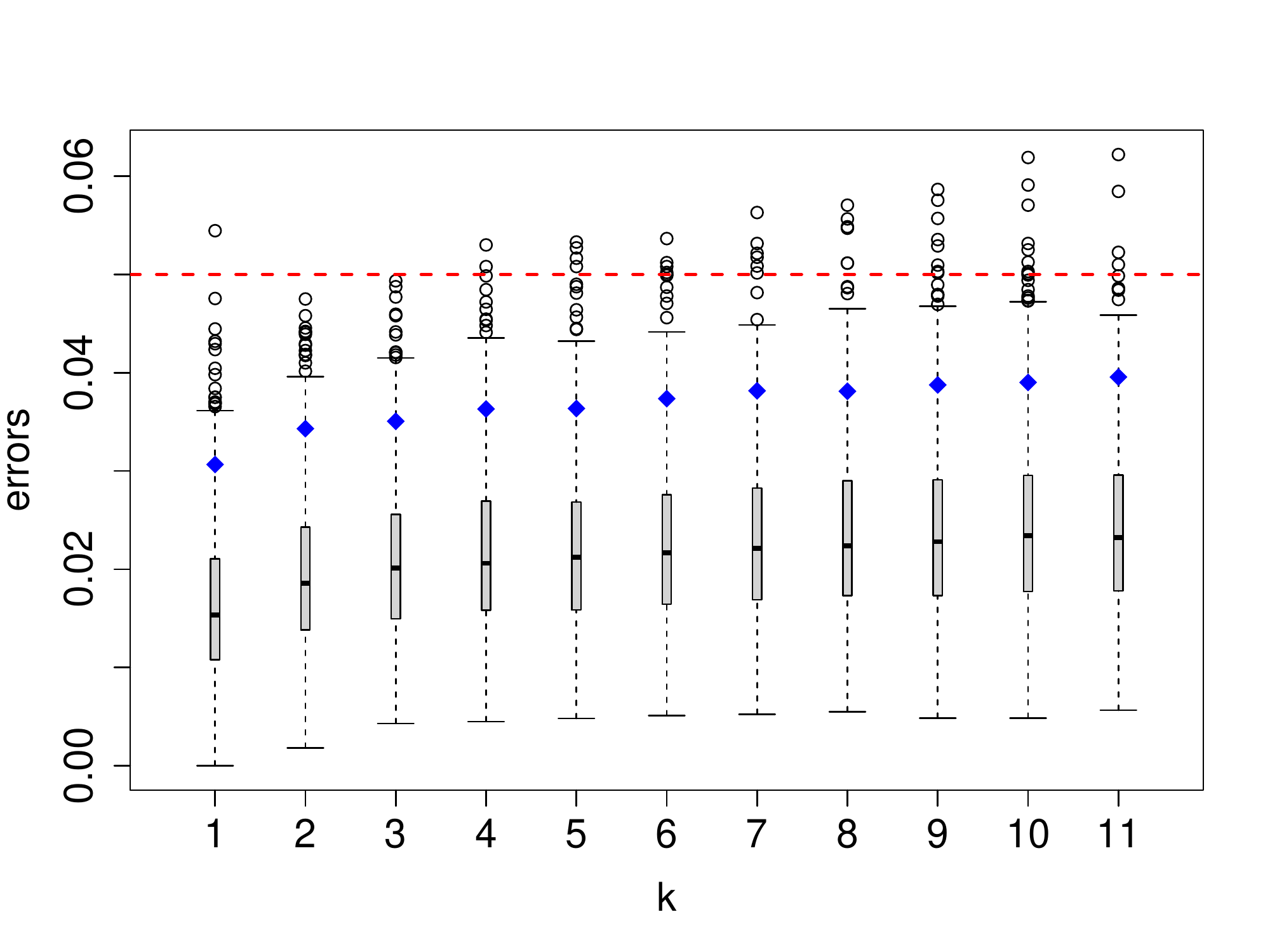} 
   \caption{$R_{2\star}$}
     \end{subfigure}
     \hspace{-0.3cm}
      \begin{subfigure}[b]{0.33\textwidth}
         \centering
    \includegraphics[width=\textwidth]{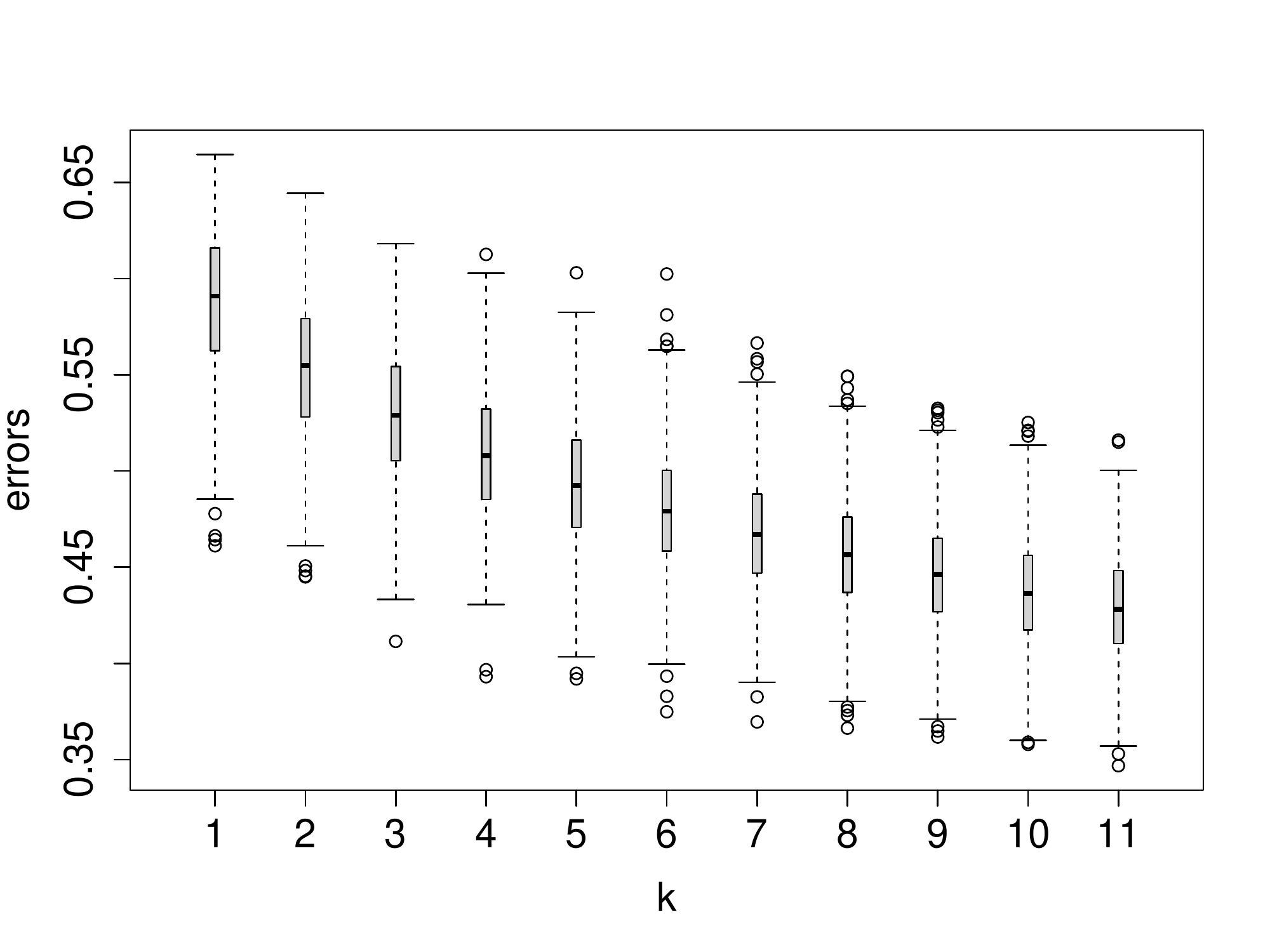}
   \caption{$R^c$}
     \end{subfigure}
     \caption{The distribution of approximate errors when $t_1$ is the  $k$-th largest element in $\cT_1\cap (-\infty, \ot_1)$. The $95\%$ quantiles of $R_{1\star}$ and $R_{2\star}$ are marked by blue diamonds.  The target control levels for  $R_{1\star}(\hatphi)$  and  $R_{2\star}(\hatphi)$ ($\alpha_1 = \alpha_2 = 0.05$) are plotted as red dashed lines.  The data are generated under the setting \textbf{T1.6} (details in Table~\ref{supp.fig:setting summary}). The setting is the same as setting \textbf{T1.1} except a different sample splitting ratio is used: samples in $\cS_1$ are randomly split into $30\%$ for $\cS_{1s}$ (score), $70\%$ for $\cS_{1t}$ (threshold); $\cS_2$ are split into $25\%$ for $\cS_{2s}$, $70\%$ for $\cS_{2t}$, $5\%$ for $\cS_{2e}$ (evaluation);  $\cS_3$ are split into $95\%$ for $\cS_{3s}$ and $5\%$ for $\cS_{3e}$. }\label{supp.fig:sim_error_0}
\end{figure}

\begin{figure}[!h]
     \begin{subfigure}[b]{0.33\textwidth}
    \centering
    \includegraphics[width=\textwidth]{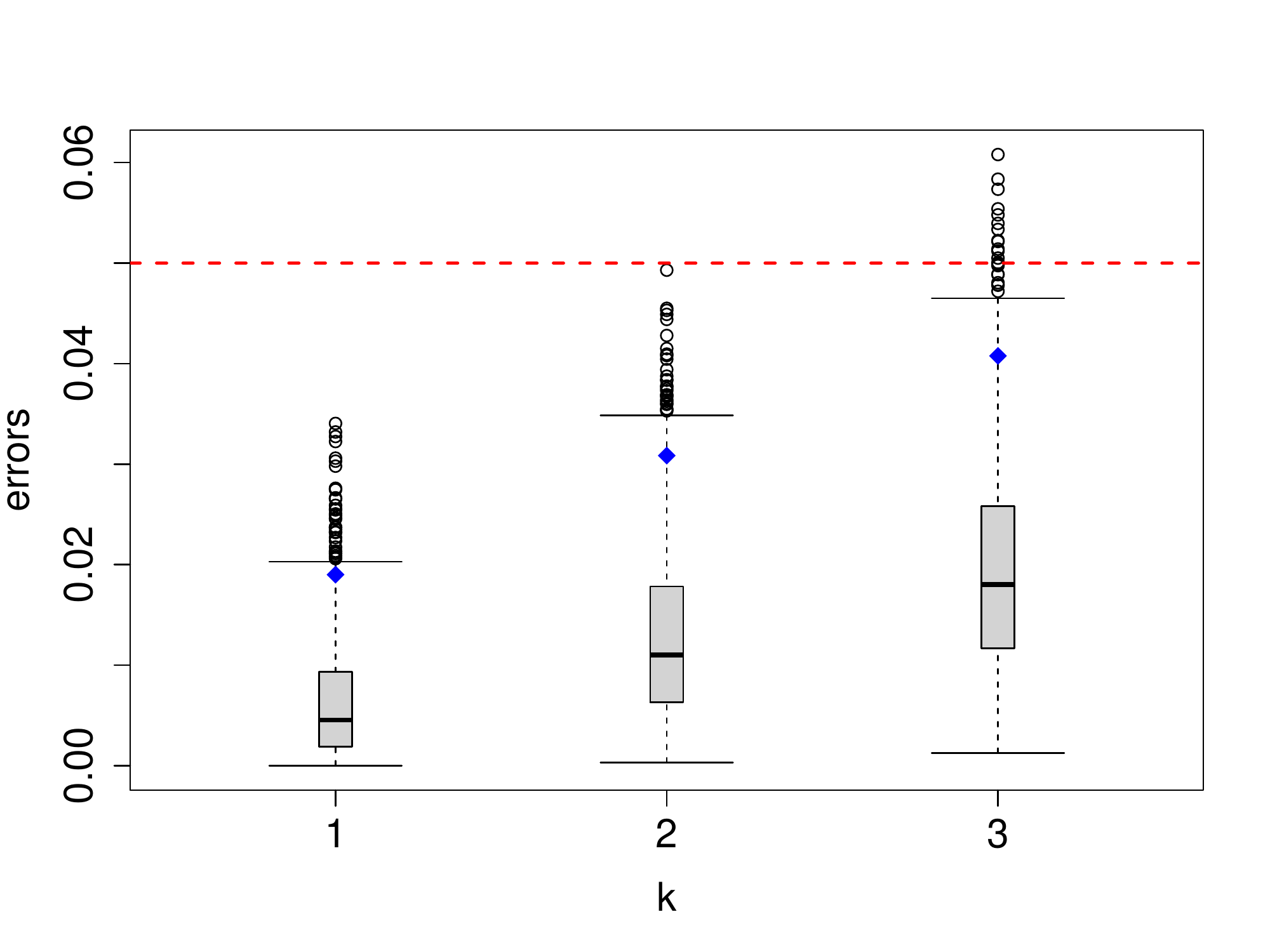}
    \caption{$R_{1\star}$}
     \end{subfigure}
     \hspace{-0.3cm}
     \begin{subfigure}[b]{0.33\textwidth}
         \centering
    \includegraphics[width=\textwidth]{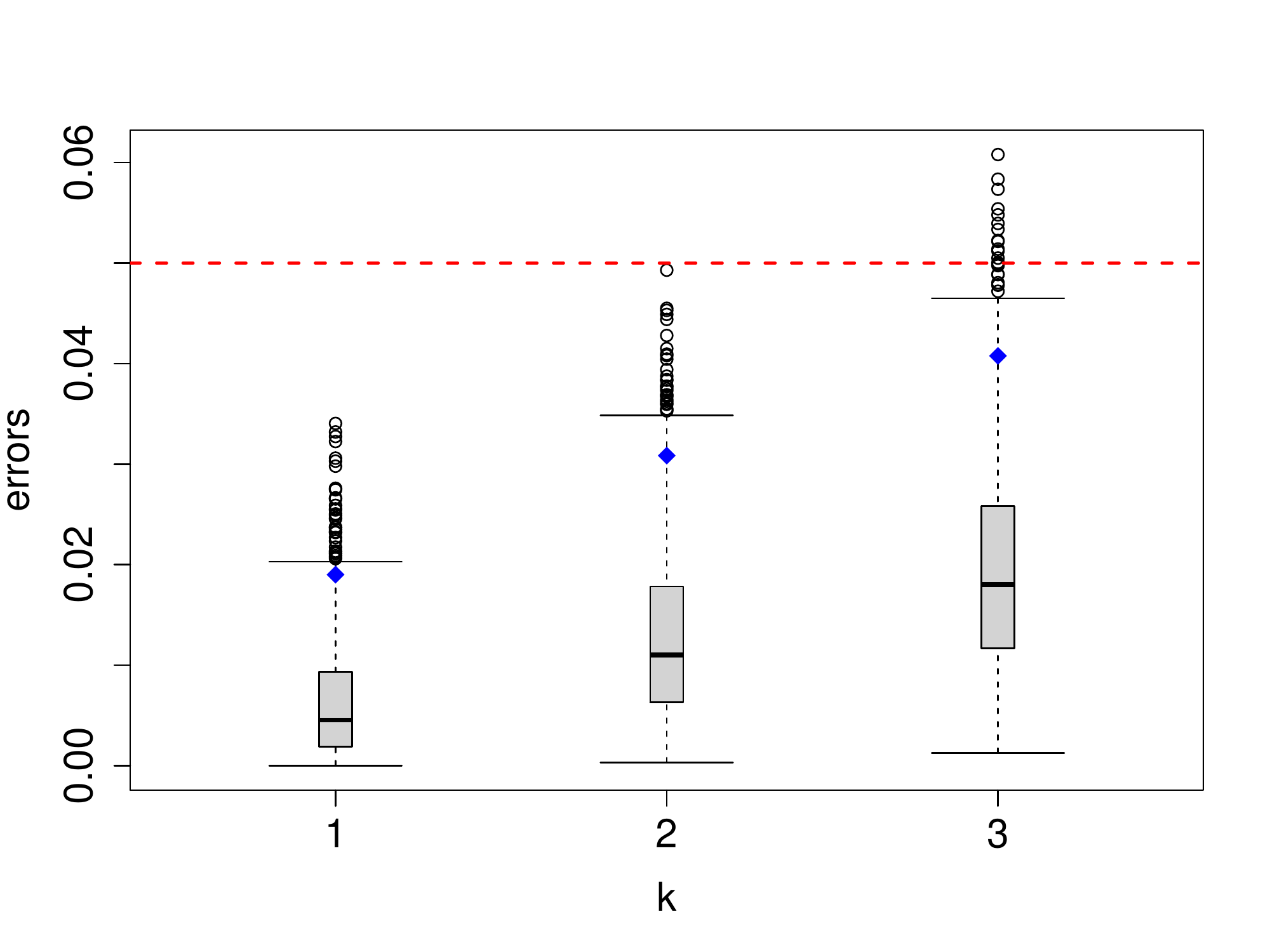} 
   \caption{$R_{2\star}$}
     \end{subfigure}
     \hspace{-0.3cm}
      \begin{subfigure}[b]{0.33\textwidth}
         \centering
    \includegraphics[width=\textwidth]{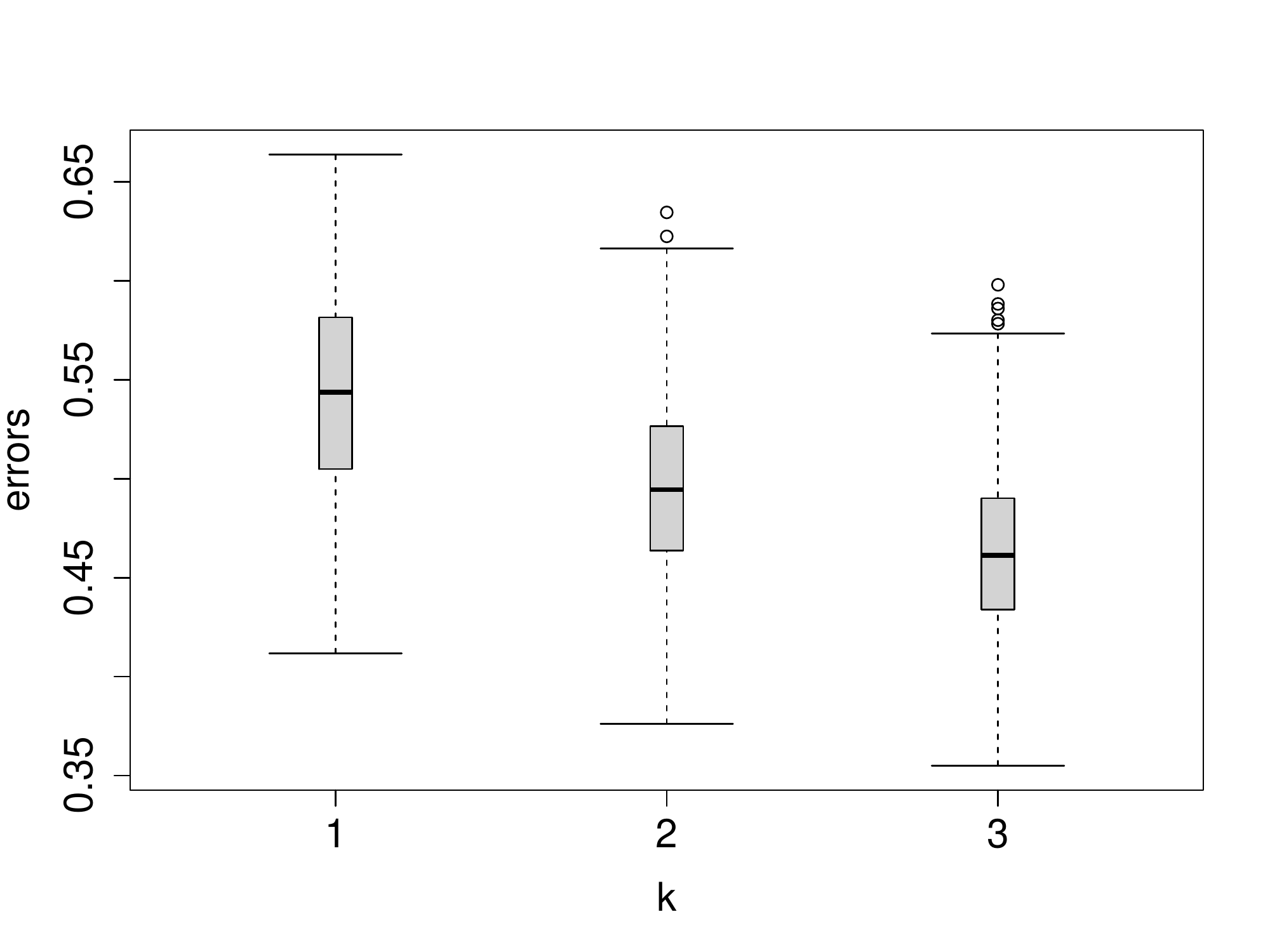}
   \caption{$R^c$}
     \end{subfigure}
     \caption{The distribution of approximate errors when $t_1$ is the  $k$-th largest element in $\cT_1\cap (-\infty, \ot_1)$. The $95\%$ quantiles of $R_{1\star}$ and $R_{2\star}$ are marked by blue diamonds.  The target control levels for  $R_{1\star}(\hatphi)$  and  $R_{2\star}(\hatphi)$ ($\alpha_1 = \alpha_2 = 0.05$) are plotted as red dashed lines.  The data are generated under the setting \textbf{T1.4} (details in Table~\ref{supp.fig:setting summary}). The setting is the same as setting \textbf{T1.1} except a different sample splitting ratio is used: samples in $\cS_1$ are randomly split into $70\%$ for $\cS_{1s}$, $30\%$ for $\cS_{1t}$; $\cS_2$ are split into $65\%$ for $\cS_{2s}$, $30\%$ for $\cS_{2t}$, $5\%$ for $\cS_{2e}$;  $\cS_3$ are split into $95\%$ for $\cS_{3s}$ and $5\%$ for $\cS_{3e}$. }\label{supp.fig:sim_error_3}
\end{figure}

\begin{figure}[!ht]
\centering
  \begin{subfigure}[b]{0.95\textwidth}
    \centering
    \includegraphics[width=\textwidth]{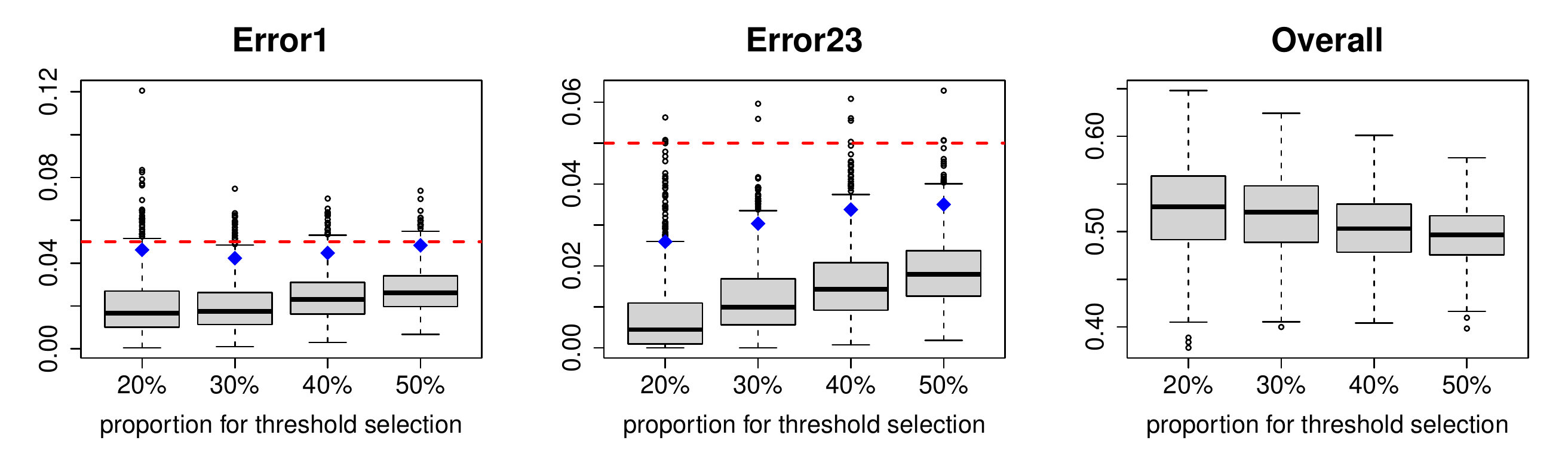}
    \caption{Logistic Regression}
     \end{subfigure}
    \begin{subfigure}[b]{0.95\textwidth}
    \centering
    \includegraphics[width=\textwidth]{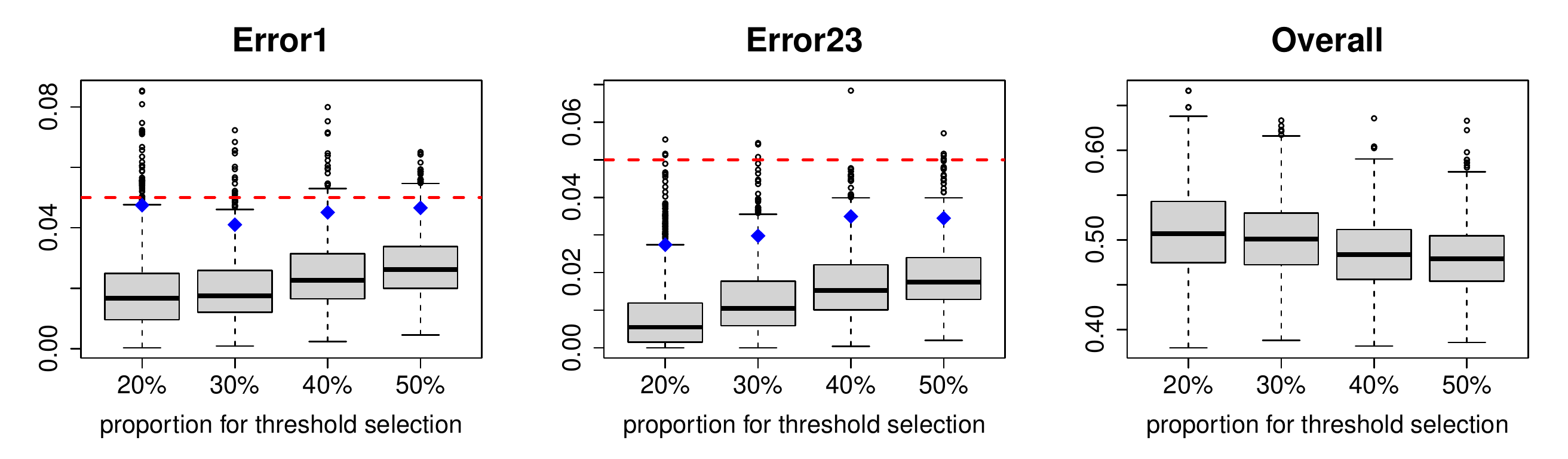}
    \caption{Random Forest}
     \end{subfigure}
 \begin{subfigure}[b]{0.95\textwidth}
    \centering
    \includegraphics[width=\textwidth]{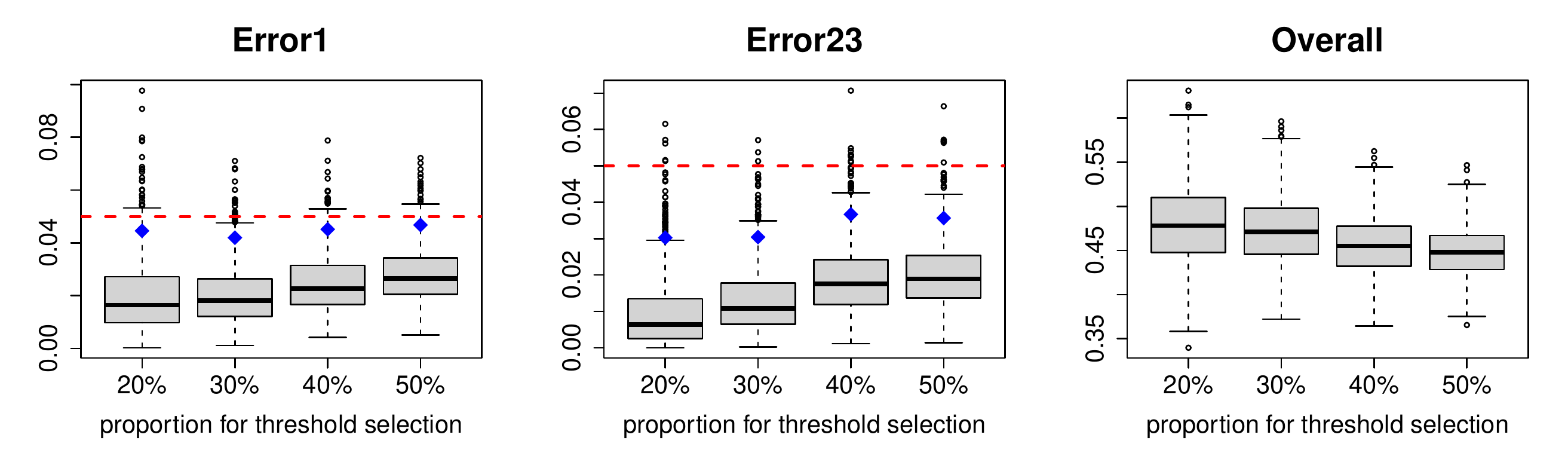}
    \caption{SVM}
     \end{subfigure}
     \caption{\footnotesize The distribution of errors when using 20\% (\textbf{T1.3}), 30\% (\textbf{T1.4}), 40\% (\textbf{T1.5}), and 50\% (\textbf{T1.1}) of the samples for threshold selection, and $\alpha_1, \alpha_2, \delta_1, \delta_2 = 0.05$.  The $(1 - \delta_1)\%$ quantiles of $R^\star_1$ and $(1 - \delta_2)\%$ quantiles of $R^\star_2$ are marked by blue diamonds. Red dashed lines represent the control levels $\alpha_1$ and $\alpha_2$. ``error1", ``error23", and ``overall" correspond to $R_{1\star}(\hatphi)$, $R_{2\star}(\hatphi)$, and $P(\hat{Y} \neq Y)$, respectively.}\label{sub.fig:partition_options}
\end{figure}

\subsection{Variations in calculating the scoring functions}\label{sec:supp_scoring}

As mentioned in Section~2.2, we have normalized each scoring function by the factor $\sum^{\cI}_{j= i + 1}\widehat{\IP}(Y = j\mid X)$, motivated by the NP lemma. 
To illustrate the benefit of normalization empirically, we compare the performance of the normalized scoring functions, $T_i(X) = \widehat{\IP}(Y = i\mid X)/\sum^{\cI}_{j= i + 1}\widehat{\IP}(Y = j\mid X)$, with that of the non-normalized ones, $T_i(X) = \hat{\IP}(Y= i \mid X)$, under the simulation setting \textbf{T1.1} (see Table~\ref{supp.fig:setting summary} for details), the same setting that generated Figure 5 in the main paper. Figure \ref{supp.fig:basic.compare} shows the results for two sets of $\alpha_i, \delta_i$ values. Both scoring functions effectively control the {under-classification} errors, namely ``error1'' and ``error23'', below the desired levels. However, the non-normalized approach controls ``error23" in a conservative manner, which results in higher values of ``error32" (i.e., $P_3(\hatY = 2)$). The other errors not depicted in the plots do not exhibit any noticeable differences.

\begin{figure}[!ht]
\centering
  \begin{subfigure}[b]{0.45\textwidth}
    \centering
    \includegraphics[width=\textwidth]{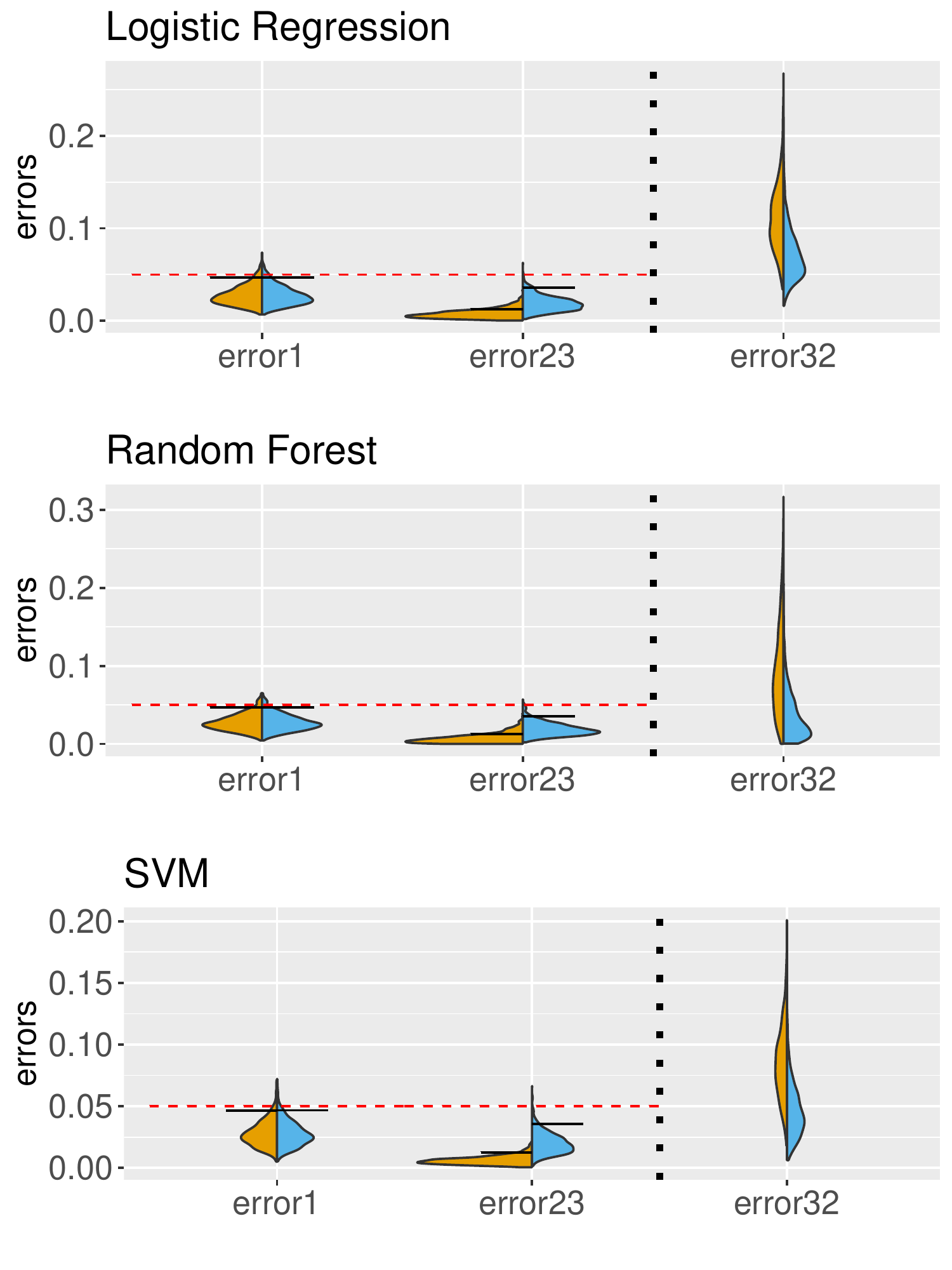}
    \caption{$\alpha_i,\delta_i = 0.05$}
     \end{subfigure}
    \begin{subfigure}[b]{0.45\textwidth}
    \centering
    \includegraphics[width=\textwidth]{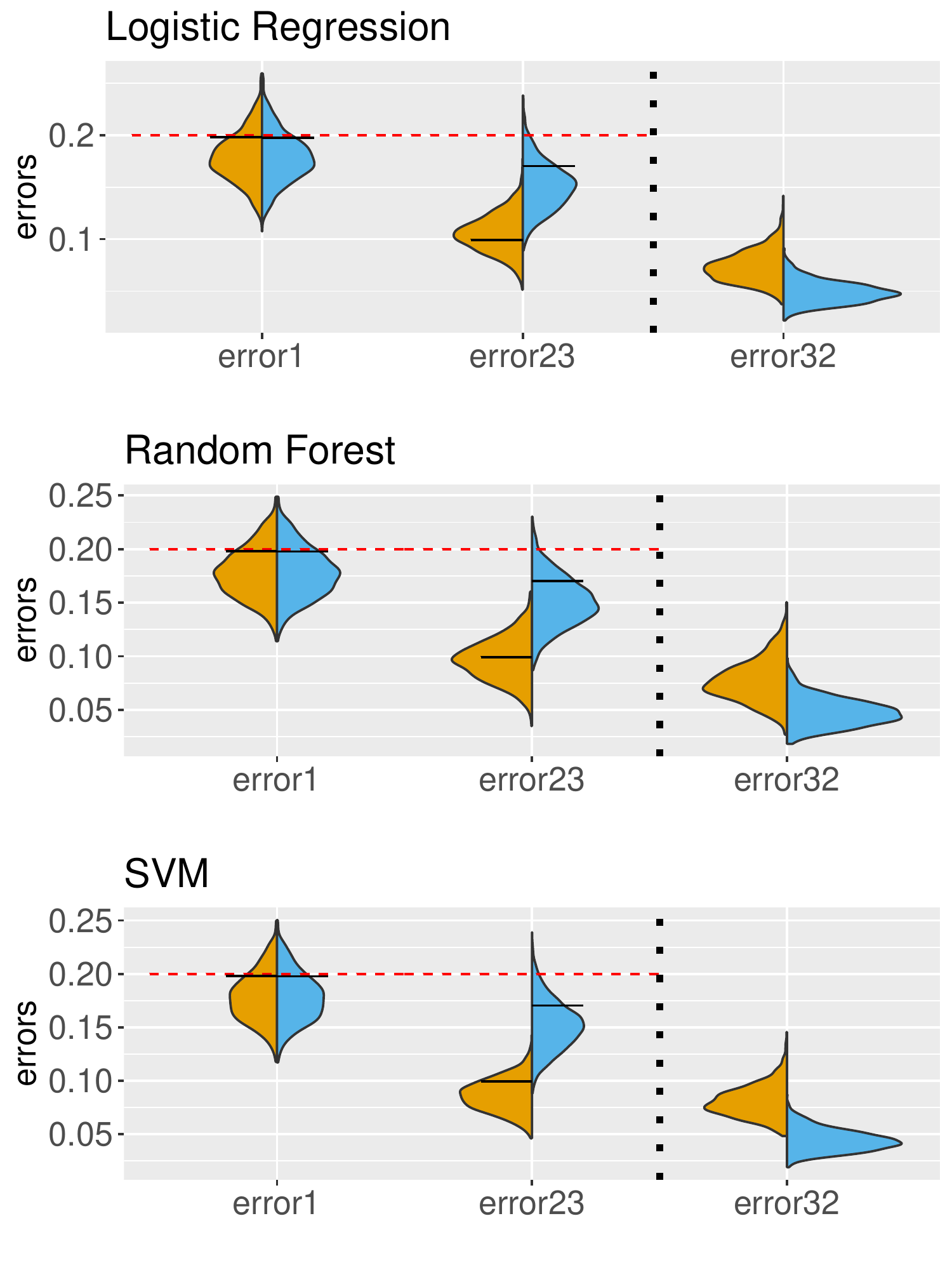}
    \caption{$\alpha_i,\delta_i = 0.2$}
     \end{subfigure}
 \begin{subfigure}[b]{0.45\textwidth}
    \centering
    \includegraphics[width=\textwidth]{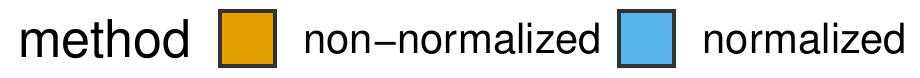}
     \end{subfigure} \caption{We compare the normalized scoring functions proposed in Section~2.2 with the non-normalized version (i.e., let $T_i(X) = \hat{\IP}(Y= i \mid X)$) under the setting \textbf{ T1.1}. ``error1", ``error23", and ``error32"  correspond to the errors $R_{1\star}(\hatphi)$, $R_{2\star}(\hatphi)$, and $P_3(\hat{Y}= 2)$.} \label{supp.fig:basic.compare}
\end{figure}

Next, we note that for general machine learning models such as the random forest and SVM, a probability calibration procedure is often applied to ``correct'' the output scores for more accurate predicted probabilities. 
Under the setting \textbf{T1.1}, we  compare the original scores from each base classifier with the calibrated scores calculated by the function \texttt{CalibratedClassifierCV}  \citep{zadrozny2002transforming} in the Python package \texttt{sklearn}. The latter is denoted as \textbf{T1.2} in Table~\ref{supp.fig:setting summary}, where 10\% of each dataset was used for calibration. Figure~\ref{supp.fig:inf_calibration} shows the effect of calibration for all the errors under two different sets of $\alpha_i, \delta_i$ values. Little differences can be seen for logistic regression and random forest. For SVM, the scores with no calibration perform better in $P_3(\hat{Y}= 1)$, $P_3(\hat{Y}= 2)$ and $P(\hat{Y} \neq Y)$. In all cases, the H-NP approach maintains effective control over the {under-classification} errors, regardless of the type of scores used.

\begin{figure}[!ht]
\centering
  \begin{subfigure}[b]{0.45\textwidth}
    \centering
    \includegraphics[width=\textwidth]{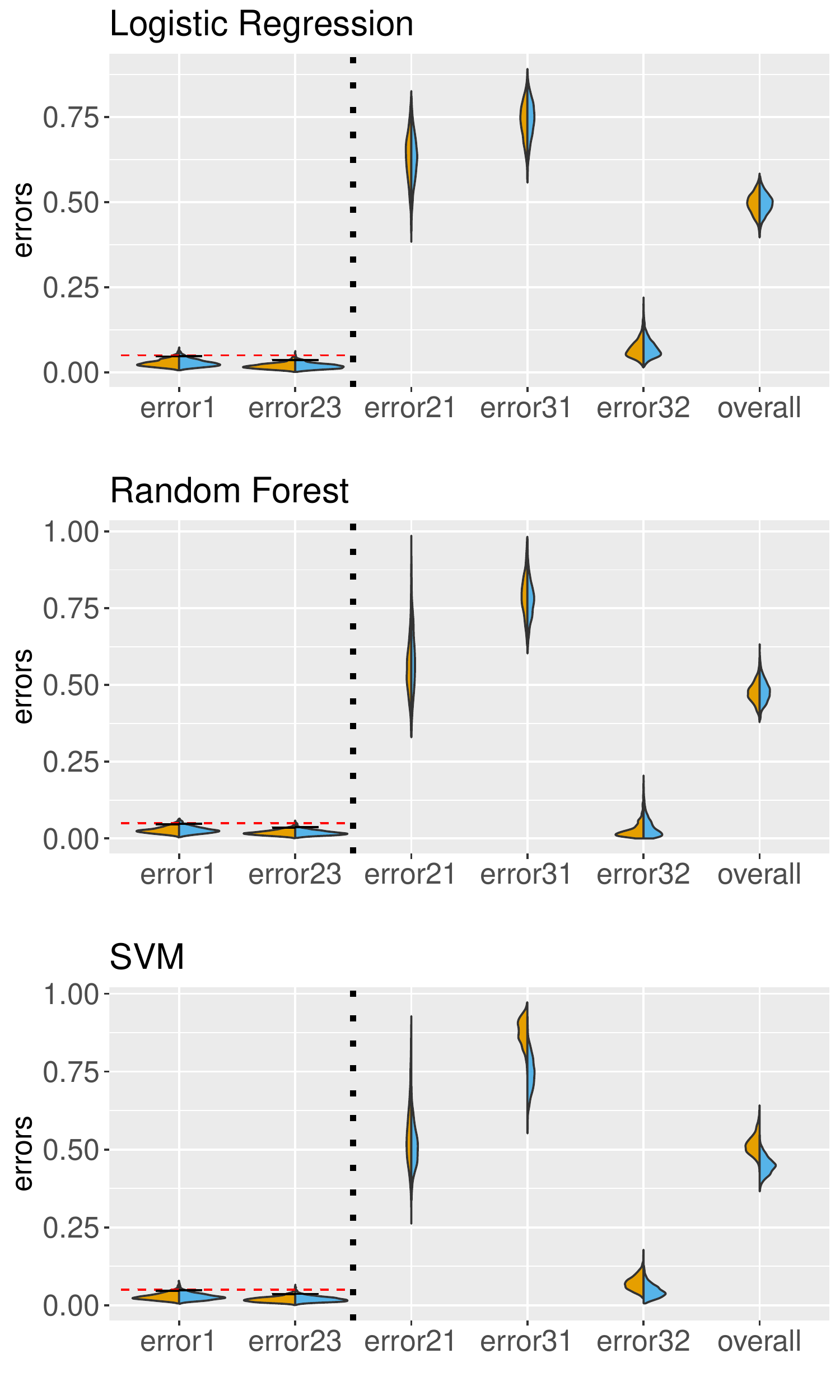}
    \caption{$\alpha_i,\delta_i = 0.05$}
     \end{subfigure}
    \begin{subfigure}[b]{0.45\textwidth}
    \centering
    \includegraphics[width=\textwidth]{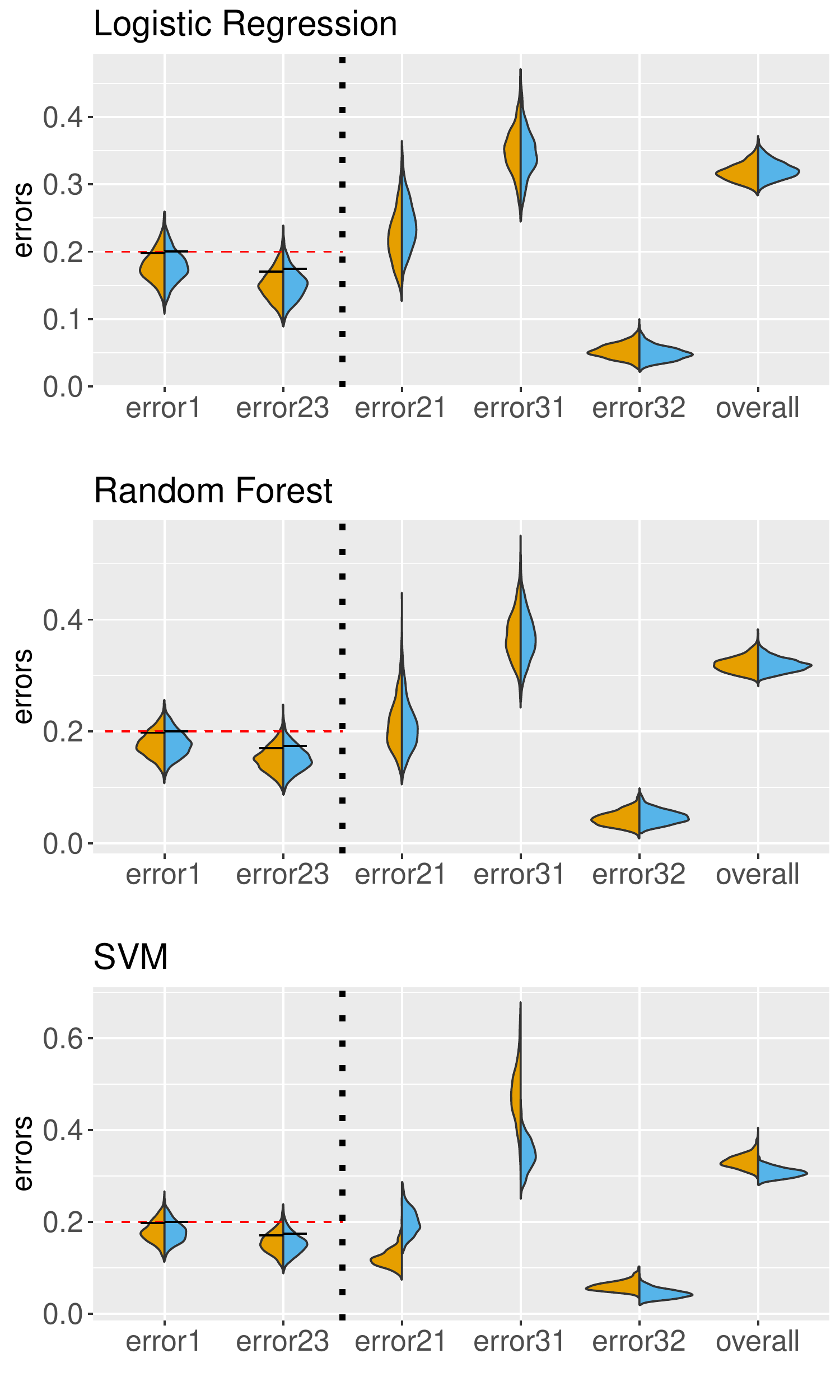}
    \caption{$\alpha_i,\delta_i = 0.2$}
     \end{subfigure}
 \begin{subfigure}[b]{0.45\textwidth}
    \centering
    \includegraphics[width=\textwidth]{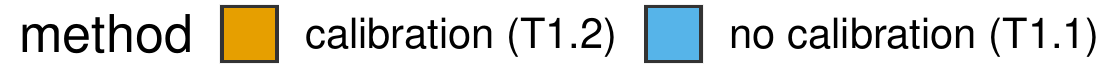}
     \end{subfigure}
     \caption{The distribution of errors when using the original (\textbf{T1.1}) and calibrated (\textbf{T1.2}) scores. ``error1", ``error23", ``error21", ``error31'', ``error32", ``overall" correspond to $R_{1\star}(\hatphi)$, $R_{2\star}(\hatphi)$, $P_2(\hat{Y}= 1)$, $P_3(\hat{Y}= 1)$, $P_3(\hat{Y}= 2)$ and $P(\hat{Y} \neq Y)$, respectively.}\label{supp.fig:inf_calibration}
\end{figure}

\subsection{Imbalanced class sizes and weight-adjusted classification}\label{supp.sec:weighted}

Imbalanced class sizes can lead to more difficulty in controlling prioritized errors. To investigate this effect, we consider the simulation setting \textbf{T1.7} (with sample sizes of $N_1 = 500, N_2 = 500, N_3 = 1{,}000$, see Table~\ref{supp.fig:setting summary}), where the prioritized classes have smaller sample sizes. We compare the results from the classical, H-NP, and weight-adjusted paradigms, where heavier weights are assigned to classes 1 and 2 when computing the empirical loss during optimization.

Table~\ref{supp.fig:T6.imbalance.small_control} shows the averages of the approximate errors and the relevant quantiles based on our chosen tolerance levels, with visualizations provided in Figure~\ref{supp.fig:T6.imbalance.weight.boxplot}. As expected, the imbalanced setting significantly increased the {under-classification} errors under the classical paradigm compared with the balanced setting. We implemented H-NP with two sets of $\alpha_i, \delta_i$ values, and in both settings the {under-classification} errors are effectively controlled under the specified levels.

\begin{table}[h!!!]
\centering
\footnotesize
\begin{tabular}{c|c|cc|c}
\hline\hline
             \multicolumn{5}{c}{Logistic Regression} \\\hline
            Paradigm            &    & Error1 & Error23 & Overall \\ \hline
     \multirow{2}{*}{classical} & mean & 0.491 & 0.176 & 0.264\\
                                &90\% quantile &0.509 & 0.188&  \\
                                &80\% quantile &0.504 & 0.184 & \\
                                &70\% quantile & 0.499 & 0.181 & \\\hline
                                H-NP &   mean   &  0.075 & 0.058  & 0.464 \\
$(\alpha_i,\delta_i = 0.10)$&90\% quantile    &0.098 & 0.077  & \\\hline
      weight & mean &0.135 & 0.007 & 0.466\\
           80:50:1 &90\% quantile &0.145 & 0.010&  \\\hline
          H-NP &   mean   &0.179 & 0.149  & 0.35\\
$(\alpha_i,\delta_i= 0.20)$ &80\% quantile    &0.199 & 0.167  & \\\hline
        weight   & mean &0.203 & 0.032 & 0.36\\
            15:12:1&80\% quantile & 0.211 & 0.036  &\\\hline\hline

    \multicolumn{5}{c}{Random Forest} \\\hline
    Paradigm &    & Error1 & Error23  & Overall  \\ \hline
        \multirow{2}{*}{classical} & mean & 0.489 & 0.184  & 0.271\\
                                &90\% quantile & 0.529 & 0.213  &  \\
                                &80\% quantile &0.514 & 0.202  & \\
                                &70\% quantile & 0.504 & 0.195  & \\\hline
                                H-NP &   mean   & 0.075 & 0.056  & 0.468 \\
$(\alpha_i,\delta_i = 0.10)$&90\% quantile    &0.098 & 0.074  & \\\hline
weight & mean &0.139 & 0.002 & 0.538\\
           15:10:1 &90\% quantile &0.158 & 0.006   &  \\\hline
                               H-NP &   mean   &0.178 & 0.148 & 0.354\\
$(\alpha_i,\delta_i = 0.20)$ &80\% quantile    &0.198 & 0.167  & \\\hline
weight & mean &0.203 & 0.045  & 0.365\\
           4:4:1 &80\% quantile &0.223 & 0.053  &  \\\hline\hline
 
\end{tabular}
\caption{ The averages and quantiles of approximate errors under the setting \textbf{T1.7}.  ``error1", ``error23", and ``overall" correspond to $R_{1\star}(\hatphi)$, $R_{2\star}(\hatphi)$, and $P(\hat{Y} \neq Y)$, respectively.}\label{supp.fig:T6.imbalance.small_control}
\end{table}

When choosing the weights for the weight-adjusted paradigm, we note that there is no direct correspondence between the weights and our parameters $\alpha_i, \delta_i$. For a fair comparison, we performed a grid search on the weights to find weight combinations that led to an overall error roughly similar to that of the H-NP classifier. The grid search is computationally intensive, and the weights are more difficult to interpret than $\alpha_i$ and $\delta_i$. Table~\ref{supp.fig:T6.imbalance.small_control} and Figure~\ref{supp.fig:T6.imbalance.weight.boxplot} show the matched H-NP and weight-adjusted classifications as pairs adjacent to each other for easy comparison. We find that significantly different weights are needed for logistic regression and random forest to achieve similar under-diagnostic errors, suggesting that the choice of weights is not robust across different base classifiers, while the H-NP method is much more consistent. At similar overall error levels, the weight-adjusted classification gives higher $R_{1\star}(\hatphi)$, while the control on $R_{2\star}(\hatphi)$ is more conservative.

\begin{figure}[!ht]
    \centering
    \includegraphics[width=\textwidth]{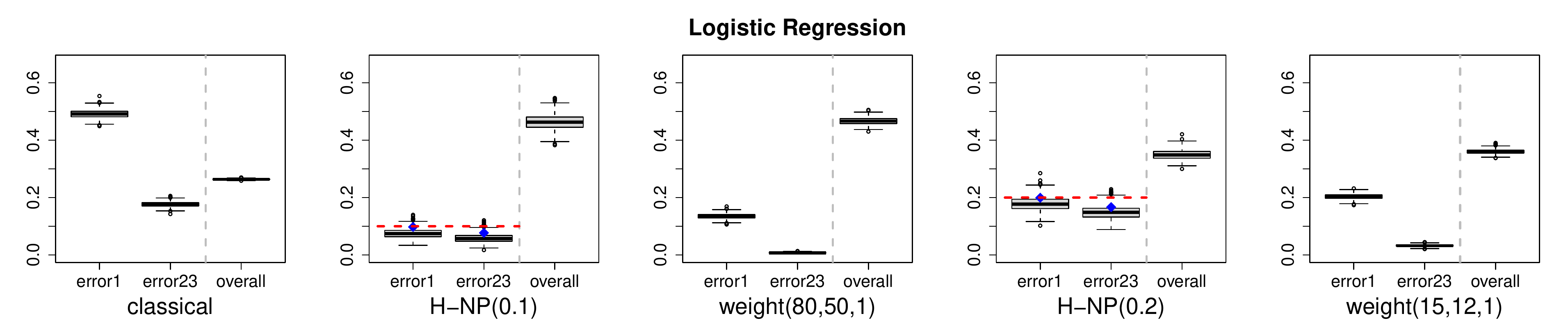}
     \includegraphics[width=\textwidth]{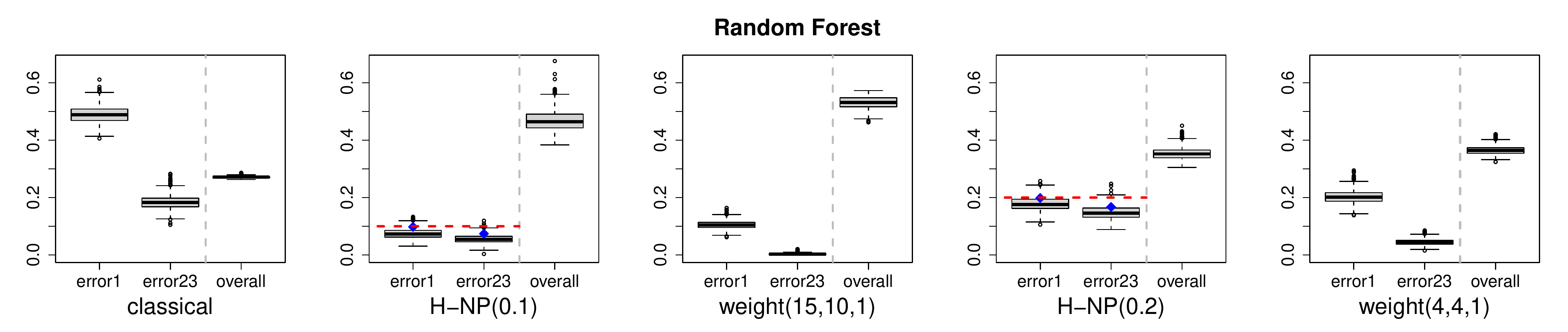}
     \caption{The distributions of approximate errors in the classical, H-NP and weight-adjusted classification paradigms under the setting \textbf{T1.7}. For H-NP, the values in parentheses indicate the values of $\alpha_i$ and $\delta_i$. For weight-adjusted classification,  the values in the parentheses indicate the weights assigned to classes 1,2, and 3, respectively. ``error1", ``error23", and ``overall" correspond to $R_{1\star}(\hatphi)$, $R_{2\star}(\hatphi)$, and $P(\hat{Y} \neq Y)$, respectively.}\label{supp.fig:T6.imbalance.weight.boxplot}
\end{figure}

\subsection{Comparing with cost-sensitive learning}\label{supp.sec:imb_and_cs}

Cost-sensitive learning is an alternative approach to asymmetric error control. However, the costs assigned to different types of errors can be less interpretable and harder to select from a practitioner's perspective (especially in the multi-class setting) than our parameters $\alpha_i$ and $\delta_i$, which represent an upper bound on error rate and a tolerance level in probability, respectively. To perform numerical comparisons, we first note that there is no direct correspondence between our $\alpha_i, \delta_i$ values and the error costs. The only work we are aware of connecting the multi-class NP (NPMC) problem with cost-sensitive learning is \citet{feng2021targeted}.
Different from our problem setting, the NPMC method controls $R_{1\star}(\hat{\phi})$ and $P_2(\hat{Y} \neq 2)$ at levels $\Tilde{\alpha}_1$ and $\Tilde{\alpha}_2$ by computing the appropriate cost assignment. Since $P_2(\hat{Y} \neq 2)$ is a sum of $R_{2\star}(\hat{\phi})$ and $P_2(\hat{Y} =1)$, for a fairer comparison, we set the control levels in NPMC and H-NP to be $\Tilde{\alpha}_2=2\alpha_2$. We compared NPMC and H-NP under the simulation setting \textbf{T1.1} (details in Table~\ref{supp.fig:setting summary}) for feasible choices $\Tilde{\alpha}_i$. 
Setting $\Tilde{\alpha}_1,\Tilde{\alpha}_2=0.2$ for NPMC and $\alpha_1=0.2$, $\alpha_2=0.1$, and $\delta_1, \delta_2=0.1$ for H-NP, Figure~\ref{supp.fig:npmc} shows the distributions of the approximate {under-classification} errors from logistic regression and SVM, with the red dashed lines representing the 90\% quantiles. The distributions of $R_{1\star}(\hatphi)$ highlight the difference between the approximate control by NPMC and the high probability control by H-NP. The NPMC's control on $R_{2\star}(\hatphi)$ appears to be slightly more conservative, resulting in slightly higher $P_3(\hat{Y}= 2)$ values.

\begin{figure}[!ht]
\centering
  \begin{subfigure}[b]{0.45\textwidth}
    \centering
    \includegraphics[width=\textwidth]{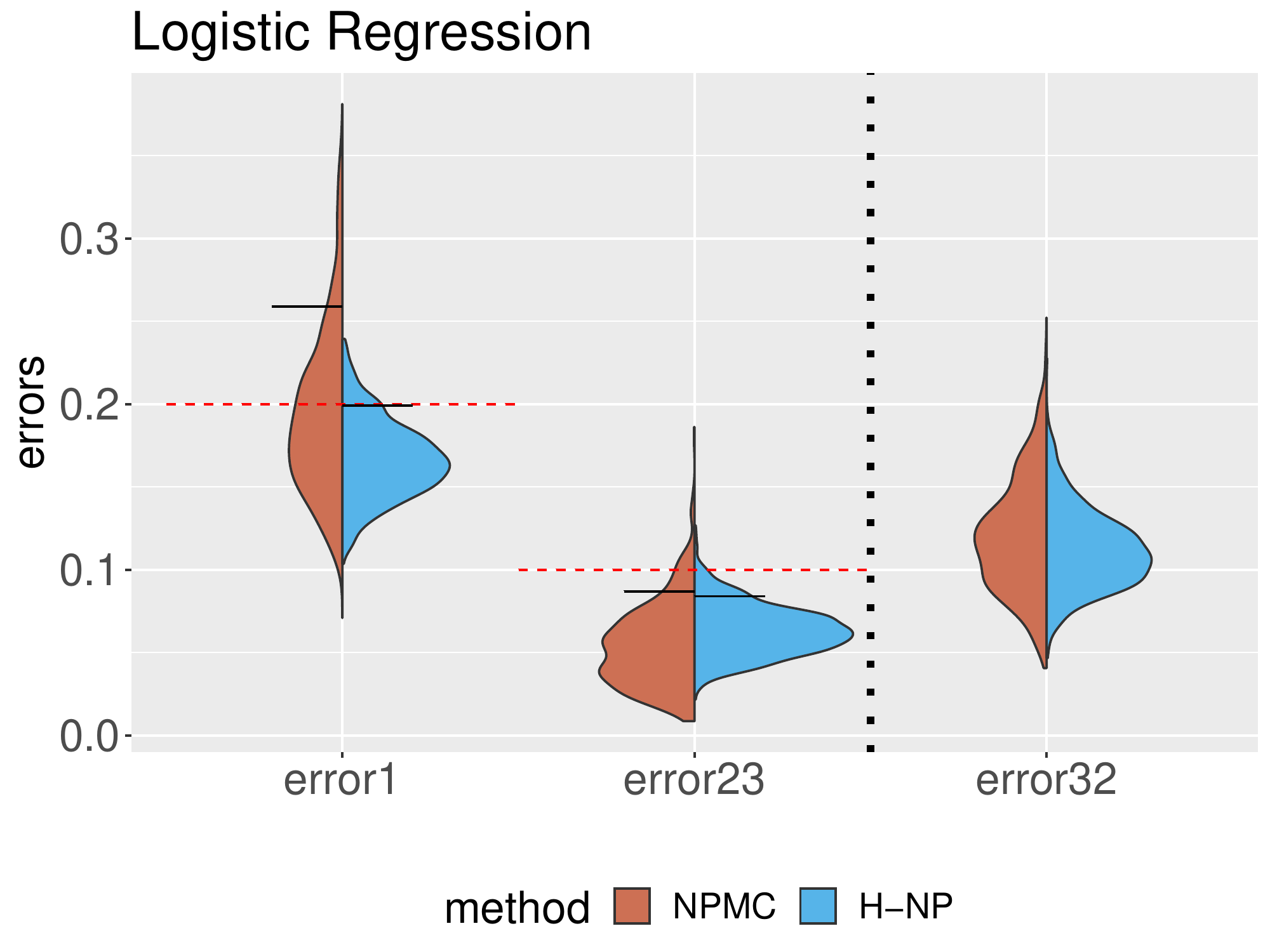}
    \caption{Logistic regression}
     \end{subfigure}
    \begin{subfigure}[b]{0.45\textwidth}
    \centering
    \includegraphics[width=\textwidth]{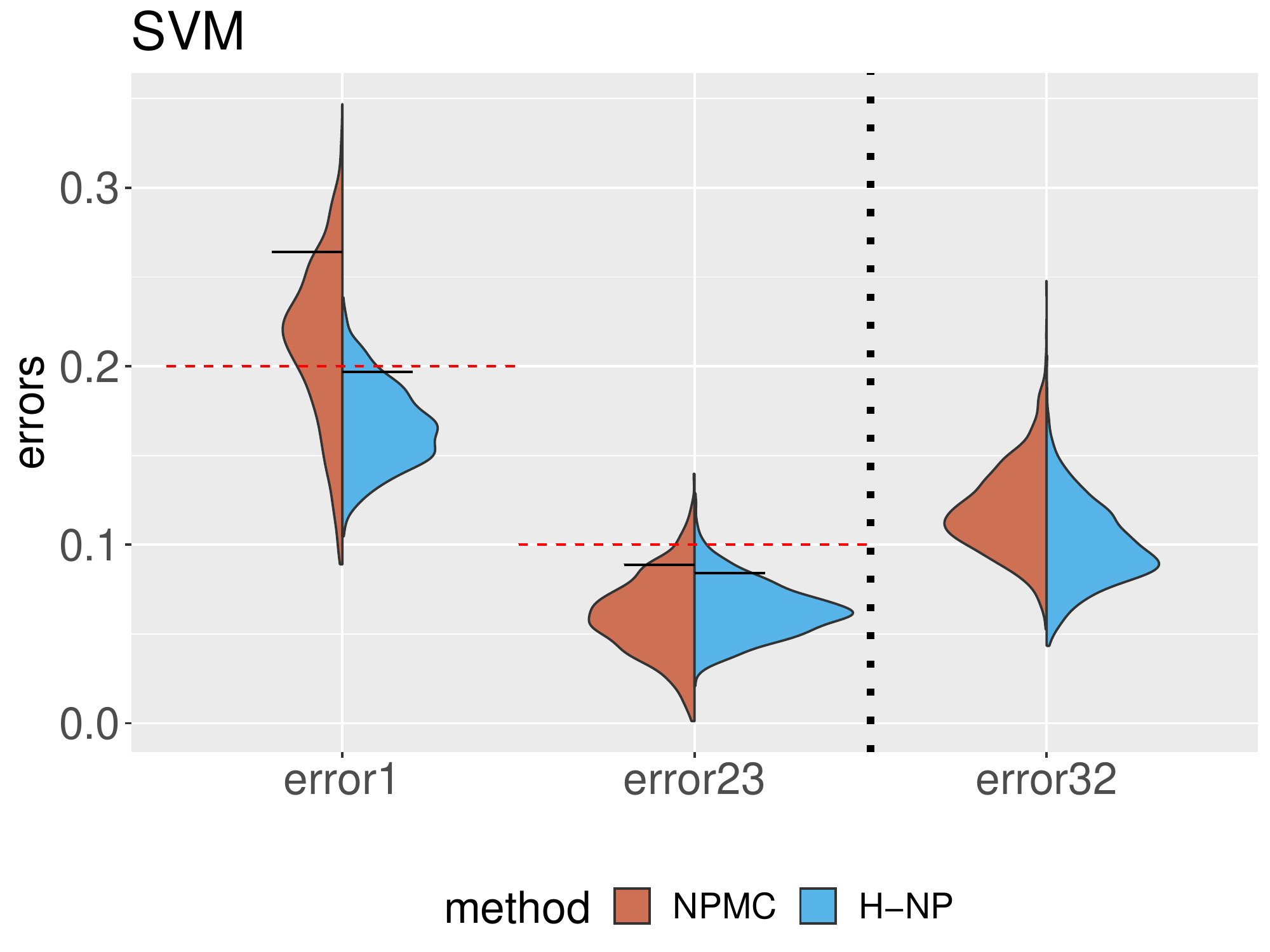}
    \caption{SVM}
     \end{subfigure}
     \caption{ The distributions of approximate errors on the test set under the setting \textbf{T1.1} for NPMC and H-NP approaches. ``error1" ``error23", and ``error32"  correspond to the errors $R_{1\star}(\hatphi)$, $R_{2\star}(\hatphi)$, and $P_3(\hat{Y}= 2)$, respectively.}\label{supp.fig:npmc}
\end{figure}

\subsection{Comparison with ordinal classification}\label{supp.sec:ordinal}

We conduct comparisons with ordinal classification methods, including: 
\begin{itemize}
    \item cumulative link models (CLMs), which are a type of generalized linear models that use cumulative probabilities to characterize ordinal outcomes. We include both the logit and probit link functions \citep{agresti2002categorical}, denoted as CLM (logit/probit) in the results below. We also include the method from \citep{hornung2020ordinal} denoted as ordinalForest, which uses a similar approach using the probit link function but based on random forest;  

    \item
    methods based on decomposing the multi-class classification problem into multiple binary classification problems. We compare with the method FH01 from \citet{frank2001simple}, which combines the results from $\cI - 1$ binary classifiers for classes $\{1, \dots, i\}$ versus classes $\{i + 1, \dots, \cI\}$, and the method oSVM from \citet{cardoso2007learning}, which considers the binary classification problem of classes $\{i - k, \dots, i\}$ versus classes $\{i + 1, \dots,i + 1 +k \}$ for some fixed constant $k$. We use logistic regression and support vector machine (SVM) as the classifiers in FH01 and oSVM respectively;

    \item
    the recent method FWOC by 
\citet{ma2021feature}, which prioritizes important features that are highly correlated with the ordinal structure by incorporating feature weighting in linear discriminant analysis to construct the classifier.
\end{itemize}

The comparison is performed under the simulation setting \textbf{T1.1} in Table~\ref{supp.fig:setting summary} and the results are shown in Figure~\ref{supp.fig:T1.orginal.ordinal}. The averages of approximate {under-classification} errors, $R_{1\star}(\hatphi)$ and $R_{2\star}(\hatphi)$, for all ordinal methods vary widely between 0.1 and 0.4, and these two errors are not guaranteed to be lower than the other errors. Most importantly, these methods do not allow users to specify a control level on the errors of interest. For an easy comparison, we set  $\alpha_i,\delta_i = 0.2$ ($i=1,2$) in our H-NP classifier, and as shown in Figure~\ref{supp.fig:T1.orginal.ordinal}(a), the {under-classification} errors are effectively controlled under the desired levels.

\begin{figure}[!ht]
\centering
  \begin{subfigure}[b]{0.32\textwidth}
    \centering
    \includegraphics[width=\textwidth]{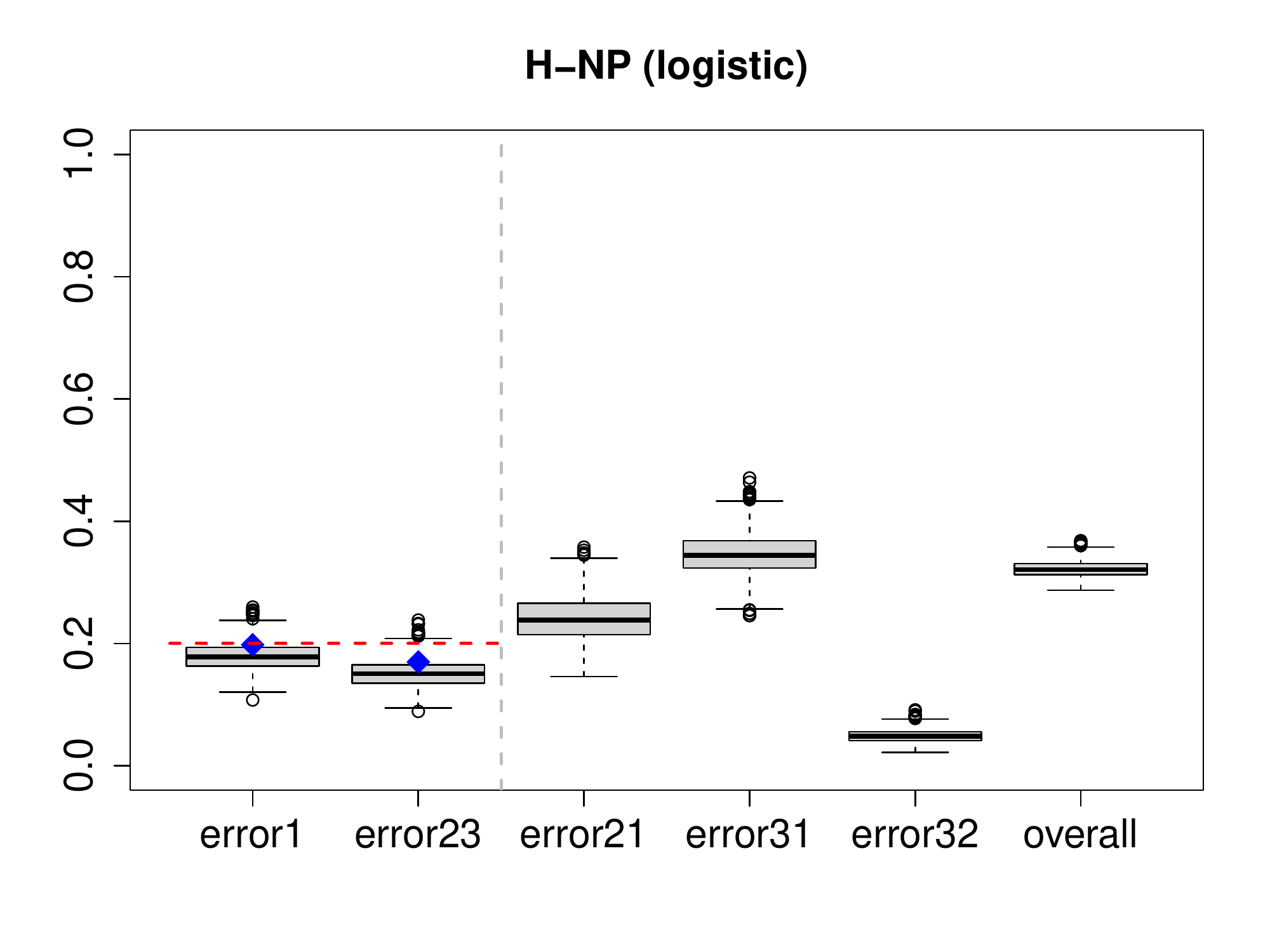}
    \caption{H-NP($\alpha_i, \delta_i = 0.20$)}
     \end{subfigure}
       \begin{subfigure}[b]{0.32\textwidth}
    \centering
    \includegraphics[width=\textwidth]{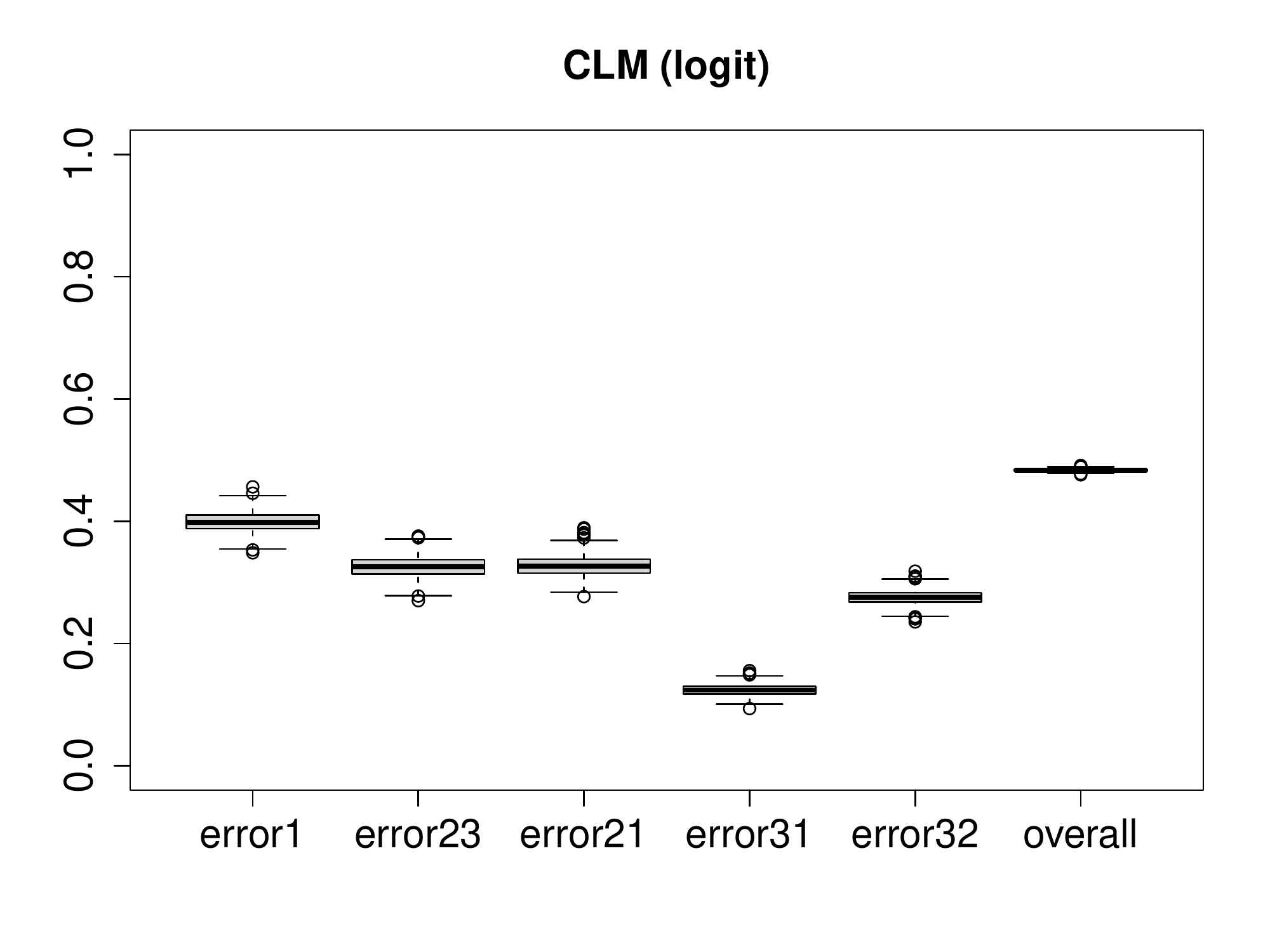}
    \caption{CLM (logit)}
     \end{subfigure}
       \begin{subfigure}[b]{0.32\textwidth}
    \centering
    \includegraphics[width=\textwidth]{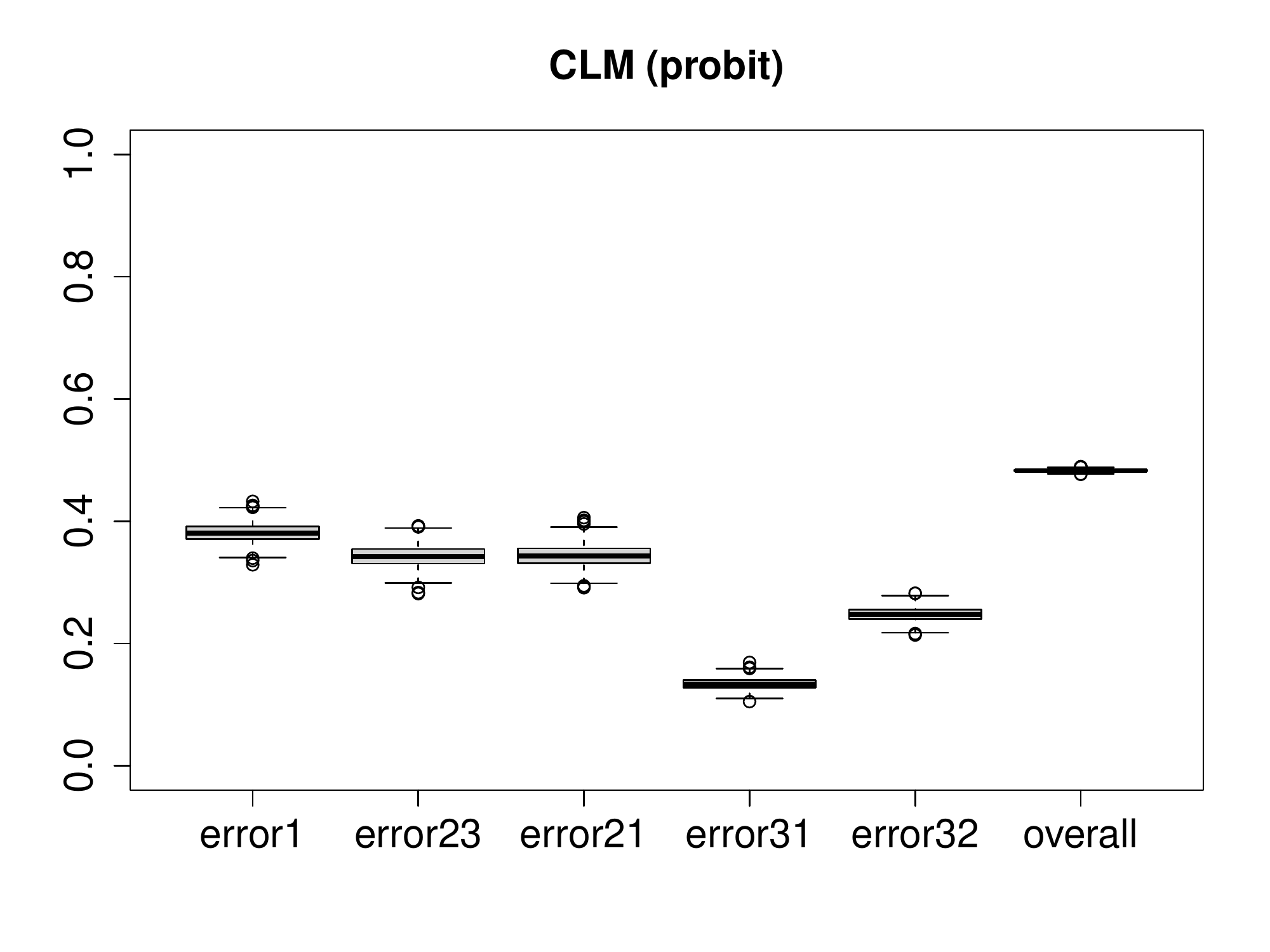}
    \caption{CLM (probit)}
     \end{subfigure}
         \begin{subfigure}[b]{0.32\textwidth}
    \centering
    \includegraphics[width=\textwidth]{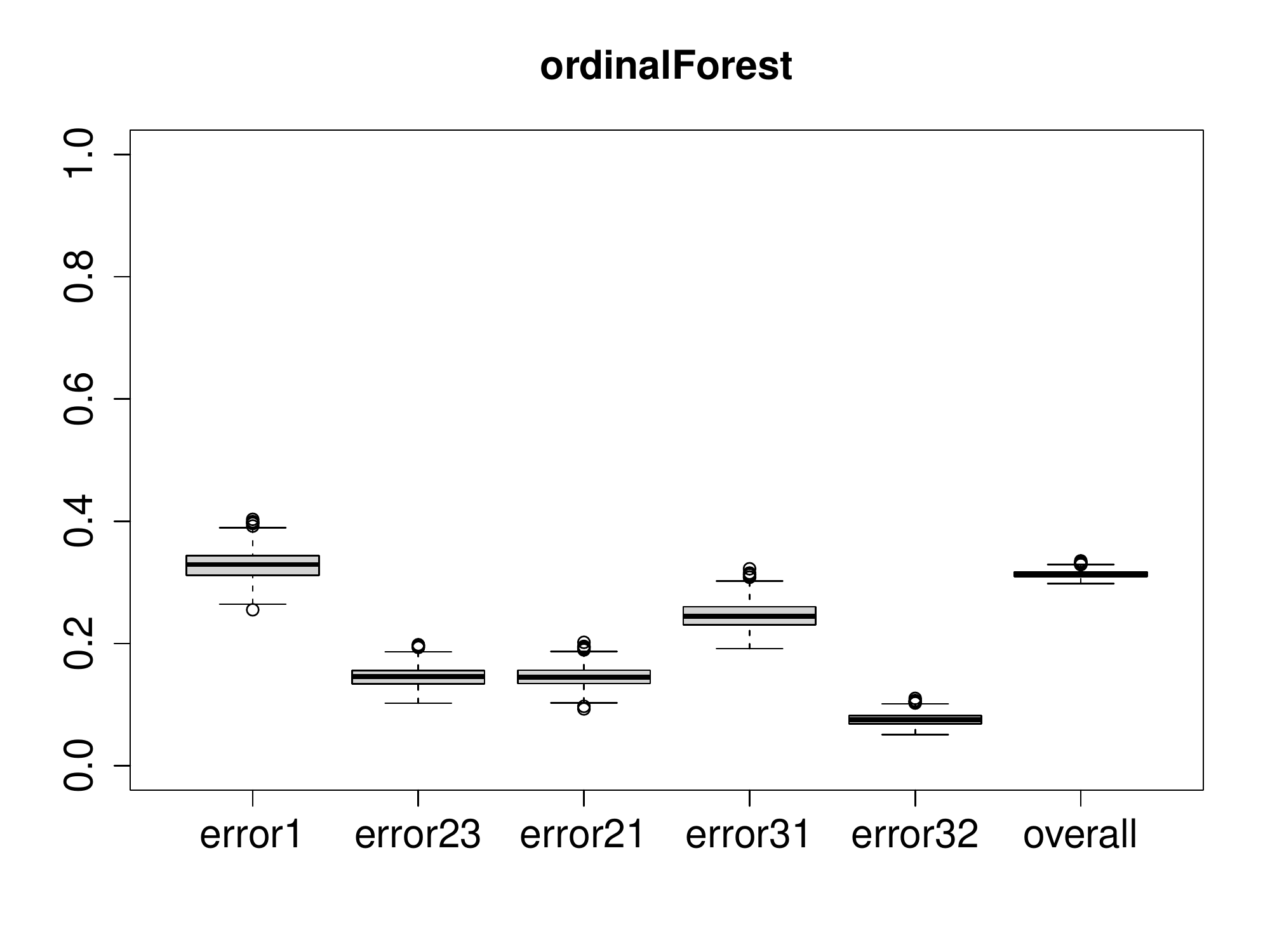}
    \caption{ordinalforest}
     \end{subfigure}
       \begin{subfigure}[b]{0.32\textwidth}
    \centering
    \includegraphics[width=\textwidth]{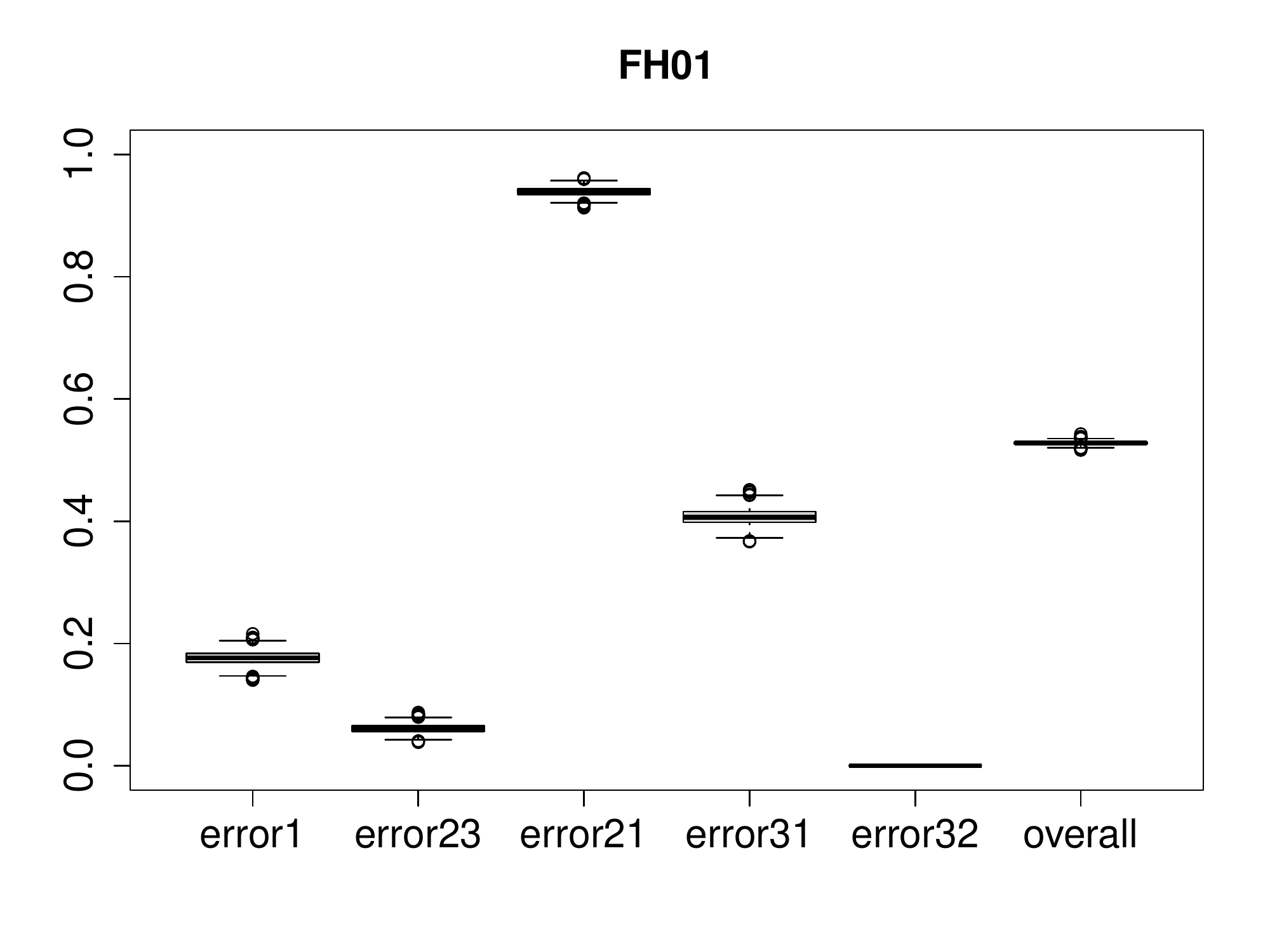}
    \caption{FH01}
     \end{subfigure}
     \begin{subfigure}[b]{0.32\textwidth}
    \centering
    \includegraphics[width=\textwidth]{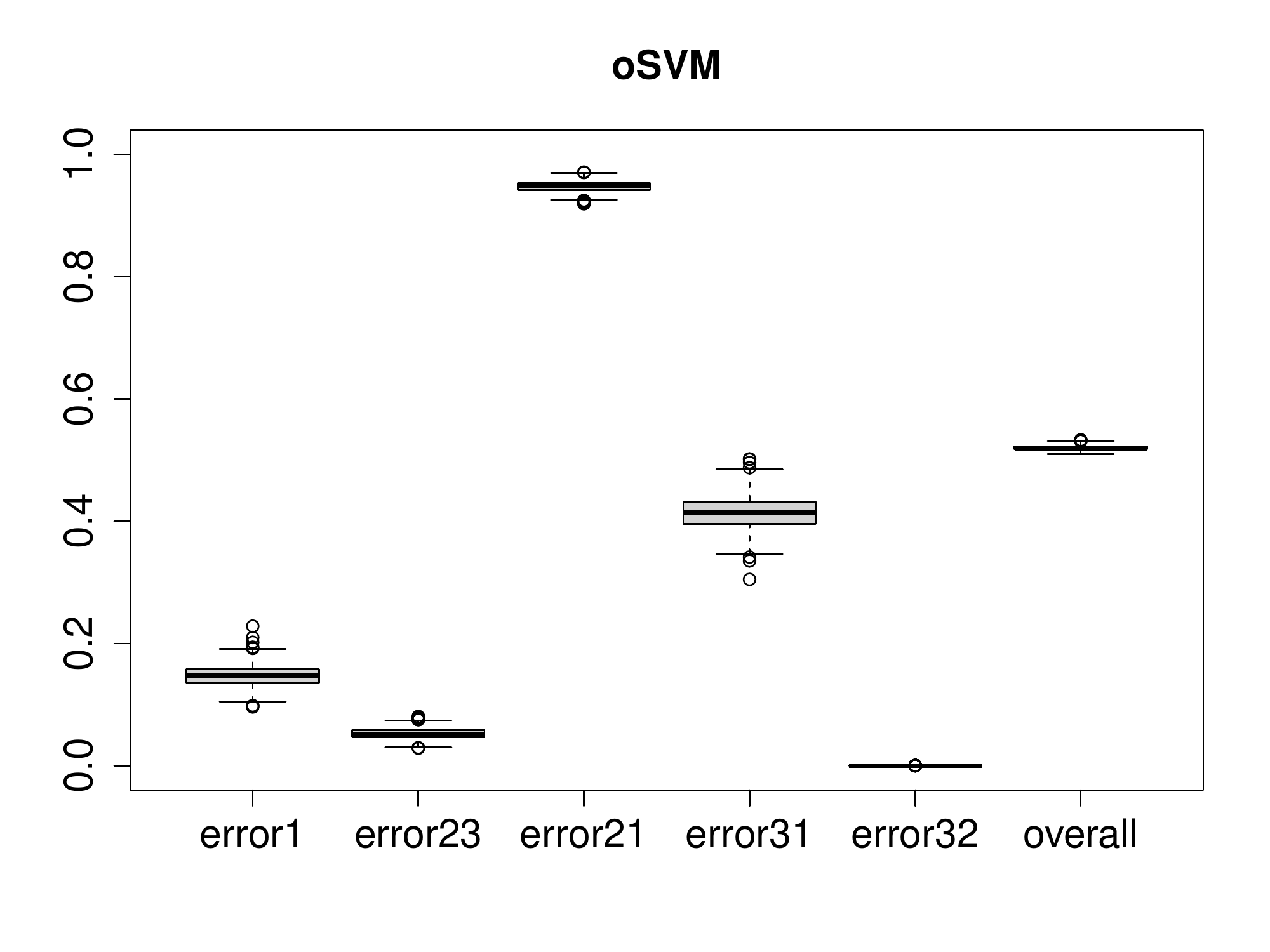}
    \caption{oSVM}
     \end{subfigure}
       \begin{subfigure}[b]{0.32\textwidth}
    \centering
    \includegraphics[width=\textwidth]{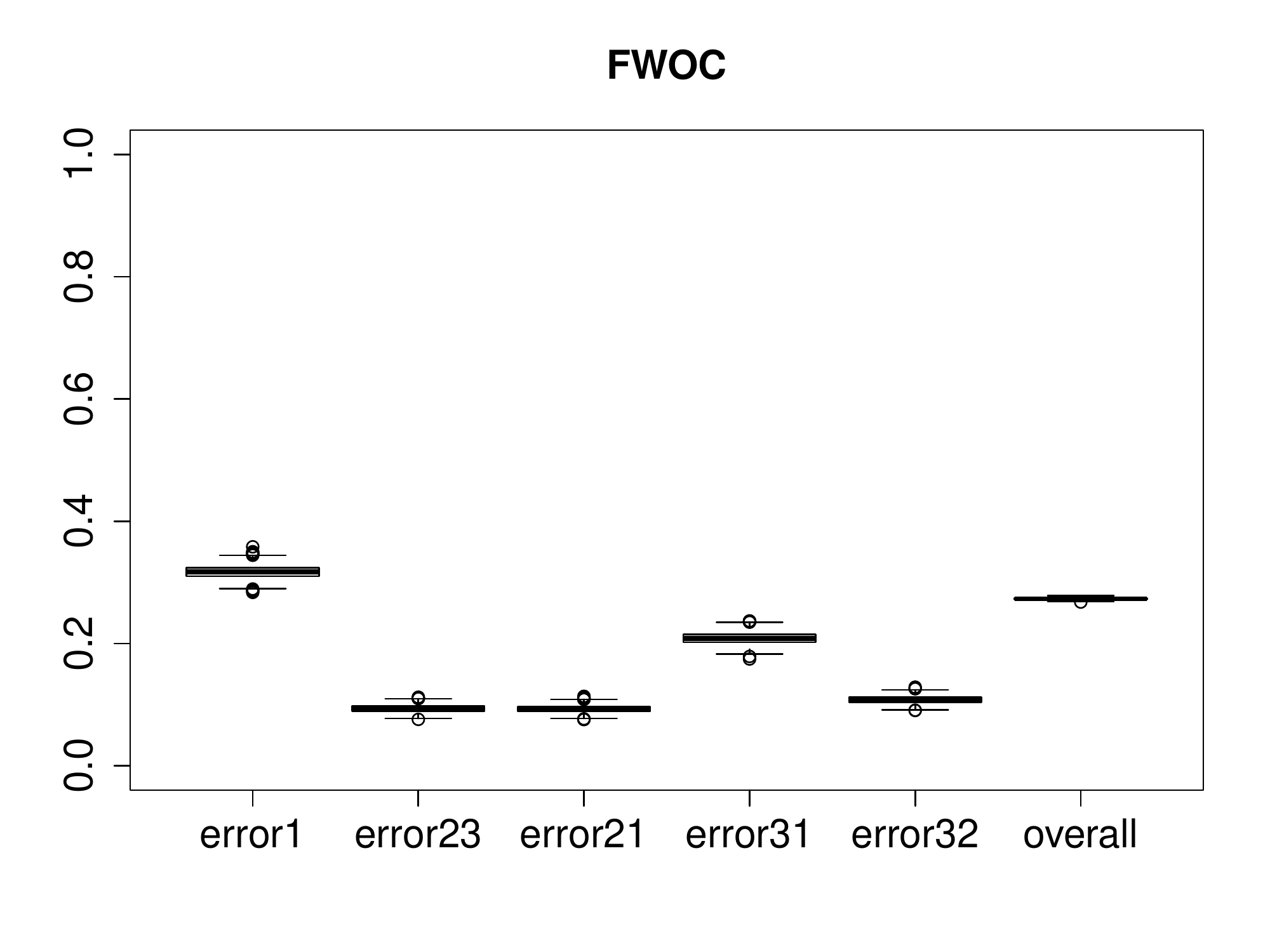}
    \caption{ FWOC}
     \end{subfigure}
     \caption{ The distributions of approximate errors under the setting \textbf{T1.1} for H-NP and ordinal classification methods. (a) H-NP with 
 $\alpha_i, \delta_i = 0.2$ ($i = 1,2$). Ordinal classification methods: (b) and (c) CLM \citet{agresti2002categorical} with logit and probit link;
   (d) method based on random forest \citep{hornung2020ordinal}; (e) method based on logistic regression \citep{frank2001simple}; (f) method based on SVM \citep{cardoso2007learning}; (e) method based on LDA \citep{ma2021feature}. ``error1", ``error23", ``error21", ``error31'', ``error32", ``overall" correspond to $R_{1\star}(\hatphi)$, $R_{2\star}(\hatphi)$, $P_2(\hat{Y}= 1)$, $P_3(\hat{Y}= 1)$, $P_3(\hat{Y}= 2)$ and $P(\hat{Y} \neq Y)$, respectively.}\label{supp.fig:T1.orginal.ordinal}
\end{figure}

\subsection{Non-monotonicity in simulation studies}

Finally, we provide a simulation example to exhibit that the influence of $t_1$ on the weighted sum of errors $R^c(\hatphi)$ is not monotonic. The cause of the  non-monotonicity is explained at the end of Section~2.3. Here, the setting \textbf{T3.1} sets the proportion of each class by setting $N_1 = 1{,}000, N_2 = 200, N_3 = 800$. As a result, the weight for $P_3(\hatY \in \{1, 2\})$  in $R^c(\hatphi)$  increases, and the weight for $P_2(\hatY = 1)$  decreases. We set $\mu_1 = (0, 0)^\top$, $\mu_2 = (-0.5, 0.5)^\top$, and $\mu_3 = (2, 2)^\top$. Under this setting, class-1 and 2 are closer in Euclidean distance compared to class 3. We increase $\alpha_1$ and $\alpha_2$ to $0.1$ to get a wider search region for $t_1$ so that the pattern of $R^c$ is easier to observe. To approximate the true errors $R_{1\star}$, $R_{2\star}$  and  $R^c$ on a test set, we generate $30{,}000$, $6{,}000$ and $2{,}4000$ observations for class $1$, $2$, $3$, respectively. The ratio of the three classes in the test set is the same as $N_1:N_2:N_3$.
Other settings are the same as in the setting \textbf{T1.1}. The results of this new simulation setting are presented in Figure~\ref{supp.fig:sim_error}. We can observe that $R^c$ is not monotonically decreasing, but our procedure still maintains effective controls on the {under-classification} errors.

\begin{figure}[!ht]
\centering
     \begin{subfigure}[b]{0.45\textwidth}
    \centering
    \includegraphics[width=\textwidth]{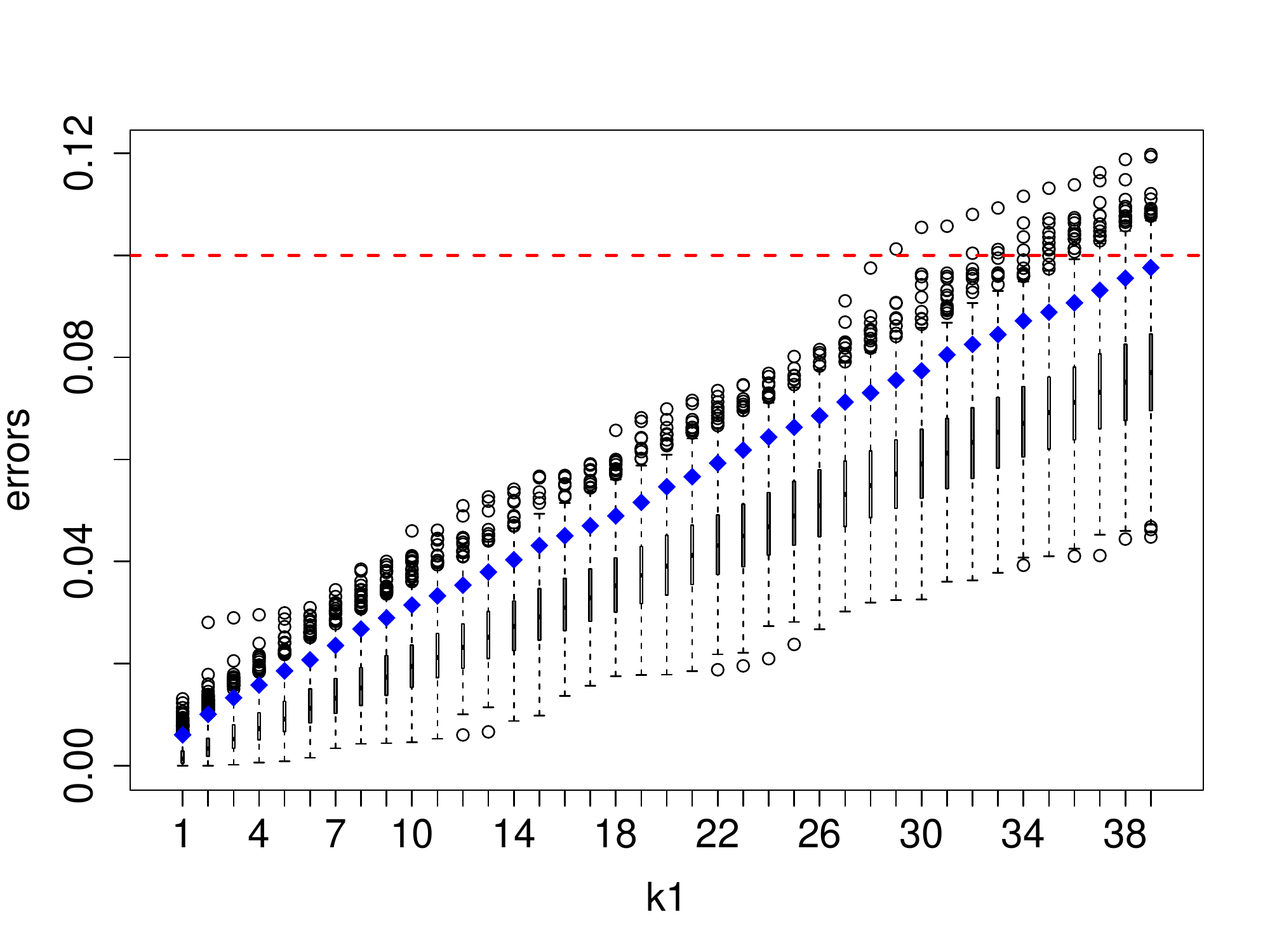}
    \caption{$R_{1\star}$}\label{supp.fig:sim_error_1}
     \end{subfigure}
     \hspace{-0.3cm}
     \begin{subfigure}[b]{0.45\textwidth}
         \centering
    \includegraphics[width=\textwidth]{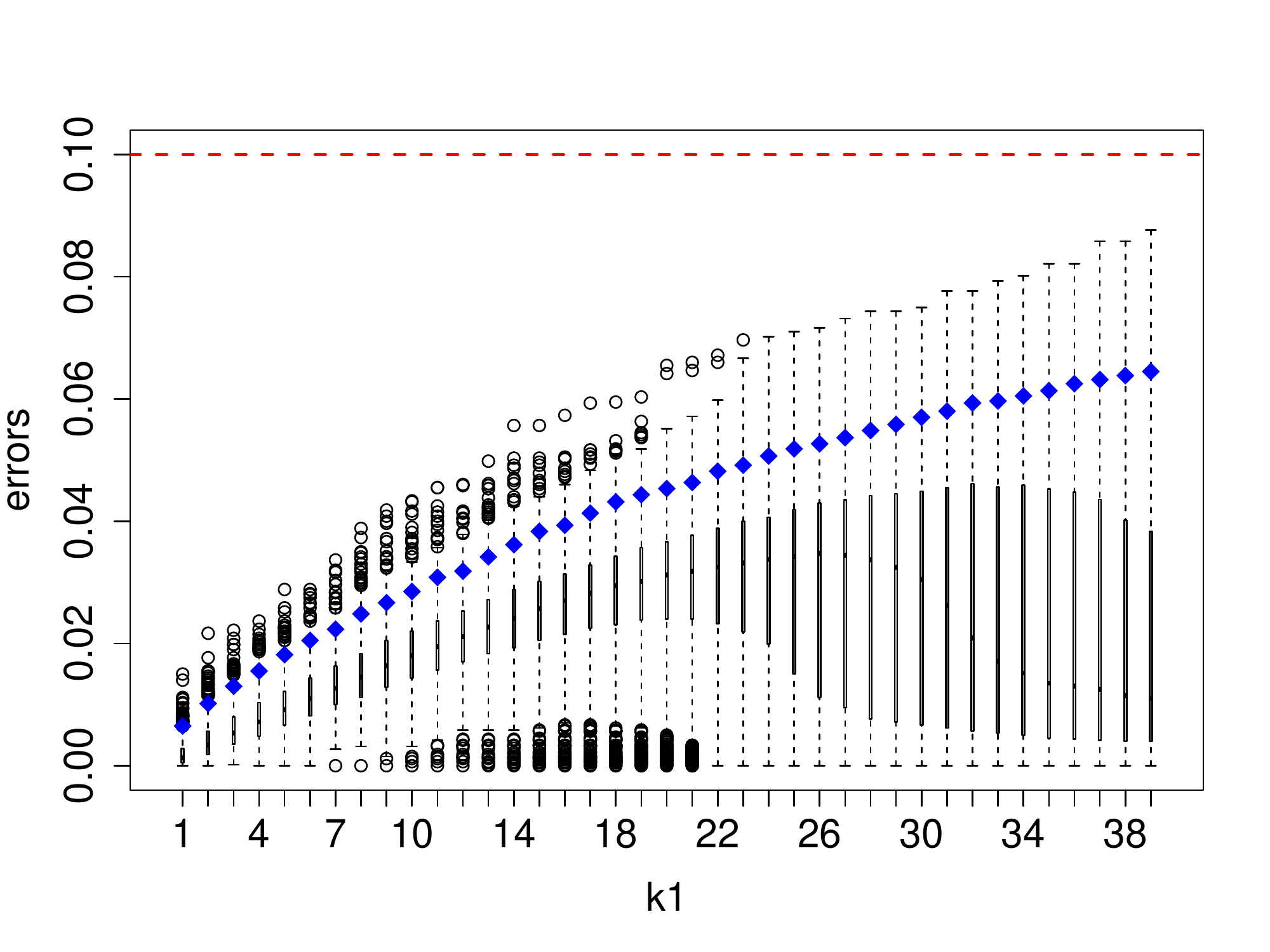} 
   \caption{$R_{2\star}$}\label{supp.fig:sim_error_2}
     \end{subfigure}
     \hspace{-0.3cm}
      \begin{subfigure}[b]{0.5\textwidth}
         \centering
    \includegraphics[width=\textwidth]{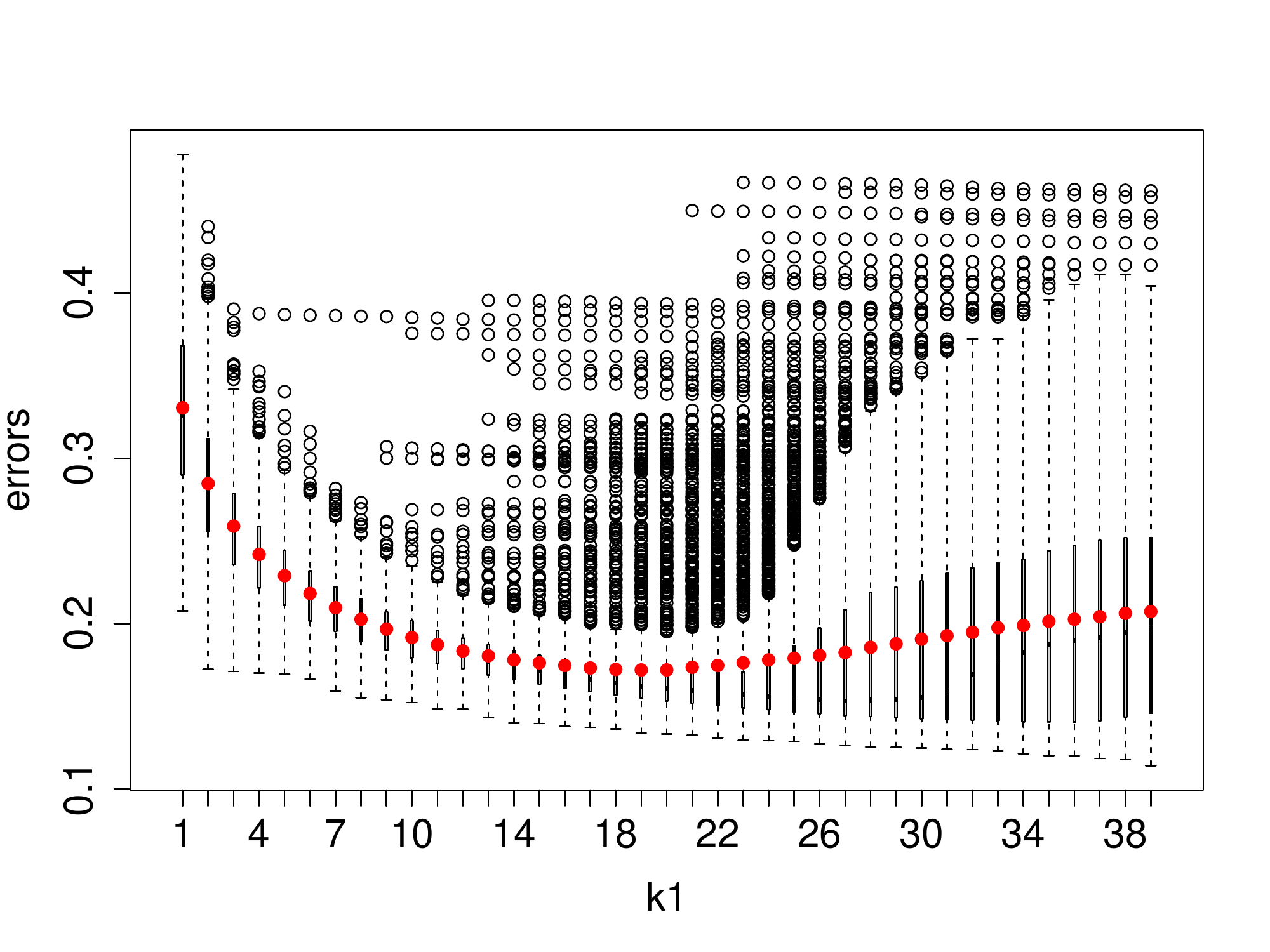}
   \caption{$R^c$}\label{supp.fig:sim_error_c}
     \end{subfigure}
     \caption{The distribution of approximate errors when $t_1$ is the  $k$-th largest element in $\cT_1\cap (-\infty, \ot_1)$. The $95\%$ quantiles  ($\delta_1 = \delta_2 = 0.05$) of $R_{1\star}$ and $R_{2\star}$ are marked by blue diamonds.  The target control levels for  $R_{1\star}(\hatphi)$  and  $R_{2\star}(\hatphi)$ ($\alpha_1 = \alpha_2 = 0.1$) are plotted as red dashed lines. Also, the averages of the $R^c$ are marked by red points in (c).}\label{supp.fig:sim_error}
\end{figure}

\clearpage

\section{Additional results for COVID-19 severity classification}\label{supp.sec:result}

\subsection{Additional table of classification results}

\begin{table}[h!!!]
	\centering
	\footnotesize
	\begin{tabular}{c|c|c|c|c|c|c|c}
		\hline\hline
		\multicolumn{8}{c}{Logistic Regression} \\\hline
		Featurization            &   Paradigm   & Error1 & Error23 & Error21 & Error31 & Error32 & Overall \\ \hline
		\multirow{2}{*}{M.1} & classical   &  0.313  & 0.110 &  0.241&   0.078 &  0.241  & 0.330 \\
		&H-NP&  0.160  & 0.119  & 0.416  & 0.177 &  0.122 &  0.344 \\\hline
		\multirow{2}{*}{M.2} &  classical &   0.466&   0.153   & 0.280   &0.267&   0.370&   0.491\\
		& H-NP         &  0.172  & 0.091  &  0.640 &  0.587 &  0.215 &  0.542\\\hline
		\multirow{2}{*}{M.3} &  classical   &  0.248 &  0.115  & 0.226 &  0.060  & 0.178  & 0.284\\
		& H-NP   &  0.159   &0.129  & 0.336 &  0.108  & 0.134 &  0.303\\\hline
		\multirow{2}{*}{M.4} &  classical &  0.241 &  0.108  & 0.216  & 0.050 &  0.157  & 0.267 \\
		& H-NP     &  0.169 &  0.131 &  0.305  & 0.093 &  0.109  & 0.285\\\hline

		\multicolumn{8}{c}{Random Forest} \\\hline
		Featurization            &     Paradigm              & Error1 & Error23 & Error21  & Error31 & Error32 & Overall  \\ \hline
		\multirow{2}{*}{M.1} &  classical  & 0.262  & 0.049 &  0.257  & 0.091  & 0.228 &  0.293\\
		& H-NP   &  0.177  & 0.121   &0.356  & 0.096  & 0.072  & 0.297\\\hline
		\multirow{2}{*}{M.2} &  classical  &   0.361 &  0.095  & 0.256  & 0.216   &0.452  & 0.426\\
		&  H-NP   & 0.158  & 0.122 &  0.491 &  0.402 &  0.247 &  0.455\\\hline
		\multirow{2}{*}{M.3} &  classical  &0.314  & 0.039 &  0.200  & 0.113  & 0.369 &  0.321\\
		&  H-NP  &  0.178  & 0.116 &  0.386 &  0.148  & 0.126  & 0.332\\\hline 
		\multirow{2}{*}{M.4} &  classical &  0.300 &  0.036  & 0.219  & 0.130  & 0.353  & 0.323\\
		&  H-NP     &  0.162 &  0.120 &  0.407&   0.175 &  0.115  & 0.340\\\hline

		\multicolumn{8}{c}{SVM} \\\hline
		Featurization            &     Paradigm  & Error1 & Error23  & Error21 & Error31 & Error32  & Overall\\ \hline
		\multirow{2}{*}{M.1} &  classical &     0.275 &  0.091 &  0.253  & 0.081  & 0.219  & 0.309 \\
		&  H-NP  &  0.159 &  0.118  & 0.394  & 0.104 &  0.158 &  0.326\\\hline
		\multirow{2}{*}{M.2} &  classical  &  0.437 &  0.164&   0.280  & 0.281  & 0.365 &  0.487\\
		&  H-NP   & 0.175 &  0.110  & 0.613   &0.542 &  0.258 & 0.539\\\hline
		\multirow{2}{*}{M.3} &  classical  &  0.227  & 0.082  & 0.229 &  0.041  & 0.157  & 0.255\\
		&  H-NP  & 0.175 &  0.123  & 0.295 &  0.045  & 0.106   &0.269\\\hline     
		\multirow{2}{*}{M.4} &  classical  &  0.229   &0.077 &  0.222  & 0.037  & 0.160 &  0.251\\
		&  H-NP     &0.172 &  0.119 &  0.288  & 0.040 &  0.104 &  0.261\\\hline

		\hline
	\end{tabular}
	\caption{ \footnotesize The averages of approximate errors. ``error1", ``error23", ``error21", ``error31", ``error32", ``overall" correspond to $R_{1\star}(\hatphi)$, $R_{2\star}(\hatphi)$, $P_2(\hatY= 1)$, $P_3(\hatY= 1)$, $P_3(\hatY= 2)$ and $P(\hatY \neq Y)$, respectively.}\label{table:distribution}
\end{table}

\subsection{A neural network  classifier}\label{supp.sec:nn_classifier}

We apply the Term Frequency - Inverse Document Frequency (TF-IDF) transformation \citep{moussa2018single} to the matrix $(A^{(1)}, \ldots, A^{(N)}) \in \R^{n_g\times (18\times N)}$, and let the vector  $A^{(j)}_{\mathrm{TFIDF}} \in \R^{(18\times n_g)\times 1}$ be the concatenated TF-IDF scores belonging to the $j$-th patient. To extract features, we use PCA to reduce the dimension of $(A^{(1)}_{\mathrm{TFIDF}},\ldots, A^{(N)}_{\mathrm{TFIDF}})^\top\in \R^{N\times(18\times n_g)}$ to $N \times 512$, i.e., for patient $j$, we obtain a feature vector $X_j \in \R^{512}$. Then, we use a fully connected neural network with one hidden layer with 32 nodes, and the results are presented in Supplementary Figure~\ref{supp.fig:result}.

 \begin{figure}[ht]
 \centering
 \begin{minipage}{1\linewidth}
 \centering
		\includegraphics[width=12cm]{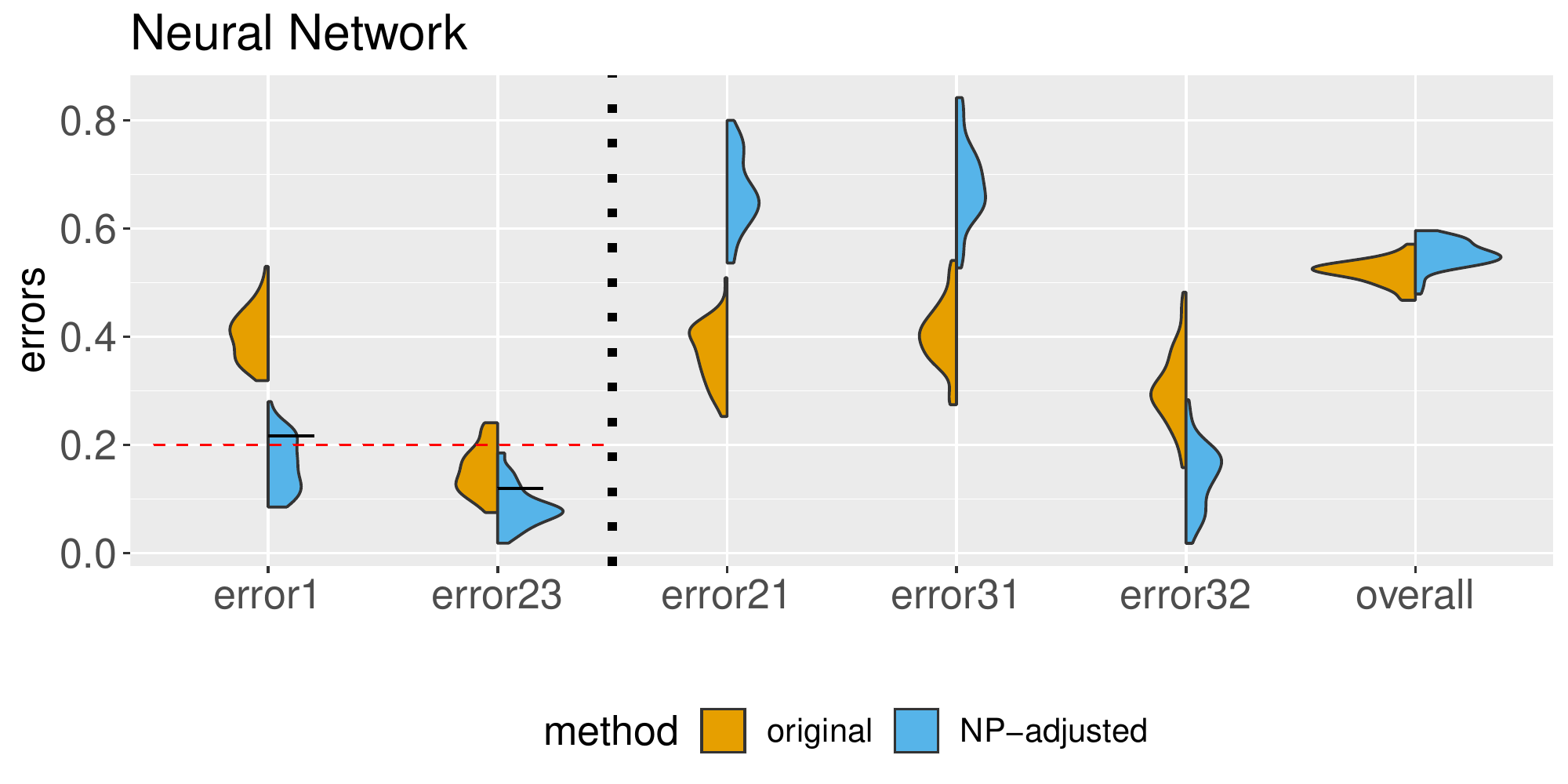}
	\end{minipage}
	
\begin{minipage}{1\linewidth}
\vspace{1cm}
		\centering
			\footnotesize
		\begin{tabular}{c|c|c|c|c|c|c}
\hline
             \multicolumn{7}{c}{Neural Network} \\\hline
               Paradigm               & Error1 & Error23 & Error21   & Error31 & Error32 & Overall \\ \hline
     \multirow{2}{*}{} classical &  0.403   &0.153  & 0.370 &  0.404  & 0.304 &  0.520 \\
         H-NP &  0.164 &  0.087 &  0.666  & 0.683 &  0.141  & 0.552 \\\hline 
		\end{tabular}
	\end{minipage}\hfill
	\caption{The distributions and averages of approximate errors for the neural network approach and the H-NP classifier. ``error1", ``error23", ``error21", ``error32", ``overall" correspond to $R_{1\star}(\hatphi)$, $R_{2\star}(\hatphi)$, $P_2(\hatY= 1)$, $P_3(\hatY= 1)$, $P_3(\hatY= 2)$ and $P(\hatY \neq Y)$, respectively.}\label{supp.fig:result}
 \end{figure}

\subsection{Importance of cell types in severe COVID-19 outcomes}

Additional results ranking importance of cell types in Table \ref{supp.table:coeff_pval}.

\begin{table}[h!!!]
\centering
\footnotesize
\begin{tabular}{c|c|c|c}
\hline
cell type & p-value & cell type & p-value \\
\hline
CD14 Mono &1.38e-05 & RBC    & 3.91e-01\\
NK  &9.95e-04& CD4 T  &4.12e-01\\
CD8 T &9.10e-03& MAIT   & 4.16e-01\\
Neutrophil&9.54e-03&  DN  &  4.26e-01\\
B & 9.49e-02& NKT &6.13e-01\\
gdT & 1.16e-01& MAST& 7.81e-01\\
 HSPC  &2.47e-01 & Plasma &8.61e-01\\
CD16 Mono &3.52e-01& DC &   9.34e-01\\
Platelet &3.73e-01& &\\\hline
\end{tabular}
\caption{Ranking cell types by their
coefficients that quantify the effect of the predictors (cell type expression) on the log odds ratios of the severe category relative to the healthy category in logistic regression with the featurization~M.2.}\label{supp.table:coeff_pval}
\end{table}

\subsection{Gene ontology (GO) enrichment analysis of the ranked gene list}\label{supp.sec:rank_gene}

We demonstrate the utility of genome-wide expression measurements in a classification setting by identifying genes and pathways associated with disease severity.  To achieve this, we employ logistic regression with the featurization~M.4, which has the best overall performance in Table~\ref{table:distribution}. Specifically, we rank the genes based on the coefficients that quantify the effect of the predictors (gene expression) on the log odds ratios of the outcome categories (severe) relative to a reference category (healthy) in a multinomial logistic regression model. A relatively high level of gene expression in the severe group is reflected as a larger positive coefficient, while a relatively high level of gene expression in healthy controls results in a negative coefficient. Thus, ranking the genes by their coefficients allows us to identify genes that are strongly associated with severe COVID-19.

To determine whether the ranked gene list is enriched in certain biological pathways, we use the R package \texttt{fgsea} \citep{korotkevich2016fast} and the Gene Ontology: Biological Process (GOBP) pathway database in the R package \texttt{msigdbr} to perform gene set enrichment analysis. Table~\ref{supp.table:severe healthy_coefficient_GOBP} shows the significant pathways and their corresponding adjusted p-values. Most of the pathways identified in the analysis are directly associated with various aspects of the immune response related to viral infections. Specifically, the leukocyte-mediated cytotoxicity pathway has been supported by biological studies, which have shown that leukocytes such as NK and CD8$^+$ T cells play critical roles in recognizing and targeting viral-infected cells for destruction through cytotoxicity \citep{peng2020broad,liu2020longitudinal}.

\begin{table}[h!!!]
\centering
\footnotesize
\begin{tabular}{c|c}
\hline\hline
                  \multicolumn{2}{c}{Severe vs. Healthy} \\\hline
      pathway  & p.adjust  \\ \hline
response to virus & 2.46e-02 \\   
defense response to symbiont & 2.46e-02 \\                                    
positive regulation of dna binding transcription factor activity & 2.46e-02 \\
regulation of dna binding transcription factor activity & 2.46e-02 \\         
leukocyte mediated cytotoxicity & 2.46e-02 \\                                 
regulation of leukocyte migration & 3.85e-02 \\                               
cell activation involved in immune response & 3.85e-02 \\\hline

\end{tabular}
\caption{The most significant GOBP pathways and their corresponding adjusted p-values, using the ranked gene list from logistic regression and the featurization M.4.}\label{supp.table:severe healthy_coefficient_GOBP}
\end{table}

\subsection{Gene functional modules from co-expression network analysis}\label{supp.sec:express_network}

{The above analysis relies on a ranked feature list from a chosen base classifier and does not directly account for correlation patterns among genes. Next, we remove the need for feature ranking and construct gene co-expression networks by measuring pairwise correlations between gene expression levels across different patient samples.} By identifying clusters or modules of co-expressed genes, these networks can provide valuable insights into the functional relationships between genes and the underlying biological processes. Furthermore, we will relate these functional modules to the H-NP classification result. In this analysis, we adopt the same data splitting strategy as in Section~3.2, where $70\%$ of the data is used for training the H-NP classifier, and the remaining $30\%$ of the data is preserved for testing the classifier and constructing gene co-expression networks and performing gene ontology enrichment analysis. We again employ logistic regression with the featurization~M.4 and the control and tolerance levels remain the same as in Section~3.2.

We begin by analyzing significant biological processes in all severity groups using the held-out data. The featurization~M.4 generates a feature vector with the same dimension as the number of genes for each patient. We then construct a gene co-expression network by computing the correlations between gene pairs across all patients, following the standard workflow in \citet{zhang2005general}. We use the \texttt{TOMdist} function from the R package \texttt{WGCNA} (version 1.71) \citep{langfelder2014tutorials} to compute dissimilarity measures between gene pairs for performing hierarchical clustering, followed by using the \texttt{cutreeDynamic} function from the R package \texttt{dynamicTreeCut} to detect functional modules. Figure~\ref{sup.fig:mc_relation} shows the correlations between the module eigengenes (computed by \texttt{WGCNA}) and the predicted severity labels from the H-NP and classical paradigms. Overall, the H-NP labels have stronger associations with most of the eigengenes, suggesting that they better capture the underlying signals in the data as represented by these functional modules. In particular, module 1 and module 3 have the strongest association, and a closer inspection of their GO terms shows that they are significantly enriched in genes related to B cell activation and immune response to virus
 (Figures~\ref{supp.fig:all_red_top_10}--\ref{supp.fig:all_2_top_10}, obtained using the R package \texttt{clusterProfiler} \citep{wu2021clusterprofiler}).

Furthermore, we study differences in the pathway enrichment between severe and healthy patients. Using the H-NP predicted labels and the test data, we construct gene co-expression matrices for the severe and healthy patients separately, followed by performing the same GO enrichment analysis as above. The top 3 enriched pathways for each module are summarized in Tables~\ref{supp.table:severe_path} (severe) and~\ref{supp.table:healthy_path} (healthy).
The tables show significantly different GO terms between the two groups; there is strong evidence of immune and viral response among the severe patients, while no such evidence is observed in the healthy group. In particular, the GO terms in the severe group are consistent with the literature that suggests severe COVID-19 is caused by an overactive immune response \citep{huang2020clinical,que2022cytokine}, known as a cytokine storm, that leads to inflammation and tissue damage. Understanding the mechanisms underlying this immune response is crucial for developing effective treatments for severe COVID-19. {Finally, comparing the GO terms from the severe patients with their labels given by the H-NP paradigm (Table~\ref{supp.table:severe_path}) and classical paradigm (Tables~\ref{supp.table:severe_path_classical}) respectively, H-NP captures more significantly
enriched modules with specific references to important cell types, including T cells, and
subtypes of T cells. }

\begin{figure}[h!]
 \centering
    \includegraphics[width=12cm]{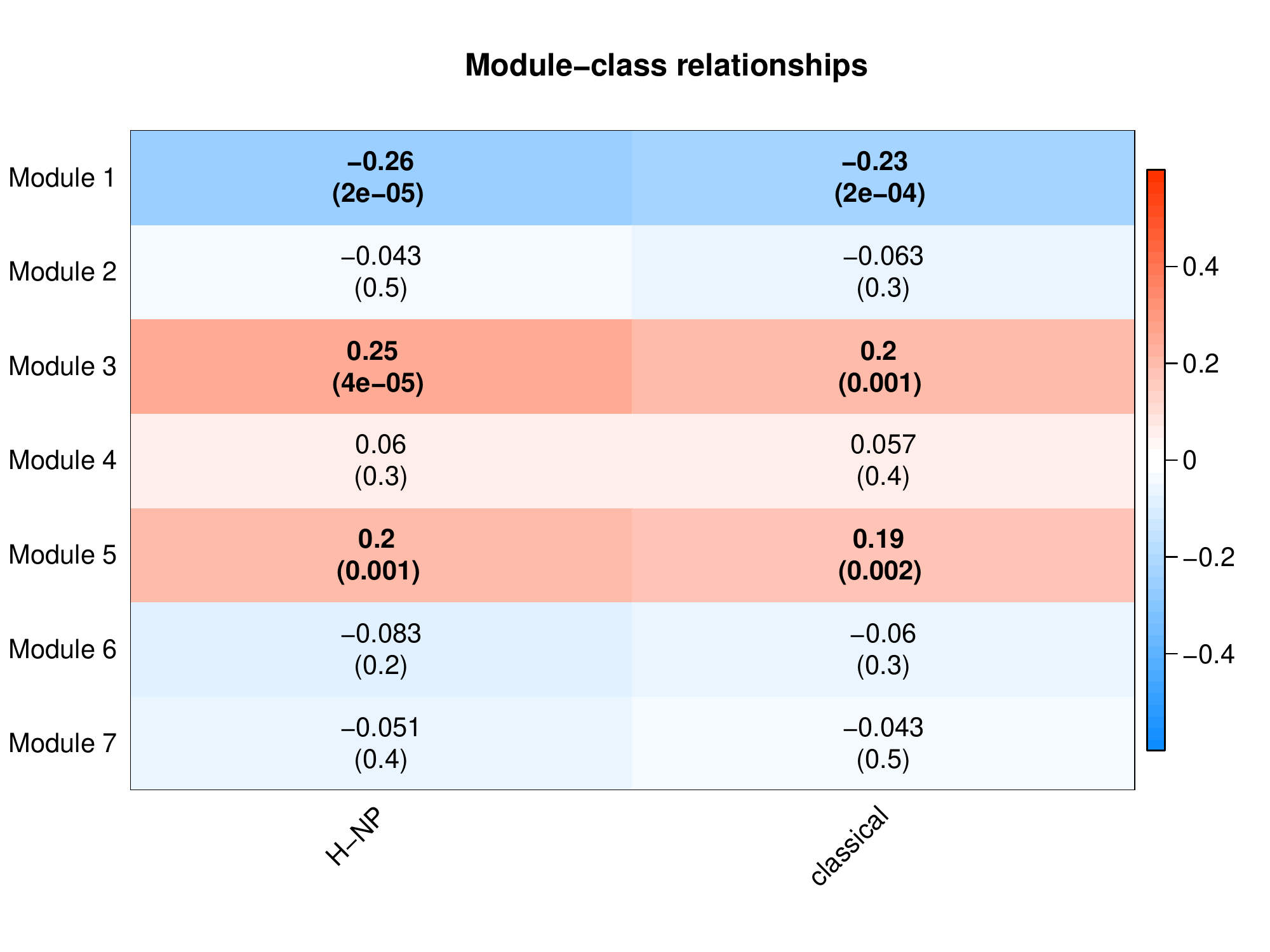}
    \caption{Correlations between consensus module eigengenes and severity labels predicted by the H-NP and classical approaches. The numbers in parentheses indicate p-values. The results with p-values less than 0.05 are in bold.}\label{sup.fig:mc_relation}
\end{figure}

\begin{figure}[h!]
 \centering
    \includegraphics[width=10cm]{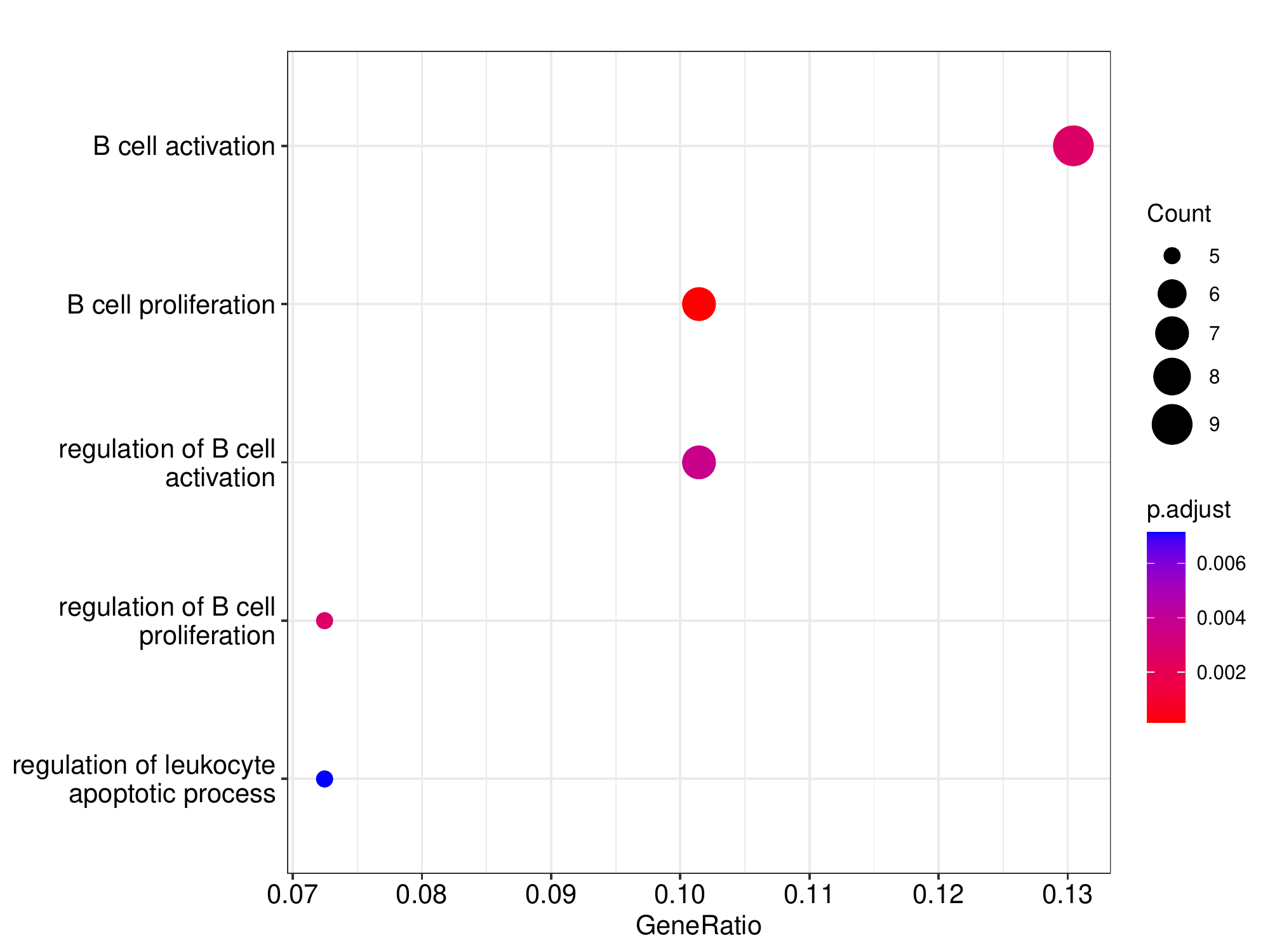}
    \caption{Significant GO terms in Module 1 of Figure~\ref{sup.fig:mc_relation}. The number of genes and significance level of each dot are represented by the dot's size and color, respectively.}\label{supp.fig:all_red_top_10}
\end{figure}

\begin{figure}[h!]
 \centering
    \includegraphics[width=10cm]{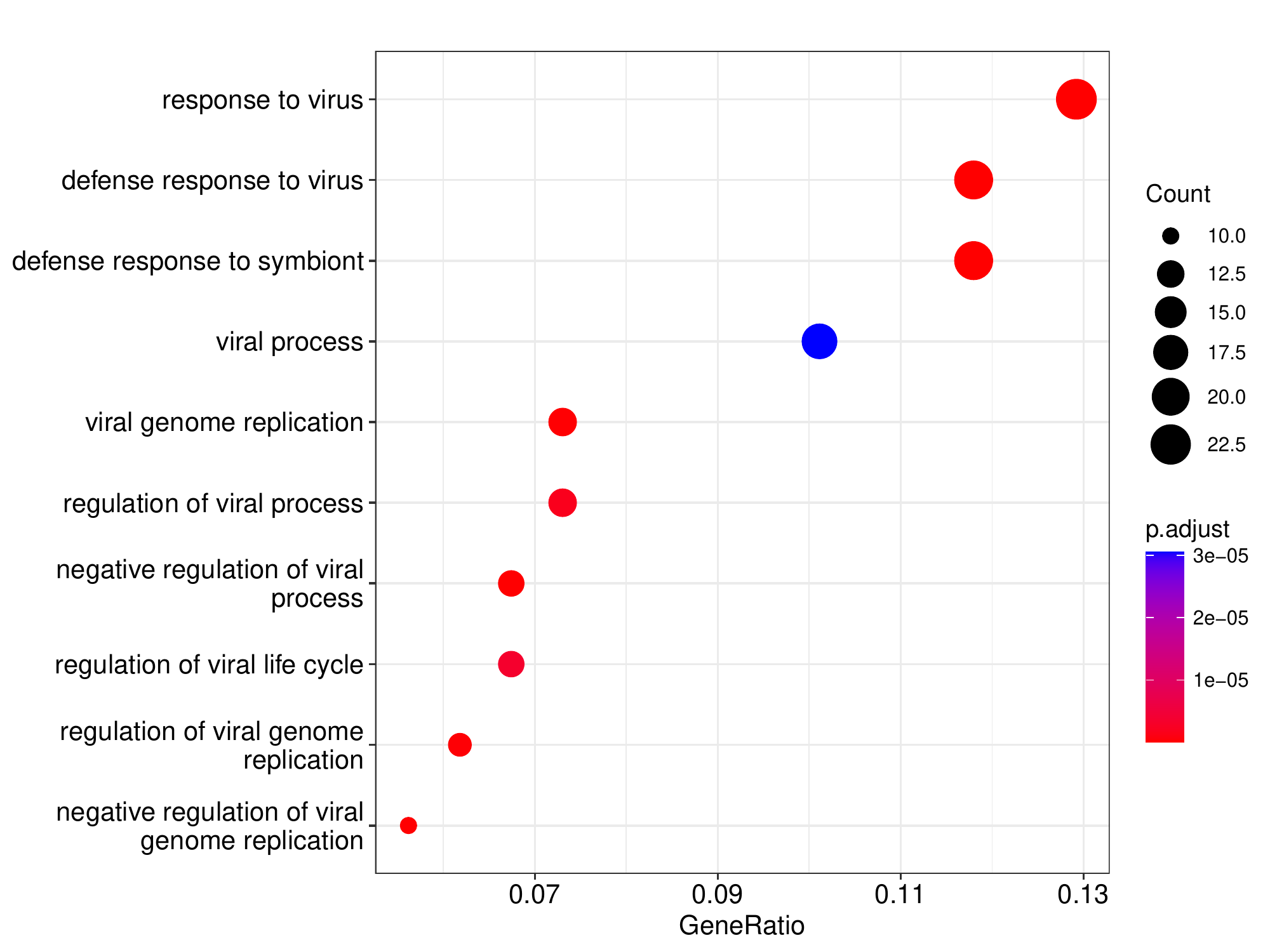}
    \caption{Significant GO terms in Module 3 of Figure~\ref{sup.fig:mc_relation}. The number of genes and significance level of each dot are represented by the dot's size and color, respectively.}\label{supp.fig:all_2_top_10}
\end{figure}

\begin{table}[h!!!]
\centering
\footnotesize
\begin{tabular}{c|c|c}
\hline\hline
                  \multicolumn{3}{c}{H-NP: Severe} \\\hline
       module  & pathway  & p.adjust  \\ \hline

\multirow{3}{*}{ 1 }& cell activation involved in immune response & 1.33e-22 \\
& leukocyte activation involved in immune response & 2.30e-22 \\              
& leukocyte migration & 5.92e-21 \\\hline                                       
\multirow{3}{*}{ 2 }& positive regulation of cytokine production & 8.65e-13 \\
& myeloid cell differentiation & 1.22e-10 \\                                 
& cytoplasmic translation & 1.86e-10 \\ \hline                                
\multirow{3}{*}{ 3 }& defense response to virus & 6.90e-10 \\
& defense response to symbiont & 6.90e-10 \\               
& negative regulation of viral process & 1.54e-09 \\     \hline   
\multirow{3}{*}{ 4 }& positive regulation of inflammatory response & 1.16e-05 \\
& CD4-positive, alpha-beta T cell differentiation & 1.16e-05 \\                
& leukocyte cell-cell adhesion & 1.16e-05 \\  \hline                               
\multirow{3}{*}{ 5 }& B cell activation & 4.84e-06 \\
& B cell differentiation & 1.68e-04 \\              
& B cell proliferation & 2.29e-04 \\ \hline             
\multirow{3}{*}{ 6 }& regulation of T cell apoptotic process & 6.29e-03 \\
& T cell apoptotic process & 2.24e-02 \\                                 
& regulation of lymphocyte apoptotic process & 2.43e-02 \\ \hline

\end{tabular}
\caption{Top 3 GO terms for each module and their corresponding adjusted p-values. The module detection is conducted on the severe group as labeled by H-NP.} \label{supp.table:severe_path}
\end{table}

\begin{table}[h!!!]
\centering
\footnotesize
\begin{tabular}{c|c|c}
\hline\hline
                  \multicolumn{3}{c}{H-NP: Healthy} \\\hline
       module  & pathway & p.adjust  \\ \hline
\multirow{3}{*}{ 1 }& mononuclear cell differentiation & 3.84e-37 \\
& lymphocyte differentiation & 2.41e-34 \\                         
& cell activation involved in immune response & 1.35e-33 \\ \hline      
\multirow{3}{*}{ 2 }& cytoplasmic translation & 2.04e-11 \\
& ribosomal small subunit biogenesis & 2.02e-02 \\        
& ribosomal small subunit assembly & 4.08e-02 \\  \hline        
 3 & platelet activation & 1.01e-02 \\\hline                  
\multirow{3}{*}{ 4 }& histone modification & 1.23e-04 \\
& peptidyl-lysine modification & 1.16e-02 \\           
& lymphocyte apoptotic process & 3.81e-02 \\ \hline
\end{tabular}
\caption{Top 3 GO terms for each module and their corresponding adjusted p-values. The module detection is conducted on the healthy group as labeled by H-NP.} \label{supp.table:healthy_path}
\end{table}

\begin{table}[h!!!]
\centering
\footnotesize
\begin{tabular}{c|c|c}
\hline\hline
                  \multicolumn{3}{c}{Classical: Severe} \\\hline
       module  & pathway  & p.adjust  \\ \hline
\multirow{3}{*}{ 1 }& positive regulation of cytokine production & 3.35e-30 \\
& mononuclear cell differentiation & 1.29e-27 \\                             
& leukocyte cell-cell adhesion & 1.29e-27 \\    \hline                             
\multirow{3}{*}{ 2 }& B cell activation & 1.04e-07 \\
& histone modification & 3.28e-07 \\                
& B cell proliferation & 5.50e-07 \\      \hline            
\multirow{3}{*}{ 3 }& positive regulation of nitric-oxide synthase biosynthetic process & 4.08e-02 \\
& urogenital system development & 4.08e-02 \\                                          & nitric-oxide synthase biosynthetic process & 4.08e-02 \\      \hline                  
\multirow{3}{*}{ 4 }& regulation of epidermal growth factor-activated receptor activity & 2.6e-02 \\
& positive regulation of transforming growth factor beta receptor signaling pathway & 2.6e-02 \\ 
& positive regulation of cellular response to transforming growth factor beta stimulus & 2.6e-02 \\ \hline

\end{tabular}
\caption{{Top 3 GO terms for each module and their corresponding adjusted p-values. The module detection is conducted on the severe groups as labeled by the classical paradigm.} } \label{supp.table:severe_path_classical}
\end{table}

\clearpage

\section{General H-NP umbrella algorithm for $\cI$ classes}
\label{sec:supp_alg}
\begin{algorithm}[htb!]
\caption{General H-NP umbrella algorithm for $\cI$ classes}
\SetKw{KwBy}{by}
\SetKwInOut{Input}{Input}\SetKwInOut{Output}{Output}

\SetAlgoLined

\Input{ Sample: $\cS = \cup_{i \in [\cI]}\cS_i$; levels: $(\alpha_1,\ldots, \alpha_{\cI - 1})$;  tolerances: $(\delta_1,\ldots, \delta_{\cI - 1})$; grid set: $A_1,\ldots, A_{\cI - 2}$ (e.g., $\cT_1,\ldots, \cT_{\cI - 2}$).}

$\hatpi_i = |\cS_i|/|\cS|$;

$\cS_{1s},\cS_{1t}\leftarrow $ Random split $\cS_1$; 

$\cS_{is},\cS_{it}, \cS_{ie}  \leftarrow $ Random split $\cS_i$ for $i = 2, \ldots, \cI - 1$;

$\cS_{\cI s}, \cS_{\cI e}  \leftarrow $ Random split $\cS_\cI$;

$\cS_s = \cup_{i \in [\cI]}\cS_{is}$;

$T_1, \ldots, T_{\cI - 1} \leftarrow \mbox{A  base classification method}(\cS_s)$ ;

$\ot_1 \leftarrow \mbox{UpperBound}(\cS_{1t}, \alpha_1, \delta_1, (T_1), \mbox{NULL}) $;

$\Tilde{R^c} = 1 $;

\For{$t_1 \in A_1 \cap (-\infty, \ot_1]$}{

$\ot_2 \leftarrow \mbox{UpperBound}(\cS_{2t}, \alpha_2, \delta_2, (T_1, T_2), (t_1) )$;

\For{$t_2 \in A_2\cap (-\infty, \ot_2]$}{

$\ot_3 \leftarrow \mbox{UpperBound}(\cS_{3t}, \alpha_3, \delta_3, (T_1, T_2,T_3), (t_1,t_2) )$;

$\cdots \cdots$;

\For{$t_{\cI - 2} \in A_{\cI - 2} \cap (-\infty, \ot_{\cI -2}]$}{

$\ot_{\cI - 1} \leftarrow \mbox{UpperBound}(\cS_{\cI - 1 t}, \alpha_{\cI - 1}, \delta_{\cI - 1}, (T_1, \ldots, T_{\cI - 1}), (t_1,\ldots,t_{\cI - 2}) )$;

$\hatphi \leftarrow$ a classifier with respect to $t_1, \ldots ,t_{\cI -1 }$;

$\Tilde{R^c}_{\mathrm{new}} =\sum^{\cI}_{i = 2} (\hatpi_i/|\cS_{ie}|) \sum_{X \in \cS_{ie}} \1\{ \hatphi(X) < i \}$;

\If{$\Tilde{R^c}_{\mathrm{new}}< \Tilde{R^c}$}{
$\Tilde{R^c} \leftarrow \Tilde{R^c}_{\mathrm{new}}$, $\hatphi^* \leftarrow \hatphi$ }}
}}

\Output{$\hatphi^*$}
\end{algorithm}

For a general $\cI$, we conduct a grid search over dimension $\cI-2$. For $1\leq i <\cI-2$, each grid point $t_i$ is selected from the set $\cT_i \cap (-\infty,\ot_i]$. Thus the grid size is smaller than $C^{\cI - 2}$, where $C$ is typically much smaller than the cardinality of any thresholding set $|\cS_{it}|$, due to restriction of the selection region by imposed by $\ot_i$. 

To visualize how the computational time changes with the number of classes $\cI$, we compare the computational time for training the scoring function and selecting the thresholds (i.e., running our H-NP algorithm) in Figure~\ref{sup.fig:computation_time} for $\cI = 3,6,9,12$. For this experiment, we set $N_i = 1{,}000$ for $i \in [\cI]$, and $\alpha_i = \delta_i = 0.05$ for $i \in [\cI - 1]$. We generate feature vectors in class $i$ as ${(X^i)}^\top \sim N(\mu_i, I)$, where $\mu_i \in \R^{200}$ and the entries are independently drawn from the normal distribution with mean 0 and standard deviation 0.05. For classes 1, \ldots, $\cI - 1$, we randomly select $\sim$10\% observations from each class for threshold selection. For classes 2, \ldots, $\cI$, we randomly select 5\% of observations to compute the empirical errors. The remaining observations are used to compute the scoring function with logistic regression as the base classification method. As expected, the computational times of both processes increase with $\cI$, but overall selecting the thresholds takes a much smaller fraction of time than the training itself.

\begin{figure}[h!]
 \centering
    \includegraphics[width=12cm]{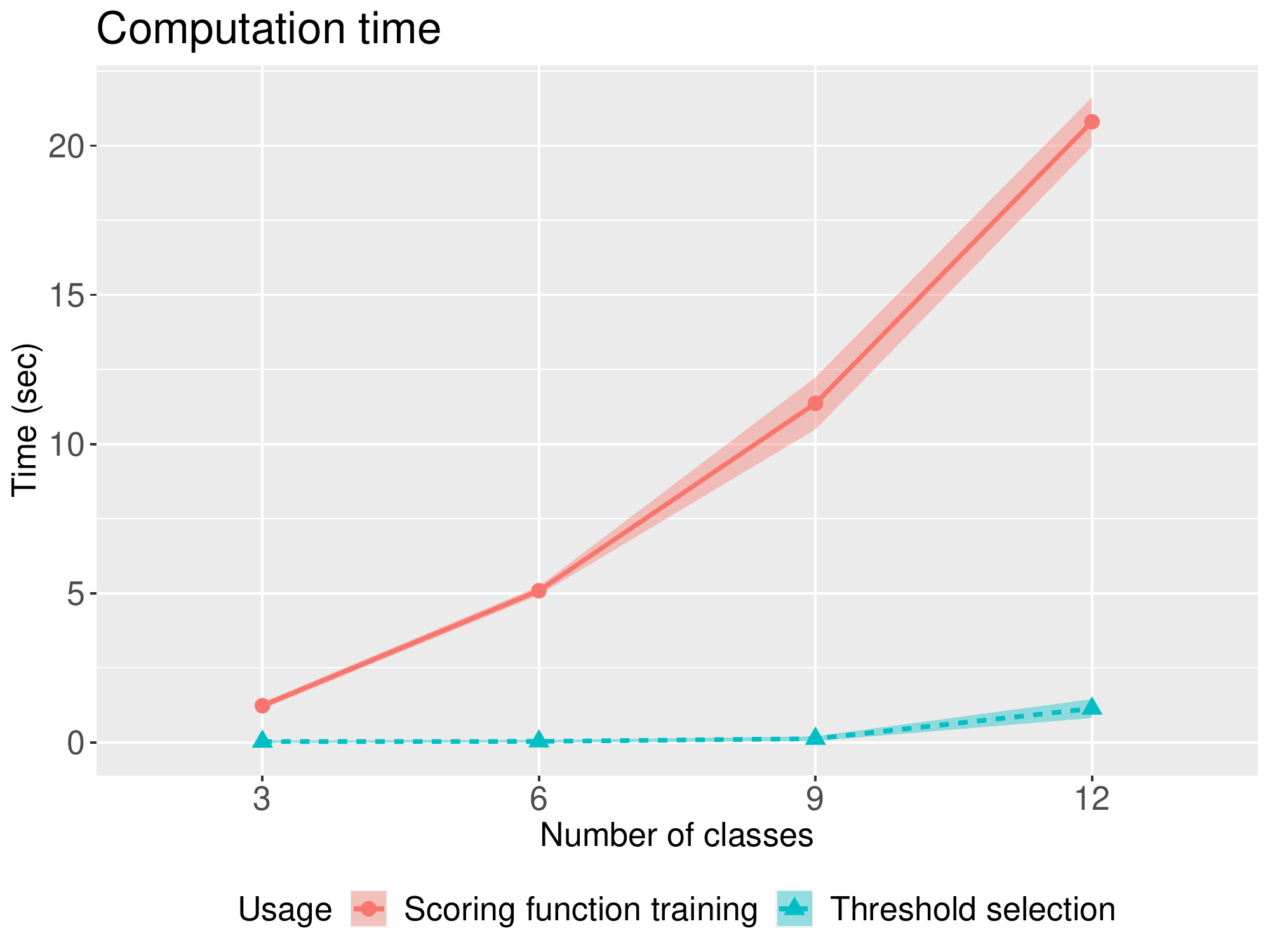}
    \caption{The computational times for training the scoring function and selecting the thresholds using the H-NP algorithm with $\cI = 3,6,9,12$. The points represent the average times, and the shade represents the magnitude of the standard deviation, from 100 repetitions.}\label{sup.fig:computation_time}
\end{figure}

\end{document}